\documentclass{article}

\usepackage[utf8]{inputenc} 
\usepackage[T1]{fontenc}    
\usepackage[colorlinks=true]{hyperref} 
\usepackage{url}            
\usepackage{booktabs}       
\usepackage{amsfonts}       
\usepackage{nicefrac}       
\usepackage{microtype}      

\usepackage{enumitem}
\usepackage{amsmath}
\usepackage{amsthm}
\usepackage{mathtools}
\usepackage{xcolor}
\usepackage{algorithm} 
\usepackage{algpseudocode} 
\usepackage{graphicx}
\usepackage{appendix}
\usepackage{subcaption}
\usepackage[margin=2.5cm]{geometry}
\usepackage{bbm}

\usepackage[style=authoryear,natbib, maxcitenames=2]{biblatex}
\usepackage{amssymb}
\usepackage{xcolor}
\usepackage{enumitem}

\usepackage[capitalize,noabbrev]{cleveref}

\usepackage{fancyhdr}
\fancyheadoffset{0pt}
\usepackage{authblk}

\newtheorem{theorem}{Theorem}
\newtheorem{proposition}[theorem]{Proposition}
\newtheorem{lemma}[theorem]{Lemma}
\newtheorem{corollary}[theorem]{Corollary}
\newtheorem{remark}{Remark}

\newcommand{\id}{\stackrel{\text{indep}}{\sim}}

\DeclareMathOperator*{\argmin}{arg\,min}

\newcommand{\mcA}{\mathcal{A}}
\newcommand{\mcB}{\mathcal{B}}
\newcommand{\mcC}{\mathcal{C}}

\newcommand{\mcE}{\mathcal{E}}

\newcommand{\mcG}{\mathcal{G}}

\newcommand{\mcL}{\mathcal{L}}
\newcommand{\mcM}{\mathcal{M}}
\newcommand{\mcN}{\mathcal{N}}

\newcommand{\mcP}{\mathcal{P}}

\newcommand{\mcR}{\mathcal{R}}

\newcommand{\mcV}{\mathcal{V}}

\newcommand{\mcX}{\mathcal{X}}

\newcommand{\emprisk}{\mcL_n}
\newcommand{\empriskhat}{\widehat{\mcL}_n}

\newcommand{\psamp}{\mathbb{P}\big( (i, j) \in \mcP(\mcG_n) \,|\, \mcG_n \big) }
\newcommand{\nsamp}{\mathbb{P}\big( (i, j) \in \mcN(\mcG_n) \,|\, \mcG_n \big) }
\newcommand{\embedip}{\langle u_i, v_j \rangle}

\newcommand{\psamphat}{f_{\mcP}(\lambda_i, \lambda_j)}
\newcommand{\nsamphat}{f_{\mcN}(\lambda_i, \lambda_j)}

\title{Community Detection Guarantees Using Embeddings 
Learned by Node2Vec}

\author[1]{Andrew Davison}
\author[2]{S. Carlyle Morgan}
\author[3]{Owen G. Ward}
\affil[1]{Department of Statistics, Columbia University}
\affil[2]{Department of Statistics, University of Michigan}
\affil[3]{Department of Statistics and Actuarial Science, Simon 
Fraser University}
\affil[ ]{ ad3395@columbia.edu, scmorgan@umich.edu, owen\_ward@sfu.ca}
\begin{document}

\maketitle

\begin{abstract}
Embedding the nodes of a large network into an Euclidean space is
a common objective in modern machine learning, with a variety of
tools available. 
These embeddings can then be used as features for tasks
such as community detection/node clustering or link prediction,
where they achieve state of the art performance. 
With the exception of spectral clustering methods,
there is little theoretical understanding for
commonly used approaches to learning embeddings. 
In this work
we examine the theoretical properties of 
the embeddings learned
by node2vec. Our main result shows that the use of $k$-means clustering on
the embedding vectors produced by node2vec
gives weakly consistent
community recovery for the nodes in
(degree corrected) stochastic block models.
We also discuss the use of these embeddings for node and link prediction tasks.
We demonstrate this result empirically,
and examine how this relates to other
embedding tools for network data.

\end{abstract}

\section{Introduction}

Within network science, a widely applicable and important inference task
is to understand how the behavior of interactions between different
units (nodes) within the network depend on their latent characteristics.
This occurs within a wide array of disciplines, from sociological \citep{freeman_development_2004}
to biological \citep{luo_modular_2007} networks.

One simple and interpretable model for such a task is the stochastic
block model (SBM) \citep{holland_stochastic_1983},
which assumes
that nodes within the network are assigned a discrete community label.
Edges between nodes in the network are then formed independently
across all pairs of edges, conditional on these community
assignments. While such a model is simplistic, 
various extensions have been proposed.
These include the degree corrected SBM (DCSBM), used to handle
degree heterogenity
\citep{karrer_stochastic_2011}, and mixed-membership
SBMs, used to allow for more complex community structures 
\citep{airoldi_mixed_2008}. 
These extensions have seen
a wide degree of empirical success \citep{latouche_blog_2011, legramanti_extended_2022, airoldi_mixed_2006}.

A restriction of the stochastic block model and its generalizations
is the requirement for a discrete
community assignment as a latent 
representation of the units within the network.
While the statistical community has previously considered more 
flexible latent representations \citep{hoff_latent_2002},
over the past decade, there have been significant advancements in
general \emph{embedding methods} for networks. These produce
general vector representations of units within a network, and
can achieve start-of-the-art performance in downstream tasks
for node classification and link prediction. 


An early example of such a method is 
spectral clustering \citep{ng_spectral_2001},
which constructs an embedding of the nodes
in the network from an eigendecomposition of the graph Laplacian.
The $k$ smallest non zero eigenvectors provides a $k$ dimensional 
representation of each of the nodes in the network.
This
has been shown to allow consistent community recovery
\citep{lei_consistency_2015}, however it
may not be computationally feasible on the large networks which are 
now common.
More recently, machine learning
methods for producing vector representations have
sought inspiration from NLP methods and the broader machine
learning literature, such as the node2vec algorithm 
\citep{grover_node2vec_2016},
graph convolutional networks \citep{zhang_graph_2019},
graph attention networks \citep{velivckovic_graph_2017} and others.
There are now a wide class of embedding methods which are available
to practitioners 
which can be applied across a
mixture of unsupervised and supervised settings. \citep{cui_survey_2018}
provides a survey of relatively recent developments
and \citep{ward_next_2021} reviews the connection between the embedding
procedure and the potential downstream task.


Embedding methods such as Deepwalk \citep{perozzi_deepwalk_2014}
and node2vec \citep{grover_node2vec_2016}
consider random 
walks on the graph, where the probability
of such a walk is a function of the embedding of the associated 
nodes. Given embedding vectors $\widehat{\omega}_u,\omega_v \in \mathbb{R}^d$
of nodes $u$ and $v$ respectively, from graph $\mathcal{G}$ with
vertex set $\mathcal{V}$, the probability of a random 
walk from node $u$ to node $v$ is modeled as
\begin{equation}
P(v|u) = \frac{\exp(\langle \omega_v, \widehat{\omega}_u\rangle)}
{\sum_{l\in \mathcal{V}}\exp(\langle\omega_l,\widehat{\omega}_u\rangle)},
\label{n2v_loss}
\end{equation}
where $\langle x,y\rangle$ is the inner product of $x$ and $y$.
This leads
to a representation of each of the nodes in the network as a vector
in $d$ dimensional Euclidean space. This representation is 
then amenable to potential downstream tasks about the network.
For example, if we wish to cluster the nodes in the network we can simply
cluster their embedding vectors. Or, if we wish to classify
the nodes in the network, we can use these embeddings 
to construct a multinomial classifier. We 
note that the sampling schemes introduced by DeepWalk and node2vec 
motivate more complex models such as GraphSAGE \citep{hamilton_inductive_2017}
and Deep Graph Infomax \citep{velickovic_deep_2018},
which utilise similar node sampling schemes
for learning embeddings of networks. 

As such, one of the key goals of learning vector 
representations of the units within networks 
is to allow for easy use for a multitude of downstream tasks.
However, there is little theoretical understanding to
what information is carried within these representations, and
whether they can be applied successfully and efficiently to
downstream tasks. This paper aims to 
address this gap by examining whether learned embeddings can 
facilitate community detection tasks in an unsupervised setting.

\subsection{Summary of main results}

Our main contribution is to describe the asymptotic distribution
of the embeddings learned by the node2vec procedure, and to then
use this to give consistency guarantees when these embeddings
are used for community detection. A simple and informal form of our
results, in the scenario of a balanced two block stochastic block model (SBM), 
is given below:

\begin{theorem}
    (Informal) Suppose we observe a sequence of graphs $\mcG_n$ on $n$
    vertices arising 
    from a two-dimensional stochastic block model: for each vertex $u
    \in [n]$ we assign a community label $c(u) \in \{0, 1\}$ with equal probability,
    and then we form edges in the graph independently with probability
    \begin{equation}
        \mathbb{P}\big(
            \text{$u$ and $v$ are connected}
        \big) = \begin{cases}
            \tilde{p} & \text{ if } c(u) = c(v) \\
            \tilde{q} & \text{ otherwise}
        \end{cases}
    \end{equation}
    where $\tilde{p} \neq \tilde{q}$. 
    Suppose that $(\widehat{\omega}_u)$ are two-dimensional
    embeddings learned by node2vec on the above graph (where we hide
    the dependence on $n$). Then there 
    exists some distinct vectors
    $\eta_{c(u)} \in \mathbb{R}^2$ such that 
    \begin{equation}
    \frac{1}{n} \sum_{u} \| \widehat{\omega}_u - \eta_{c(u)} \|_2^2 \to 0 \text{ in probability as }
    n \to \infty.
    \end{equation}
    Consequently, if we apply a k-means algorithm to the embeddings
    learned via node2vec, as $n \to \infty$ we will classify at least
    $100(1-\epsilon)$\% of vertices to the correct community 
    (up to permutation) with asymptotic probability $1$, for any $\epsilon > 0$.
\end{theorem}

We give formal theorem statements, complete with full
conditions, in Section~\ref{sec:results}; we note that our results
extend to graph models beyond SBMs and are not limited to the dense regime. 
To give some brief intuition for the
method of proof, we show that the probability that a pair $(u, v)$ is 
positively or negatively sampled within node2vec 
concentrates around a function which depends only on the underlying communities 
$c(u)$ and $c(v)$ of $u$ and $v$. With this, we are able to argue that 
the node2vec loss concentrates uniformly (in a neighborhood of their minima) 
around a function whose minima $M^*$ is such that
$M^*_{u,v} = \widetilde{M}_{c(u), c(v)}$ for some matrix $\widetilde{M}$. 
This allows us to show that any set of embeddings
which minimize the node2vec loss will converge (up to rotation) 
to vectors which depend only on the community label, which 
consequently allows us to give consistency guarantees for clustering algorithms 
such as k-means.

We highlight that while the theoretical properties of 
spectral clustering are well studied in the literature,
there are relatively few theoretical guarantees provided 
for more modern embedding procedures such as node2vec. 
Our work provides some of the
first theoretical results for models of this form. Our
main contributions are the following:
\begin{enumerate}[label=\roman*)]
    \item We give convergence guarantees for embeddings learned via node2vec, 
    under various sparsity regimes of (degree corrected) stochastic block models. We
    then use this to give weak consistency guarantees for community detection, when using
    the embeddings as features within a k-means clustering algorithm.
    
    \item We verify the theoretical guarantees for simulated networks and examine the the performance of this
    procedure on real networks. We also empirically investigate important extensions of these theoretical 
    results, relating to rates of recovery for community detection between node2vec and spectral clustering methods. We identify that as these networks grow the sampling parameters
    in node2vec have little impact on the performance of the proposed procedure.
\end{enumerate}



The layout of the paper is as follows. 
In Section~\ref{sec:framework} we formulate the problem of constructing an 
embedding of the nodes in a network and state the criterion 
under which we consider community detection. In Section~\ref{sec:results}
we give the main
result of this paper, the conditions under
which k-means clustering of the node2vec embedding of
a network gives consistent community recovery.
In Section~\ref{sec:exps} we verify these theoretical results 
empirically and investigate potential further results.
In Section~\ref{sec:conclusion} we summarize our contributions and consider potential
extensions. 


\subsection{Related Works}

Community detection for networks is a widely studied area with a large literature
of existing work. Several notions of theoretical guarantees for community recovery 
are provided in \citep{abbe_community_2017}, along with a survey of many 
existing approaches.  There are many existing works which consider
the embeddings obtained from the eigenvectors of the adjacency matrix of Laplacian
of a network. For example, 
\citep{lei_consistency_2015} considers spectral clustering
using the eigenvectors 
of the adjacency matrix for a stochastic block model.
Spectral clustering has provided such guarantees
for a wide variety of network models, including 
\citep{ma_determining_2021,
deng_strong_2021, rubin-delanchy_statistical_2017, levin_limit_2021, lei_network_2021}.

With the more recent development of random walk based embeddings, several
recent works have begun to examine the theoretical properties of such embeddings, 
however the treatment is limited compared to
spectral embeddings.
\citep{qiu_network_2018} study the global minimizers of the node2vec loss
in the setting where $d = n$, viewing the problem as a matrix factorization
problem. If $M^*$ is the global minimizing matrix, we highlight that their results apply 
for any $d \geq \mathrm{rank}(M^*)$. That said, this minimizer equals the entrywise logarithm of functions of the adjacency matrix $A$; we note that entrywise logarithms of matrices typically blow up their rank, and that even when "in expectation" the adjacency matrix is of low rank, the actual adjacency matrix is of full rank with high probability \citep{vu_rank_random_graph_2008}. This means that it is unlikely when $d \ll n$ that the global minimizer is the actual minimizer, which is the regime where
embedding dimensions are considered in practice. We contrast that with our results, where we can take $d =\Omega(\kappa)$ where $\kappa$ is the number of communities, and obtain rigorous guarantees for the embeddings.

\citep{zhang_consistency_2024} then studies the concentration
of the best rank $d$ approximation (with respect to the Frobenius norm) of the matrix $M^*$ about it’s expected value under SBM and DCSBM models for node2vec with $p = q =1$ only, to argue that the best rank $d$ approximation can be used for strongly consistent community detection. We note that our results can be applied to node2vec without this restriction on the hyperparameters. Otherwise, they give similar types of guarantees as our paper in similar sparsity regimes and with similar rates, but in stronger
norms. The key difference between our work and that of \citep{zhang_consistency_2024} is that we are able to give guarantees for the
the actual minimizers of the node2vec loss as soon as $d = \Omega(\kappa)$, whereas \citep{zhang_consistency_2024} use an approximation to the
global minimizer, without studying the gap between this matrix and any minimizer of the node2vec loss (which is a cross-entropy loss, and therefore difficult to relate to a Frobenius norm approximation). \citep{davison_asymptotics_2022}
and \citep{davison_asymptotics_2023} study node2vec with in the constrained setting 
(where $U=V$), and focus on giving more abstract guarantees for the gram matrix
in the setting of graphons. In \citep{davison_asymptotics_2023} the norm guarantees 
extend only to the $L_1$ norm between the gram matrix of the embeddings 
and the minimizer, which is not sufficient to give guarantees on the individual 
embeddings. In \citep{davison_asymptotics_2022} the norm guarantees are 
upgraded to the $L_2$ norm, albeit with less optimal rates of
convergence than what we show here. Our results also give guarantees for node2vec 
in full generality (no restriction on $p$ and $q$) and give 
the calculation details for SBMs and DCSBMs to explicitly describe the
asymptotic distribution in certain regimes.
\section{Framework}
\label{sec:framework}

We consider a network $\mathcal{G}$
consisting of a vertex set $\mathcal{V}$ of size $n$ and
edge set $\mathcal{E}$. We can express this also
using an $n\times n$ symmetric adjacency matrix $A$, 
where $A_{uv}=1$ indicates there is an undirected edge between node $u$
and node $v$, with $A_{uv}=0$ otherwise, where $u,v\in \mathcal{V}$.
Given a realisation of such a network, we wish to 
examine models for community structure of the nodes in the network.
We then examine the embeddings which can be 
obtained from node2vec and examine how they can be used for community
detection.

\subsection{Probabilistic models for community detection}
The most widely studied statistical model for community detection
is the Stochastic Block Model (SBM) \citep{holland_stochastic_1983}.
The SBM specifies a
distribution for the communities, placing each of the $n$ nodes into
one of $\kappa$ communities, where these community 
assignments are drawn from some categorical distribution
$\text{Categorical}(\pi)$.
Writing $c(u)\in [\kappa]$ for the community of $u$, the
connection probabilities between edges are independent,
conditional on these community assignments, with probability
\begin{equation}
\mathbb{P}(A_{uv}=1|c(u),c(v)) = \rho_n P_{c(u), c(v)},
\end{equation}
where $P$ is a $\kappa\times \kappa$ matrix of probabilities, 
and $\rho_n$ is the overall network sparsity (so that the network has $O(\rho_n n^2)$
edges on average).
As a special case, the \textit{planted-partition} model
considers $P$ as being a matrix with $\tilde{p}$ along its diagonal
and the value $\tilde{q}$ elsewhere, with
$\kappa$ equally balanced communities, so
$\pi = (\kappa^{-1}, \ldots, \kappa^{-1})$. We will denote such a
model by $\text{SBM}(n, \kappa, \tilde{p}, \tilde{q}, \rho_n)$.

The most widely studied extension of the SBM is to incorporate 
a degree correction, equipping each node with a non negative
degree parameter $\theta_u$ drawn from some distribution independently
of the community assignments \citep{airoldi_mixed_2008}.
This alters the previous model, instead giving
\begin{equation}
\mathbb{P}(A_{uv}=1|c(u), c(v), \theta_u,\theta_v) = \rho_n \theta_u \theta_v P_{c(u), c(v)}.
\end{equation}
Degree corrected SBM models can be more appropriate for modeling the 
degree heterogeneity seen within communities in real world network
data \citep{karrer_stochastic_2011}.

Performance of stochastic block models is assessed in terms 
of their ability to recover the true community assignments of the
nodes in a network, from the observed adjacency matrix $A$. 
Given an estimated community assignment vector $\hat{\mathbf{c}}\in [\kappa]^n$
and the true communities $\mathbf{z}$ then we can compute
the agreement between these two assignment vectors, up to 
a relabeling of $\mathbf{c}$, as 
\begin{equation}
    L(\widehat{\mathbf{c}}, \mathbf{c}) = \min_{\sigma \in \mathrm{S}_{\kappa}} \frac{1}{n} \sum_{i=1}^n \mathbbm{1}\big[ \widehat{c}(i) \neq \sigma(c(i)) \big]
\end{equation}
where $S_{\kappa}$ denotes the symmetric group of permutations $\sigma : 
[\kappa] \to [\kappa]$. We can also control the
worst-case misclassification rate across all the
different communities. If $\mcC_k$ is the set of nodes belonging to
community $k$, then this is defined as 
\begin{equation}
\widetilde{L}(\widehat{\mathbf{c}}, \mathbf{c}) := \max_{k \in [\kappa]} 
    \min_{\sigma \in S_\kappa} \frac{1}{|\mcC_k|} \sum_{i \in \mcC_k}
    \mathbbm{1}\big[ \widehat{c}(i) \neq \sigma(k) \big].
\end{equation}

Guarantees of the form $L(\widehat{\mathbf{c}}, 
\mathbf{c}) = o_p(1)$ as $n \to \infty$ are known as
\emph{weak consistency} 
guarantees in the community detection literature.
Strong consistency 
considers the stronger setting where 
$L(\widehat{\mathbf{c}}, 
\mathbf{c})=0$ with asymptotic probability 1.
\citep{abbe_community_2017}
provides a review of results for guarantees of these forms.
In this work we consider only the weak consistency setting; we
highlight that stricter assumptions are necessary in order
to give these type of guarantees. 

\subsection{Obtaining embeddings from node2vec}

Machine learning methods such as node2vec aim to obtain
an embedding of each node in a network. In general,
for each node $u$ two $d$-dimensional embedding vectors are learned,
a centered representation $\omega_i \in \mathbb{R}^d$ and a context representation
$\widehat{\omega}_i \in \mathbb{R}^d$.
node2vec modifies the simple random walk considered 
in DeepWalk \citep{perozzi_deepwalk_2014},
incorporating tuning parameters $p,q$ which encourage the
walk to return to previously sampled nodes or transition to 
new nodes. Formally, this is defined by sampling concurrent
pairs of vertices in the second-order random walk
$(X_n)_{n \geq 1}$ defined via
\begin{equation}
    \mathbb{P}\big( X_n = u \,|\, X_{n-1} = s, X_{n-2} = v \big)
    \propto \begin{cases}
        0 & \text{ if } (u, s) \not\in \mcE, \\ 
        1/p & \text{ if } d_{u, v} = 0 \text{ and } (u, s) \in \mcE, \\ 
        1 & \text{ if } d_{u, v} = 1 \text{ and } (u, s) \in \mcE, \\
        1/q & \text{ if } d_{u, v} = 2 \text{ and } (u, s) \in \mcE.
    \end{cases}
\end{equation}
where $d_{u, s}$ denotes the length of the shortest path between $u$ and $s$,
after selecting some initial two vertices. Here we consider the case where 
$(X_0, X_1)$ is drawn uniformly from the set of edges in order to initialize the walk.
We note that when $p = q = 1$, corresponding to DeepWalk, this reduces
down to a simple random walk, in which case the initial distribution samples
a vertex proportionally to their degree.

A negative sampling approach is also used 
to approximate the
computationally intractable loss function, replacing 
$-\log(P(v|u))$ in 
\eqref{n2v_loss} with
\begin{equation}
    -\log\sigma(\langle \omega_u, \widehat{\omega}_v\rangle)
    -\sum_{l=1}^{L}\log\sigma(-\langle \omega_u, 
    \widehat{\omega}_{n_l}\rangle),
    \label{n2v_loss_3}
\end{equation}
where $\sigma(x)= (1+e^{-x})^{-1}$,
the sigmoid function.
The vertices $n_1,\ldots, n_L$ are sampled according to a negative 
sampling distribution, which we denote 
as $P_{ns}(\cdot|u)$.
This is usually chosen as the unigram
distribution,
\begin{equation}
 P(v|u) = \frac{\text{deg}(v)^{\alpha}}{\sum_{v'\in \mathcal{V}}
\text{deg}(v')^{\alpha}},
\label{uni_gram}
\end{equation}
which does not depend on the current location of 
the random walk, $u$.
This unigram distribution has
parameter $\alpha$, which is
commonly chosen as $\alpha =3/4$, as was used by
word2vec \citep{mikolov_distributed_2013}.
Given this, and using \eqref{n2v_loss_3},
the loss considered by node2vec for a random walk of
length $k$ can be written as

\begin{equation}
    =\sum_{j=1}^{k+1}\sum_{i: 0<|j-i| <W} \biggr[
     -\log\sigma(\langle \omega_{v_j}, \widehat{\omega}_{v_i}\rangle) 
    -\sum_{l=1}^{L}\mathbb{E}_{n_l\sim P_{ns}(\cdot|v_i)}
    \log\sigma(-\langle \omega_{v_j}, 
    \widehat{\omega}_{n_l}\rangle)
    \biggr].
\end{equation}
Here we use $\mathbb{E}_{n_l\sim P_{ns}(\cdot|v_i)}$
to denote the procedure to sample a draw from the negative
sampling distribution, with $W=1$ commonly chosen.
Given this loss function, stochastic gradient updates 
are used to estimate the embedding vector for each 
node. This amounts to minimizing an 
empirical risk function (e.g 
\citep{robbins_stochastic_1951, veitch_empirical_2019}), which 
we can write as
\begin{equation}
    \mathcal{L}_{n}(U, V) := 
    \sum_{i \neq j} \Big\{ - \mathbb{P}_n((i, j) \in \mcP) \log( \sigma(  \embedip ) ) 
    - \mathbb{P}_n((i, j) \in \mcN) \log(1 - \sigma( \embedip )) \big\}.
\label{eqn:emp_risk}   
\end{equation}
where $\mathbb{P}_n(\cdot) := \mathbb{P}(\cdot \,|\, \mcG_n)$, and $\mcP = \mcP(\mcG_n) $ and $\mcN = \mcN(\mcG_n)$
are sets of positive and negative samples respectively. We consider a sequence of graphs
$\mcG_n$ with $|\mcV| = n$ and study the behavior of this loss function
when $n$ is large. 
To be explicit, $\mathbb{P}_n((i, j) \in \mcP)$ denotes the
probability (conditional on a realization of the graph) that the vertices $(i, j)$
appear concurrently within a random walk of length $k$,
and $\mathbb{P}_n((i, j) \in \mcN)$ denotes the probability
that $(i, j)$ is selected as a pair of
edges through the negative sampling scheme
(conditional on the random walk process in the first stage). 

The loss depends on two matrices $U,V\in \mathbb{R}^{n\times d}$,
with $u_i, v_j \in \mathbb{R}^d$ denoting
the $i$-th and $j$-th rows of $U$ and $V$ respectively.
The rows of $U$ correspond to the "centered representations"
of each node, while the rows of $V$ correspond to the 
"context representation" (borrowing the terminology used by e.g Word2Vec). 
In practice we can constrain the embedding vectors $u_i$ and $v_i$ to be equal 
if we wish; we will consider both approaches in this paper. (If these
are not constrained to be equal, the centered representation is commonly used 
for downstream tasks.)
We highlight Equation~\eqref{eqn:emp_risk} is defined only as a function of
$UV^T$. There are two potential approaches to
deal with this. We can regularize the objective function to enforce
$U^TU = V^T V$, which does not change the matrix $UV^T$ that we recover \citep{zhu_global_2021}.
Alternatively, if these matrices are initialized to be balanced then they will
remain balanced during the gradient descent procedure \citep{ma_beyond_2021}.
Either procedure can be used to implicitly enforce $U^T U = V^TV$, which 
reduces the symmetry group of $(U,V)\rightarrow UV^T$ to the orthogonal group.
Similarly, if we constrain $U=V$ then we obtain the same reduction. 



\subsection{Using embeddings for community detection}

Having learned embedding
vectors $\omega_i$ for each node, we seek to use them for a further task, such as node clustering
or classification.
For community detection a natural procedure is to
perform k-means clustering on the embedding vectors,
using the estimated cluster assignments
as inferred communities. k-means clustering \citep{hartigan_kmeans_1979}
aims to find $k$ vectors $x_1, \ldots, x_k \in \mathbb{R}^d$ which
minimize the within cluster sum of squares.
This can be formulated in terms of a matrix 
$X \in \mathbb{R}^{k \times d}$ and a membership matrix 
$\Theta\in\{0,1\}^{n\times k}$
where each row of $\Theta$ has exactly $k-1$ zero entries.
Then the k-means clustering objective can be written as
\begin{equation}
    \mathcal{L}_{\text{k-means}}(\Theta, X) = \frac{1}{n} \| \widehat{\Omega} - \Theta X \|_F^2
\end{equation}
where $\widehat{\Omega} \in \mathbb{R}^{n \times d}$ is the matrix
whose rows are the $\widehat{\omega}_i$.
The non-zero entries in each row of $\Theta$ gives the
estimated community assignments.
Finding exact minima 
to this minimization problem is NP-hard in general 
\citep{aloise_np-hardness_2009}.
For theoretical purposes, we will give guarantees for
any $(1+\epsilon)$-minimizer to the above problem, which returns any pair
$(\widehat{\Theta}, \widehat{X})$ for which $\mathcal{L}_{\text{k-means}}
(\widehat{\Theta}, \widehat{X}) \leq (1 + \epsilon) \min_{\Theta, X} 
\mathcal{L}_{\text{k-means}}(\Theta, X)$, and can be solved efficiently
\citep{kumar_linear_2005}.


\section{Results}
\label{sec:results}

Within this section, we give theoretical results which allow us to describe
what happens when we use node2vec to learn embedding vectors for each node
in the network, and then use these as features for a k-means clustering algorithm
to perform community detection. Throughout, we assume that we observe a sequence of graphs
$(\mcG_n)_{n \geq 1}$ on $n$ vertices drawn from a probabilistic model and fit a node2vec model,
according to one of the three scenarios below:
\begin{enumerate}[label=(\roman*)]
    \item We use DeepWalk ($p = q =1$ in node2vec), and the 
    graph is drawn according to a SBM with $\rho_n \gg \log(n)/n$;
    \item We use node2vec, and the graph is drawn according to a SBM with $\rho_n = n^{-\alpha}$
    for some $\alpha < \alpha'$, where $\alpha'$ depends on node2vec's hyperparameters;
    \item We use DeepWalk and a unigram parameter of $\alpha = 1$, and the graph is drawn according to a DCSBM with $\rho_n \gg \log(n)/n$ where the degree heterogeneity
    parameters $\theta_u \in [C^{-1}, C]$ for some $C < \infty$.
\end{enumerate}
All probabilistic statements below are with respect to the joint law of $\mcG_n$ and the 
sampling which occurs to form the node2vec loss. All 
proofs are deferred to the Appendix. There we also provide extensions for the tasks of
node classification and link prediction.

\subsection{Asymptotic distribution of the embeddings}

We begin with a result which describes the asymptotic distribution
of the gram matrices formed by the embeddings which 
minimize the loss $\mcL_n(U, V)$ over matrices $U, V \in
\mathbb{R}^{n \times d}$. 

\begin{theorem}
    \label{main_paper:thm:gram_converge:gram_converge}
    There exist constants $\tilde{A}_{\infty}$ and $\tilde{A}_{2, \infty}$
    (depending on $\pi, P$ and the sampling scheme)
    and a matrix $M^* \in \mathbb{R}^{\kappa \times \kappa}$
    (also depending on $\pi, P$ and the sampling scheme) such that 
    when $d \geq \mathrm{rk}(M^*)$, 
    for any minimizer $(U^*, V^*)$ of $\mcL(U, V)$ over $X \times X$
    where
    \begin{equation*}
        X = \{ U \in \mathbb{R}^{n \times d} \,:\, \| U \|_{\infty} \leq \tilde{A}_{\infty},
        \| U \|_{2, \infty} \leq \tilde{A}_{2, \infty} \},
    \end{equation*}
    we have that
    \begin{equation*}
        \frac{1}{n^2} \sum_{i, j \in [n]}
        \big( \langle u_i^*, v_j^* \rangle - M^*_{c(i), c(j)} \big)^2
        = C \cdot \begin{cases} O_p( (\tfrac{ \max\{\log n, d\} }{n \rho_n} )^{1/2} ) 
        & \text{under scenarios (i) and (iii);} \\ 
            o_p(1) & \text{under scenario (ii);}
        \end{cases}
    \end{equation*}
    where $C$ is a constant depending on the (DC)SBM parameters, the node2vec hyperparameters, 
    $\tilde{A}_{\infty}$ and $\tilde{A}_{2, \infty}$. In the case
    where we constrain $U = V$ within node2vec, 
    the same result holds under scenarios i) and ii). Moreover, 
    under all scenarios we can allow the number of communities $\kappa$ to grow
    with $n$ - provided
    $\kappa = o(n \rho_n)$ - and still maintain consistency as $n \to \infty$. 
\end{theorem}


To give some intuition, we describe the form of $M^*$
when the graph arises from a SBM$(n, \kappa, \tilde{p}, \tilde{q}, \rho_n)$ model
when using DeepWalk. 
In this case, we show in the Appendix that
\begin{equation*}
    M^*_{lm} = \alpha^* \delta_{lm} + \beta^* (1 - \delta_{lm}) \text{ for } l, m \in [\kappa]
\end{equation*}
for some constants $\alpha$ and $\beta$
and $\delta_{lm}$ is the Kronecker delta. In the unconstrained case
we have that
\begin{gather}
    \alpha^* = \log\Big( \frac{1}{1 + k^{-1}} \cdot 
    \frac{\kappa \tilde{p} }{ \tilde{p} + (\kappa - 1) \tilde{q} } \Big), \quad 
    \beta^* =  \log\Big( \frac{1}{1 + k^{-1}} \cdot 
    \frac{\kappa \tilde{q} }{ \tilde{p} + (\kappa - 1) \tilde{q} } \Big).
\end{gather}
In the constrained case we instead have that
$\beta^* = -\alpha^*/(\kappa - 1)$, and that $\alpha^*$ is a function of
$p / q$ which is non-negative iff $p > q$, and equals zero when
$p \leq q$. With regards to the constants $\tilde{A}_{\infty}$ and $\tilde{A}_{2, \infty}$, 
we have that $\| M^* \|_{\infty} \leq O(|\log(p/q)|)$.
Additionally, it is possible to write $M^* = U_M^* (V_M^*)^T$
where $\| U_M^* \|_{2, \infty}$ and 
$\| V_M^* \|_{2, \infty}$ are upper bounded by
$O(|\log(p/q)|^{1/2})$. In particular, this means
that $\tilde{A}_{\infty}$ and $\tilde{A}_{2, \infty}$ do not have any
implicit dependence on $n$ or $\kappa$, and so the constant
in Theorem~\ref{main_paper:thm:gram_converge:gram_converge}
is not affecting the rate here.


While Theorem~\ref{main_paper:thm:gram_converge:gram_converge}
gives guarantees from the gram matrices formed by the
embeddings, in practice we want guarantees for the actual
embedding vectors themselves. For convenience we suppose
that the embedding dimension $d$ is chosen exactly to be
the rank of $M^*$; upon doing so, we can then
obtain guarantees for the embedding vectors themselves.
We recall that in the unconstrained case, we implicitly
suppose that we find embedding matrices $U^*$ and $V^*$
which are balanced in that
$(U^*)^T U^* = (V^*)^T V^*$.

\begin{theorem}
    \label{thm:gram_converge:embed_converge}
    Suppose that the conclusion of
    Theorem~\ref{main_paper:thm:gram_converge:gram_converge} holds, and further suppose that $d$
    equals the rank of the matrix $M^*$. Then there exists a matrix
    $\widetilde{U}^* \in \mathbb{R}^{\kappa \times d}$ such that
    \begin{equation}
        \min_{Q \in O(d)} \frac{1}{n} \sum_{i=1}^n 
        \| u_i^* - \widetilde{u}_{c(i)}^* Q \|_2^2 = C \cdot \begin{cases} O_p( (\tfrac{ \max\{\log n, d\} }{n \rho_n} )^{1/2} ) 
        & \text{under scenarios (i) and (iii);} \\ 
            o_p(1) & \text{under scenario (ii);}
        \end{cases}
    \end{equation}
\end{theorem}

\subsection{Guarantees for community detection}

With Theorem~\ref{thm:gram_converge:embed_converge}, we are
now in a position to give guarantees for machine learning
methods which use the embeddings as features for 
a downstream task. We only discuss using the embeddings for clustering;
in Appendix~\ref{app:sec:ml:node_classification} we discuss what
can be said for other downstream tasks.

\begin{theorem}
    \label{thm:ml:kmeans_embed}
    Suppose that we have embedding vectors $u_i^* \in \mathbb{R}^d$ 
    for $i \in [n]$ such that
    \begin{equation}
        \label{eq:thm:ml:kmeans_embed:condition}
        \min_{Q \in O(d)} \frac{1}{n} \sum_{i=1}^n \|u_i^* - \widetilde{u}_{c(i)}^* Q \|_2^2
        = O_p(r_n)
    \end{equation}
    for some rate function $r_n \to 0$ as $n \to \infty$ 
    and vectors $\eta_l \in \mathbb{R}^d$ for $l \in [\kappa]$.
    Moreover suppose that $\delta := \min_{l \neq k} \| \widetilde{u}^*_l - 
    \widetilde{u}^*_k \|_2 > 0$.
    Then if $\hat{\mathbf{c}}(i)$ is the community assignment of node $i$
    produced by applying
    a $(1+\epsilon)$-approximate k-means clustering 
    with $k = \kappa$ to the matrix 
    whose columns are the $u_i^*$, we have that $L(\mathbf{c}, \hat{\mathbf{c}}) = O_p(
        \delta^{-2} r_n)$
    and $\widetilde{L}(\mathbf{c}, \hat{\mathbf{c}}) = O_p(\delta^{-2} r_n)$. In the case
    where the RHS of \eqref{eq:thm:ml:kmeans_embed:condition} is only $o_p(1)$ instead,
    then instead $L(\mathbf{c}, \hat{\mathbf{c}})$ and $\widetilde{L}(\mathbf{c}, \hat{\mathbf{c}})$ 
    are $\delta^{-2} o_p(1)$.
\end{theorem}

Within the SBM$(n, \kappa, \tilde{p}, \tilde{q}, \rho_n)$
model, we can show in the unconstrained case
that $\delta^2 = \Theta(|\log(\tilde{p}/\tilde{q})|)$, and
in the constrained case that $\delta^2 = \Theta((\tilde{p}/\tilde{q}))$.
As a result, this suggests that as $\tilde{p}/\tilde{q}$ approaches $1$,
the task of distinguishing the communities becomes more
difficult. This is inline with basic intuition, along
with our experimental results in Section~\ref{sec:exps}.
We note that, due to the nature of the embedding vectors,
for any proportion of vertices arbitrarily close to 1, the nodes will, 
with high probability for sufficiently large $n$,
be separated in the embedding space according to their community assignments.
This separation allows clustering methods, such as DBSCAN,
to accurately recover the communities of these nodes also.


Recall that from the discussion before, we know that $M^*$
equals the zero matrix in the constrained regime
when $\tilde{p} \leq \tilde{q}$ (and therefore the embeddings asymptotically
contain no information about the network). As in the
case where $\tilde{p} > \tilde{q}$ we can show that $\delta > 0$,
we get the immediate corollary.

\begin{corollary}
    Under scenario (i), suppose the embedding vectors 
    learned through the node2vec loss are obtained by constraining
    the embedding matrices $U = V$. Then the embeddings can be used for
    weakly consistent recovery of the communities if and only if $\tilde{p} > \tilde{q}$.
\end{corollary}

As a result, the constrained model can be disadvantageous 
if used without a-priori knowledge of the network beforehand
(in that within-community connections outnumber between-community
connections), even though it avoids interpretability issues
about which embedding vector should be used as single
representation for the node.

\section{Experiments}
\label{sec:exps}
In this section we provide simulation and real data
experiments
to empirically validate the previous theoretical results.
We demonstrate the performance, in terms of community detection, of 
k-means clustering of the embedding vectors learned by node2vec, for both the regular and 
degree corrected stochastic block model. We also investigate
the role of the negative sampling parameter $\alpha$
and the node2vec tuning parameters $p$ and $q$, before examining performance
on a real network with known community structure.

We first simulate data from the planted partition stochastic block model, 
$\text{SBM}(n/\kappa,\kappa, \tilde{p}, \tilde{q}, \rho_n)$.
We consider $\tilde{q} = \tilde{p}\beta$ for 
a range of values of $\beta \ll 1$, 
giving varying strengths of associative community structure.
In each setting we vary both the number of true 
communities present and the number of nodes in each community, 
considering $n=200$ to $n=5000$
and $K=2,3,4,5$. 
We use node2vec to construct 
an embedding of the nodes in the network.
\footnote{We use the implementation of 
node2vec available at \url{https://github.com/eliorc/node2vec}
without any modifications.}
We use an embedding dimension 
of 64 and do not modify other default tuning parameters for the 
embedding procedure unless specified,
so that $p=q=1$. 
We
investigate the role of these tuning parameters below, allowing them to vary 
as is considered in node2vec.
We pass these embedding vectors into k-means clustering, where $k=\kappa$, 
the true 
number of communities present in the network. 
This estimates a community 
assignment for each of 
the 
nodes in the network.

\begin{figure}[ht]
    \centering
    \begin{subfigure}{0.495\textwidth}
       \centering 
       \includegraphics[width=\textwidth]{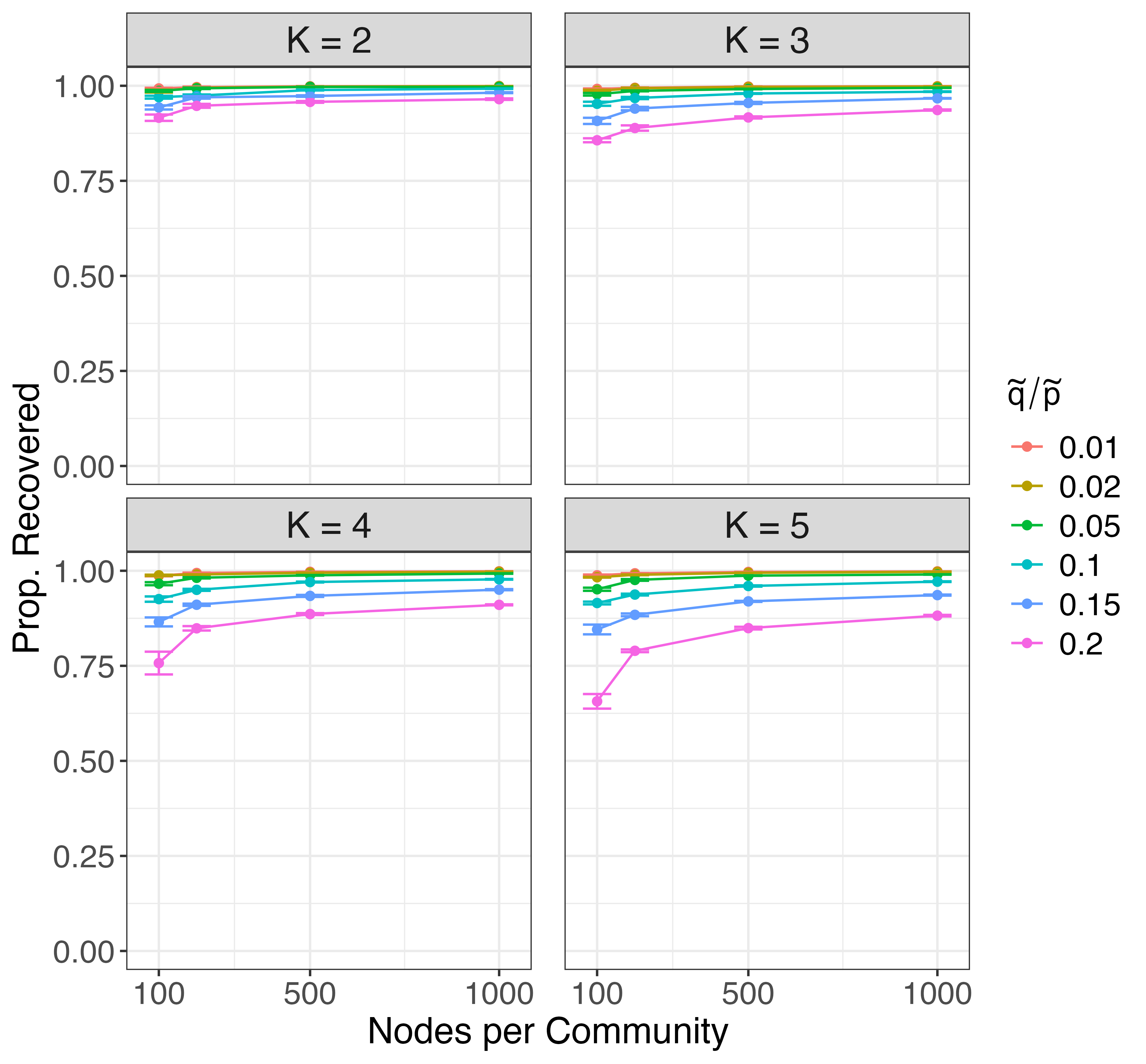}
       \caption{A relatively sparse SBM.}
    \end{subfigure}
    \hfill
    \begin{subfigure}{0.495\textwidth} 
    \includegraphics[width=\textwidth]{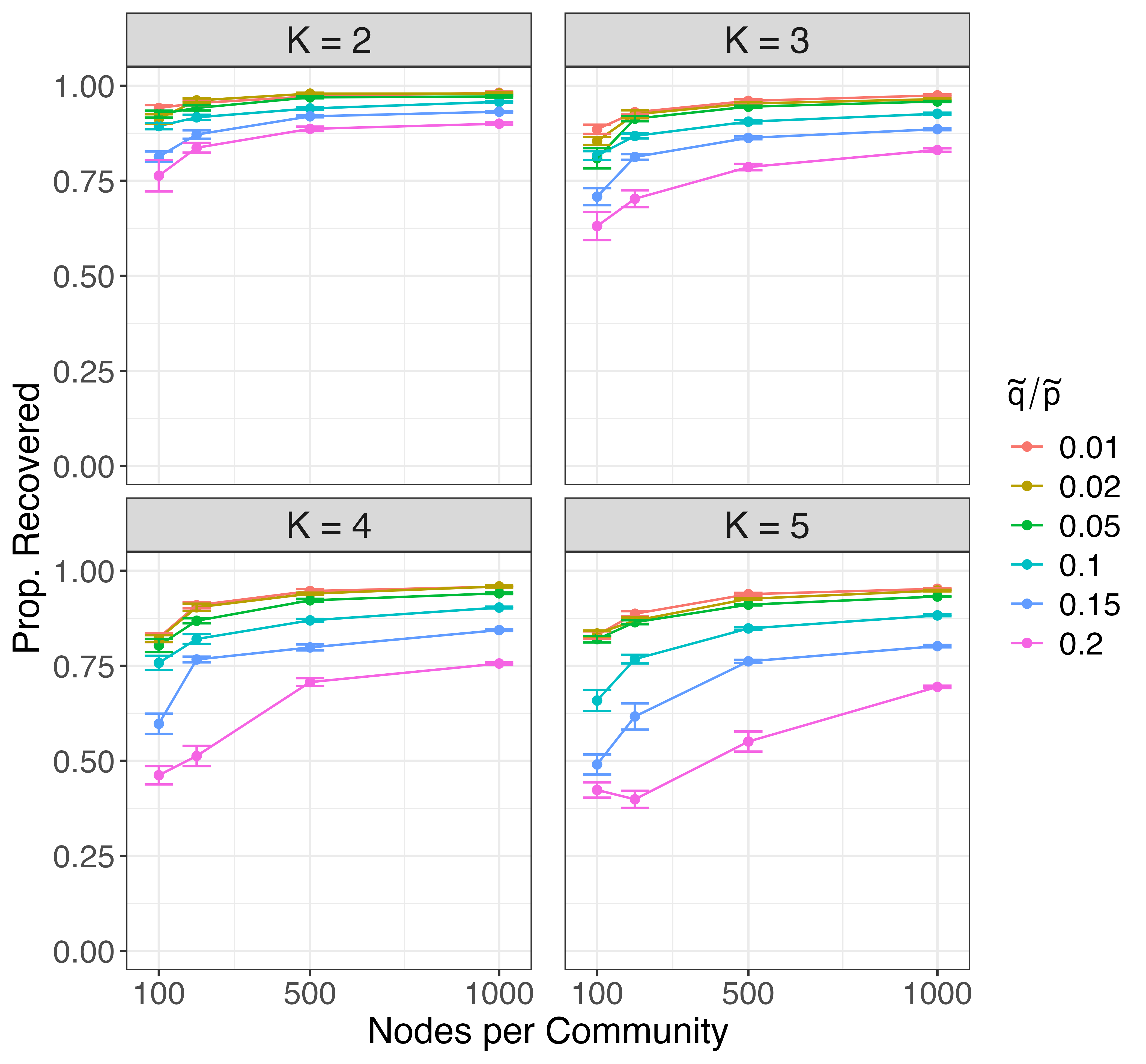}
    \caption{A degree corrected relatively sparse SBM.}
    \end{subfigure}
    \caption{Proportion of nodes correctly recovered for both the regular and degree
    corrected relatively sparse SBM.}
    \label{fig:rsparse_pcr}
\end{figure}

To evaluate the performance of our procedure, we compute the proportion of 
nodes correctly classified, up to permutation of the community assignments. For each simulation 
setting we perform 10 replications. We show the resulting estimates in 
Figure~\ref{fig:rsparse_pcr}(a), 
for the relatively sparse setting where $\rho_n = \log(n)/n$. For all settings, the 
proportion of nodes assigned to the correct community by k-means clustering of the node2vec 
embeddings is high, particularly when the ratio of the between to within community edge 
probabilities, $\beta$, is small. As expected, as we increase the number of nodes in the 
network, a larger proportion of nodes are correctly recovered. 
We examine the empirical rate of convergence of this procedure in the Appendix. 
This appears to be  approximately super-linear for dense networks 
and sub-linear for relatively sparse networks. Compared to 
the results of \cite{zhang2023fundamental},
this indicates that node2vec may be supoptimal.
In the Appendix we also show community recovery using
normalized mutual information (NMI) \citep{danon2005comparing}.
We also see good performance.

We can similarly evaluate the performance of node2vec for data generated from a degree corrected
SBM (DC-SBM). To generate such networks we modify the simulation setting used
by \citep{gao-dcsbm}. We generate the degree correction parameters 
$\theta_u = |Z_u|+1 - (2\pi)^{-1/2}$ where $Z_u\sim N(0, \sigma = 0.25)$ and incorporate
these 
into the $\text{SBM}(n/\kappa,\kappa,\tilde{p},\tilde{q}, \rho_n)$ considered previously. 
Two nodes $u$ and $v$ in the same community will have connection probability
$\theta_u \theta_v \rho_n \tilde{p}$ while for nodes in different communities it will be
$\theta_u \theta_v \rho_n \tilde{q}$.
We again learn an embedding of the nodes using a default implementation of 
node2vec and cluster these embedding vectors using k-means
clustering. We show the corresponding results, in terms of the proportion of the nodes 
assigned to their correct communities under this setting in 
Figure~\ref{fig:rsparse_pcr}(b).
As expected, the degree corrections make community recovery somewhat
more challenging however as
we increase the number of nodes in the network, we are able to correctly recover a 
high proportion of nodes.


We next wish to examine empirically the role of the unigram 
parameter
$\alpha$ of Equation~\eqref{uni_gram},
and how this affects community detection. While the 
previous theoretical 
results require $\alpha=1$ for weak consistency
of community recovery in the DC-SBM,
we investigate if good empirical performance 
is possible with other choices of this parameter. We consider 
the DC-SBM simulation described previously, 
where we now vary $\alpha\in\{-1, 0, 0.25, 0.5, 0.75,1\}$ when learning the node embeddings.
For each of these settings (with all other parameters as before)
we consider the proportion of nodes correctly recovered. 
We show this result for networks with $\kappa=2$ communities in 
Figure~\ref{fig:alpha_rsparse_dcsbm_pcr}.
These experiments indicate 
similar performance for a range of values of $\alpha$. 
Further work is needed to confirm the guarantees do indeed extend 
to these alternative choices of $\alpha$, and we investigate this for real networks in
Section~A of the appendix.

We also investigate the role of the node2vec tuning parameters $p$ and $q$ on performance.
For $\kappa=2$ we consider $\beta=0.01$ and $\beta=0.2$, giving networks with strong
and weak associative community structure respectively.
We simulate from the previous relatively sparse DC-SBM with varying numbers of nodes and
fit node2vec, using $p,q\in\{0.5, 1, 2\}$. As the number of nodes in the network increases all choices of 
$p$ and $q$ give similar good performance for both choices of $\beta$.
This indicates that the impact of these sampling parameters becomes limited as the networks
become sufficiently large.
We provide further discussion and a visualization of this result
in Appendix~A.


\begin{figure}

\begin{minipage}{0.62\textwidth}
    \includegraphics[width=0.95\textwidth]{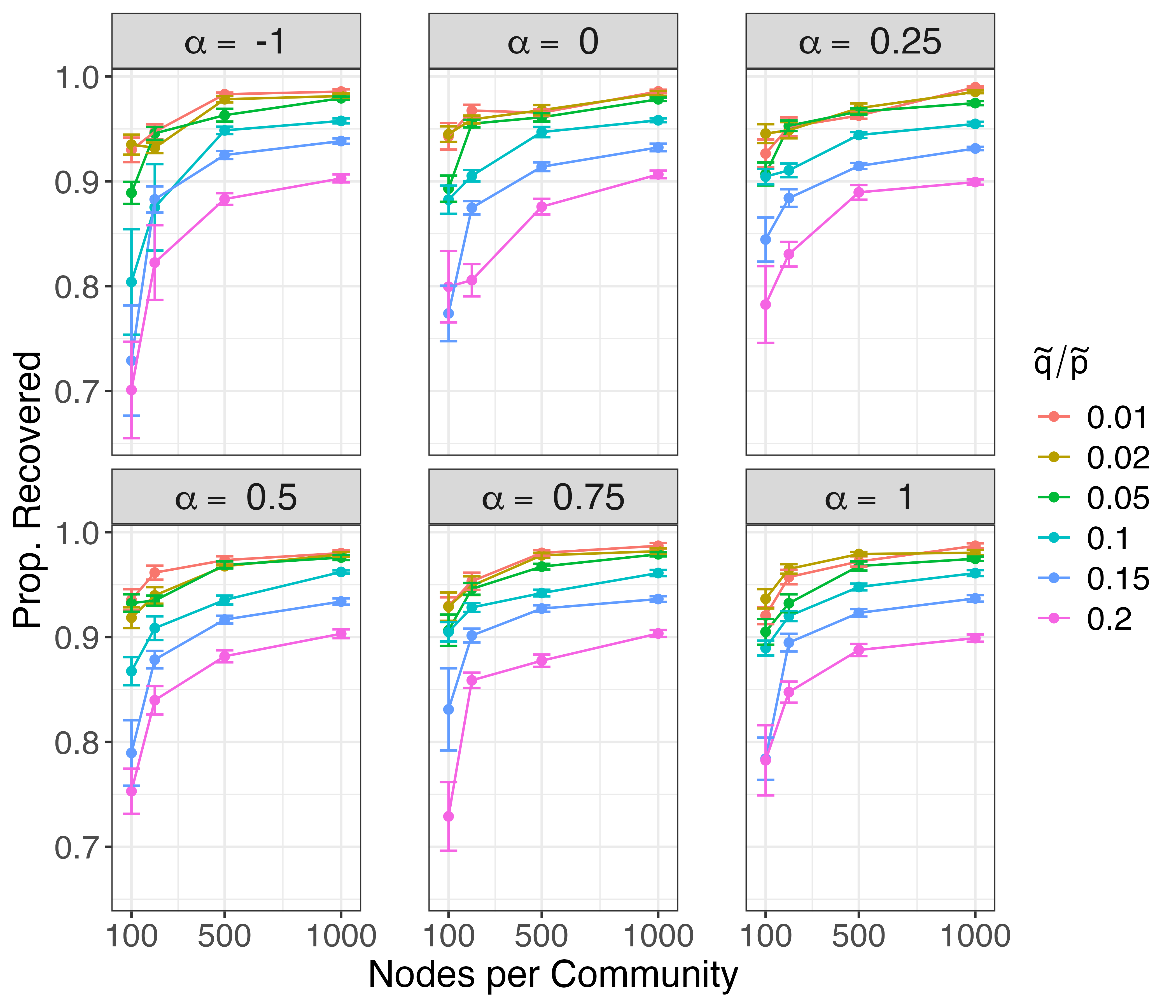}
    \caption{Proportion of nodes correctly recovered as we vary the negative sampling parameter in node2vec with
    mean and one standard error for each setting.
    We see similar performance for each choice of $\alpha$.}
    \label{fig:alpha_rsparse_dcsbm_pcr}
\end{minipage}
\hfill
\begin{minipage}{0.32\textwidth}
        \includegraphics[width=\textwidth]{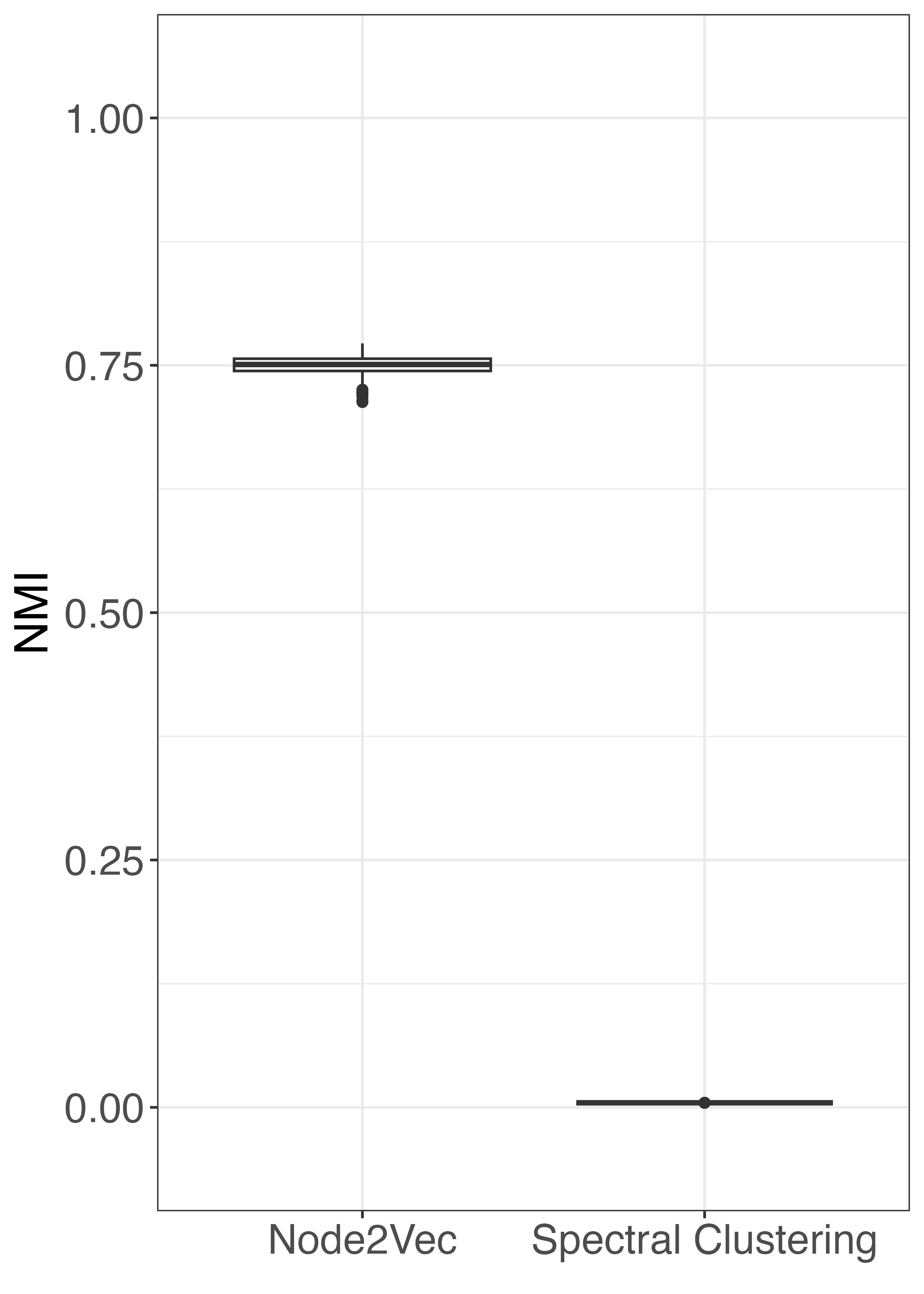}
        \caption{Node2vec with 
        k-means clustering can recover the communities in the political blog data while spectral 
        clustering fails.}
        \label{fig:blogs}
\end{minipage}

\end{figure}


Finally, we briefly examine the performance of our community detection procedure on
the political blog data collected by \citep{polblogs}. As highlighted 
by \citep{karrer_stochastic_2011}, degree heterogeneity makes community 
recovery challenging for methods which do not account for this. We see similar
performance if we cluster using a Gaussian mixture model rather than k-means clustering.
In particular, spectral clustering struggles regardless of the graph Laplacian used.
Our procedure shows excellent community recovery (average
NMI of 0.75) for a range of embedding dimensions and unigram parameter settings as shown
in Figure~\ref{fig:blogs}, with
further details and an additional real network example in Appendix~A.
\section{Conclusion and Future Work}
\label{sec:conclusion}

In this work we consider the theoretical 
properties of node embeddings learned
from node2vec. We show, when the network
is generated from a (degree corrected)
stochastic block model, that the embeddings
learned from DeepWalk and node2vec converge
asymptotically to vectors depending only on
their community assignment. As a result,
we show that K-means clustering
of the node2vec embedding vectors can provide weakly
consistent estimates of the true community assignments
of the nodes in the network. We verify these results empirically
using simulated networks.

There are several important future directions
which can extend this work. One direction is in
extending the recovery results within the degree corrected
SBM to the full range of hyperparmaeters for node2vec, as our simulation
studies indicate that a more general
result may hold. There is also the
matter of increasing the strength of our 
results to give better rates and strongly consistent community
detection; one possible avenue of exploration would be to see
whether our results and the results of \citep{zhang_consistency_2024}
could be combined to achieve this. Another improvement would be to study
the behavior of the random walk on the graph in the sparse regime, 
although this would require a generalization of e.g the result of
\citep{jian2014giantcomponent}. We have also 
not considered the task of estimating $\kappa$, the
number of communities in a SBM model, using the embeddings obtained by
node2vec. This has been considered for alternative approaches 
to community detection, 
(\citep{jin_optimal_2022, le_estimating_2022} are some recent results)
but not in the context of a general embedding of the nodes.
Finally, there is a desire to obtain consistency results
for more recent and complex network embedding methods, such as
\citep{hamilton_inductive_2017} and \citep{velickovic_deep_2018}.


\subsubsection*{References}

\printbibliography[heading=none]


\newpage

\appendix

\begin{center}
\textbf{\Large Supplemental Materials: Community Detection
Guarantees Using Embeddings Learned by Node2Vec}
\end{center}

The Supplementary Material consists of the proofs of the results
stated within the paper, along with some extra discussions which
would detract from the flow of the main paper. We also provide some
additional simulation results relating to node classification, and
community detection when measuring performance in terms of the ARI metric.

\setcounter{equation}{0}
\setcounter{theorem}{0}
\setcounter{figure}{0}
\setcounter{table}{0}
\makeatletter
\renewcommand{\theequation}{S\arabic{equation}}
\renewcommand{\thefigure}{S\arabic{figure}}
\renewcommand{\thelemma}{S\arabic{lemma}}
\renewcommand{\thetheorem}{S\arabic{theorem}}

\section{Additional Experimental Results}

Here we provide additional details describing the experimental results presented in 
the main paper. We also describe additional experiments.
All experiments were run on a computing cluster utilising 4 cores of an Intel E5-2683 v4 Broadwell 2.1GHz CPU 
or similar with 2 GB of memory per core. 
Each individual experimental run required at most 2 hours of computing time.
All experiments, including initial preliminary experiments,
required approximately 25k CPU hours.
All code required to reproduce all results is included in the code repository in the supplemental files.

\paragraph{Additional Simulation, Node Classification}
We provide a simple experiment to support the theoretical results on node classification 
demonstrated 
in Section~D of the appendix.
We simulate data from a $\text{SBM}(n/\kappa,\kappa, \tilde{p}, \tilde{q}, \rho_n)$ as before
with $\tilde{q}=\tilde{p}\beta$ as in the main text. We learn an embedding of each node using node2vec 
with embedding dimension of 64 and all other parameters set at 
their default values.
We then use the true community labels of 10\% of these nodes to 
train a (multinomial) logistic regression classifier, and predict the class
label for the remaining 90\% of nodes in the network.
We examine the performance of this classification tool using
the node2vec embeddings in terms of classification accuracy.
We show these results in Figure~\ref{fig:class_acc} for $\rho_n = \log(n)/n$, with 
10 simulations for each setting, with the mean across these simulations
and error bars indicating one standard error.
This classifier has excellent accuracy at
predicting the labels of other nodes.

\begin{figure}[ht]
    \centering
    \includegraphics[width=0.75\textwidth]{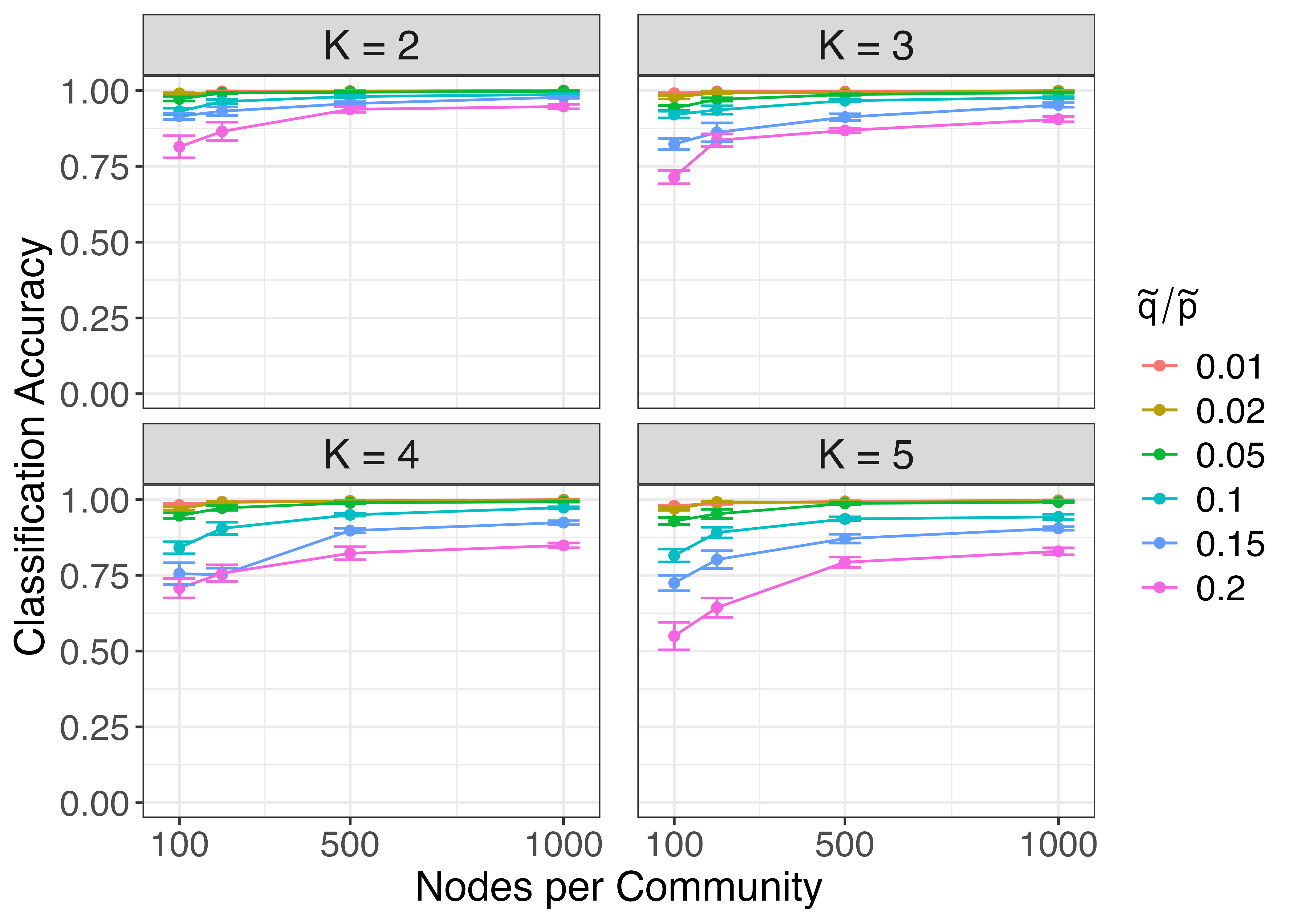}
    \caption{Classification accuracy using 10\% of the node embeddings to learn a
    multinomial logistic regression classifier. Mean and one standard error shown.}
    \label{fig:class_acc}
\end{figure}

\paragraph{Additional Results, Community Detection}
Here we include additional simulation results which were omitted from the
main text. In particular, for the simulations considered in the main manuscript 
we now examine the community recovery performance in 
terms of the normalized mutual information \citep{danon2005comparing}.
We show the average NMI score across these simulations, along with error bars 
corresponding to one standard error.
In each case, the NMI metric is similar to the 
proportion of nodes correctly recovered. As we increase the number of nodes
this performance improves.

\begin{figure}[ht]
    \centering
    \includegraphics[width=0.75\textwidth]{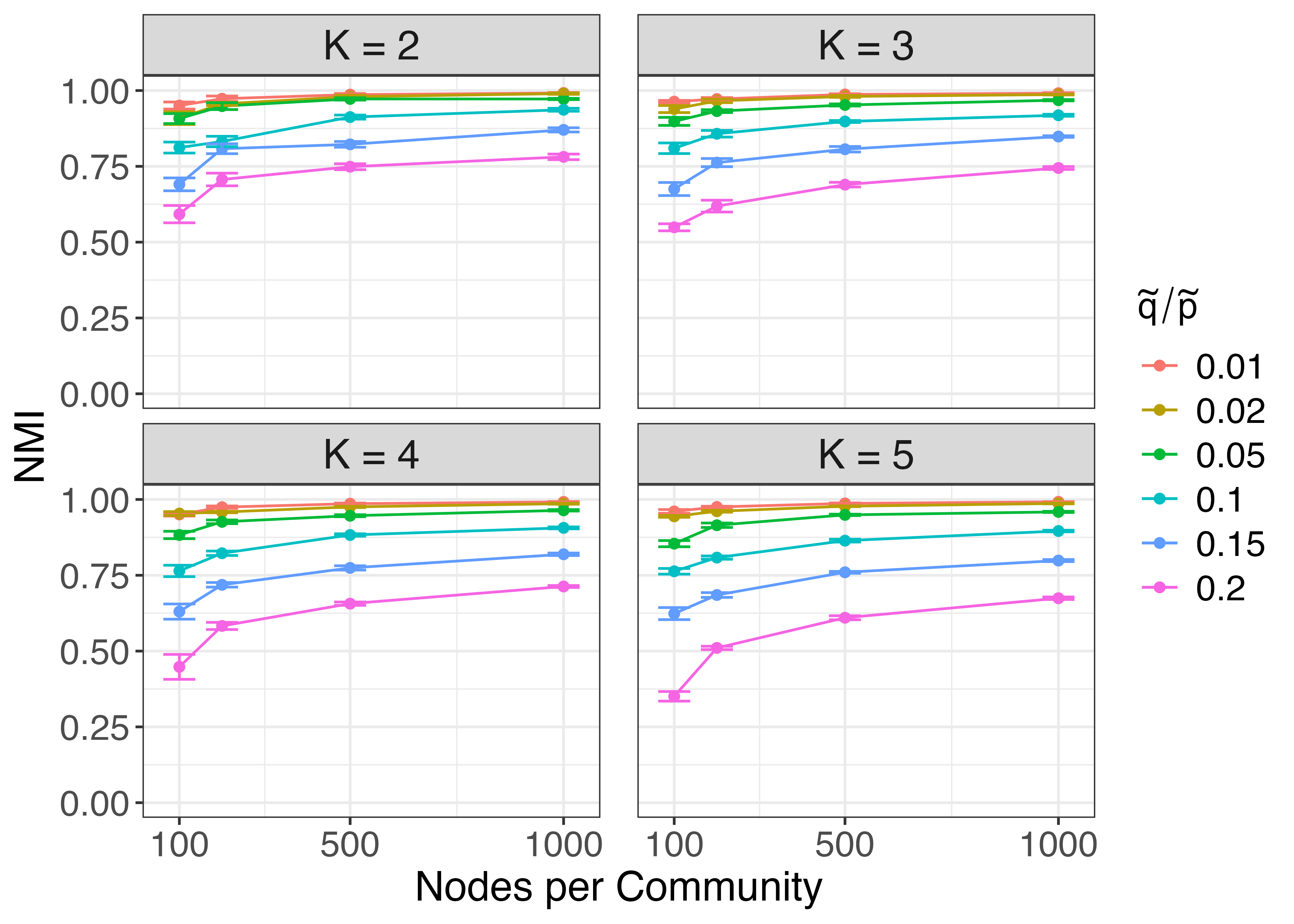}
    \caption{NMI for relatively sparse SBM. Mean and one standard error shown.}
    \label{fig:rsparse_sbm_nmi}
\end{figure}

\begin{figure}[ht]
    \centering
    \includegraphics[width=0.75\textwidth]{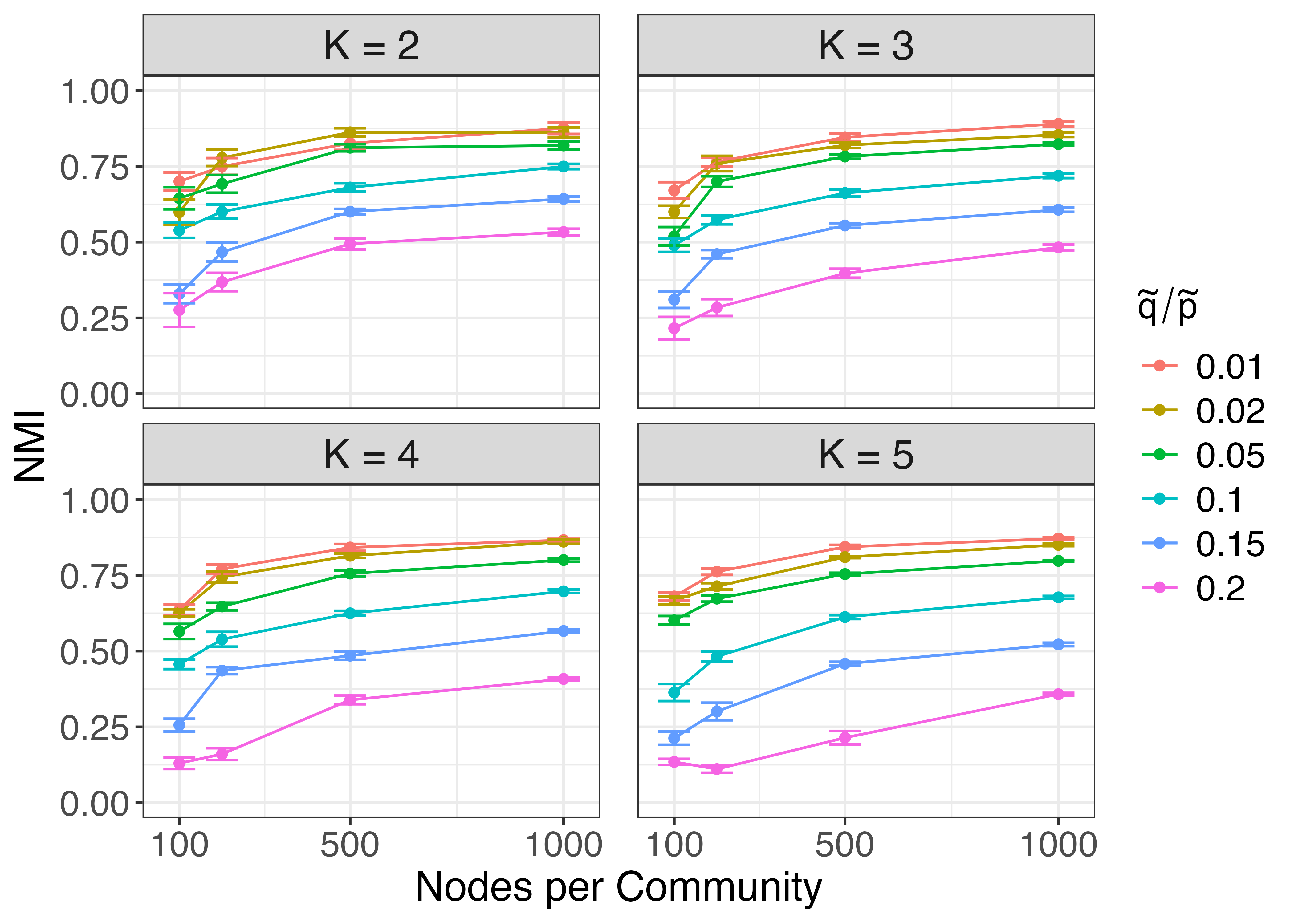}
    \caption{NMI for relatively sparse DC-SBM. Mean and one standard error shown.}
    \label{fig:rsparse_dcsbm_nmi}
\end{figure}

\begin{figure}[ht]
    \centering
    \includegraphics[width=0.75\textwidth]{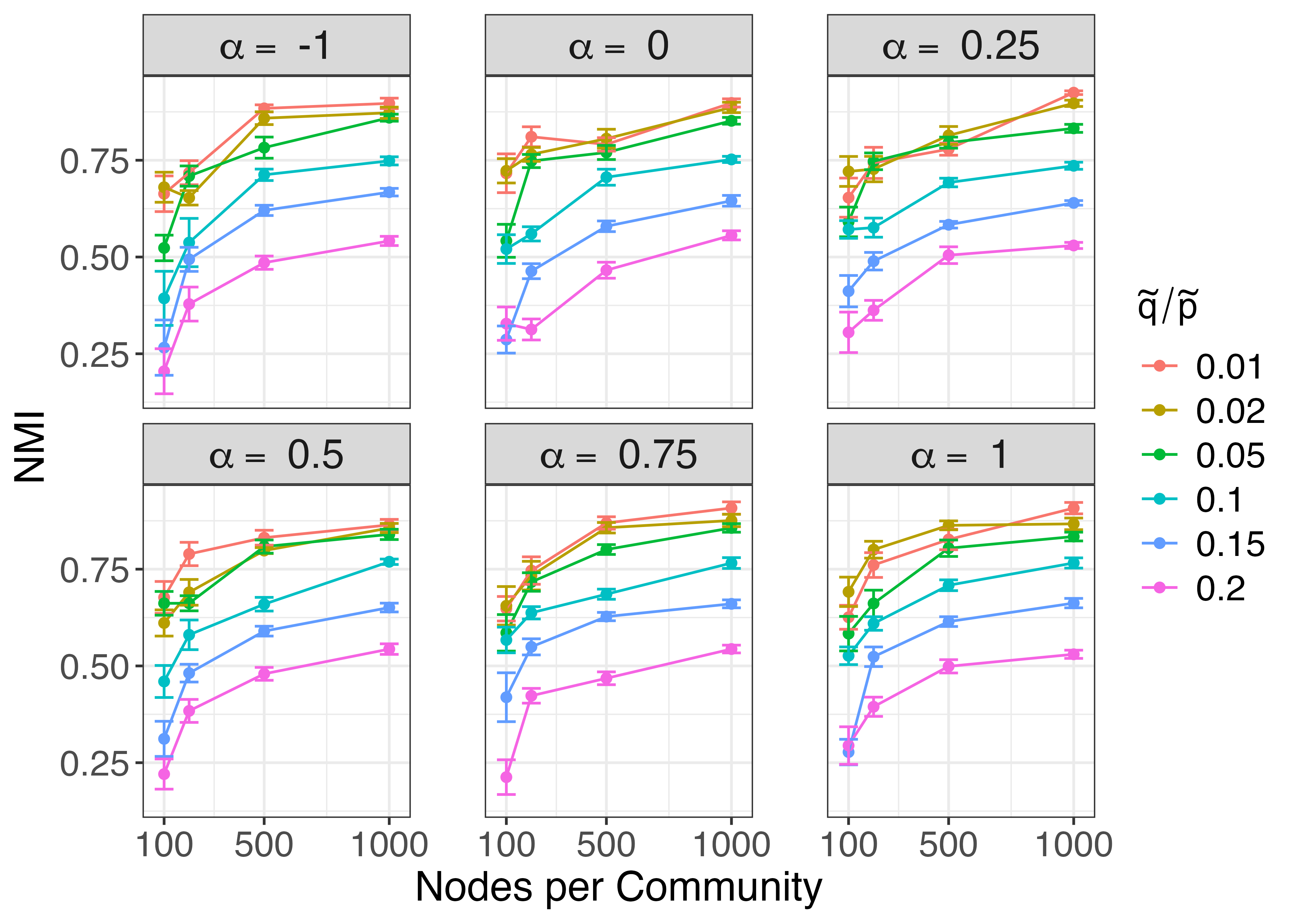}
    \caption{NMI varying $\alpha$ for relatively sparse DC-SBM. 
    Mean and one standard error shown.}
    \label{fig:alpha_rsparse_dcsbm_nmi}
\end{figure}

\paragraph{Rates of Convergence}

We can also investigate the empirical convergence of these methods. Here, we consider the same 
simulated SBM data as above, and examine the convergence in the proportion of nodes 
correctly recovered, as we increase the number of nodes in the network, for $\kappa=2,3,4,5$.
We empirically investigate this convergence using a log-log plot, which is shown
in Figure~\ref{fig:roc_rsparse} for a relatively sparse SBM. 
Our node2vec procedures demonstrates empirical convergence which is super-linear for dense
networks while being sub-linear for relatively sparse networks.

\begin{figure}[ht]
    \centering
    \includegraphics[width=0.75\textwidth]{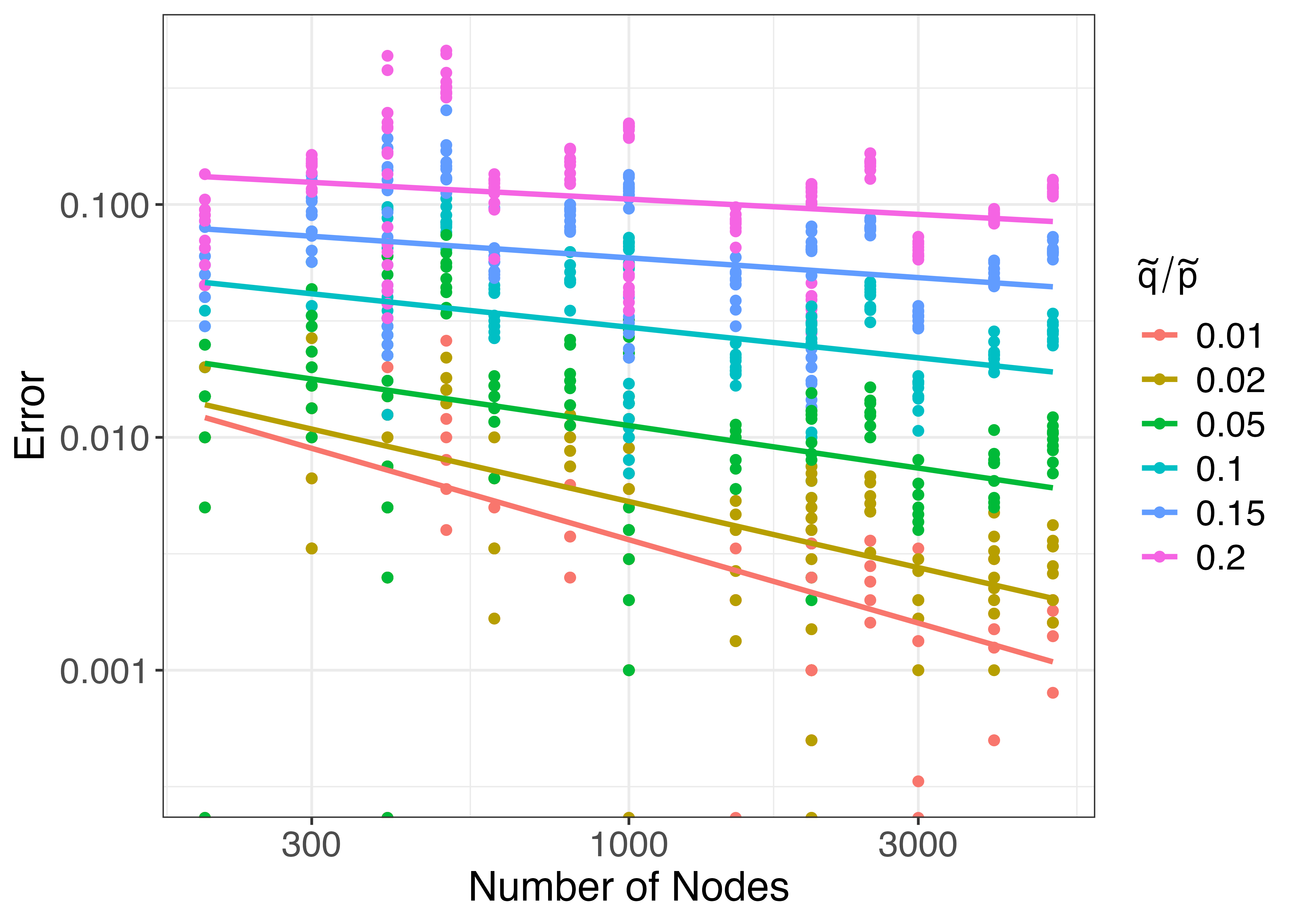}
    \caption{Log-Log plot showing the rate of convergence as we increase the number of nodes in the network. We show a fitted regression for each of the values of $\beta$, showing better convergence
    when the difference between the within and between community edge probabilities is higher.}
    \label{fig:roc_rsparse}
\end{figure}

\paragraph{Varying the node2vec walk parameters}

We also wish to examine the performance of our proposed clustering procedure
when the parameters of the random walk are varied. While $p$ and $q$ are both 
commonly chosen to be $1$, resulting in a simple random walk, other values are possible.
We consider data simulated from the relatively sparse DC-SBM considered previously with $\kappa=2$
communities and consider the within between community probability ratio $\beta = .01$ and
$\beta=0.2$, corresponding to an easier and harder setting to recover the communities respectively.
We then consider $p,q\in \{0.5, 1, 2\}$, the common possible values and vary the number of nodes
in each community as before. For each of these settings we perform community detection using node2vec and spectral
clustering. When $\beta=0.01$ weobtain excellent community recovery for all values of $p$ and $q$, as shown in 
Figure~\ref{fig:pq_pars}(a). When $\beta=0.2$ community recovery is more challenging
for small networks for all values of $p$ and $q$. As the number of nodes increases, Figure~\ref{fig:pq_pars}(b)
shows that all choices of $p$ and $q$ result in good performance.


\begin{figure}[t!]
    \centering
    \begin{subfigure}[t]{0.5\textwidth}
        \centering
        \includegraphics[width=\textwidth]{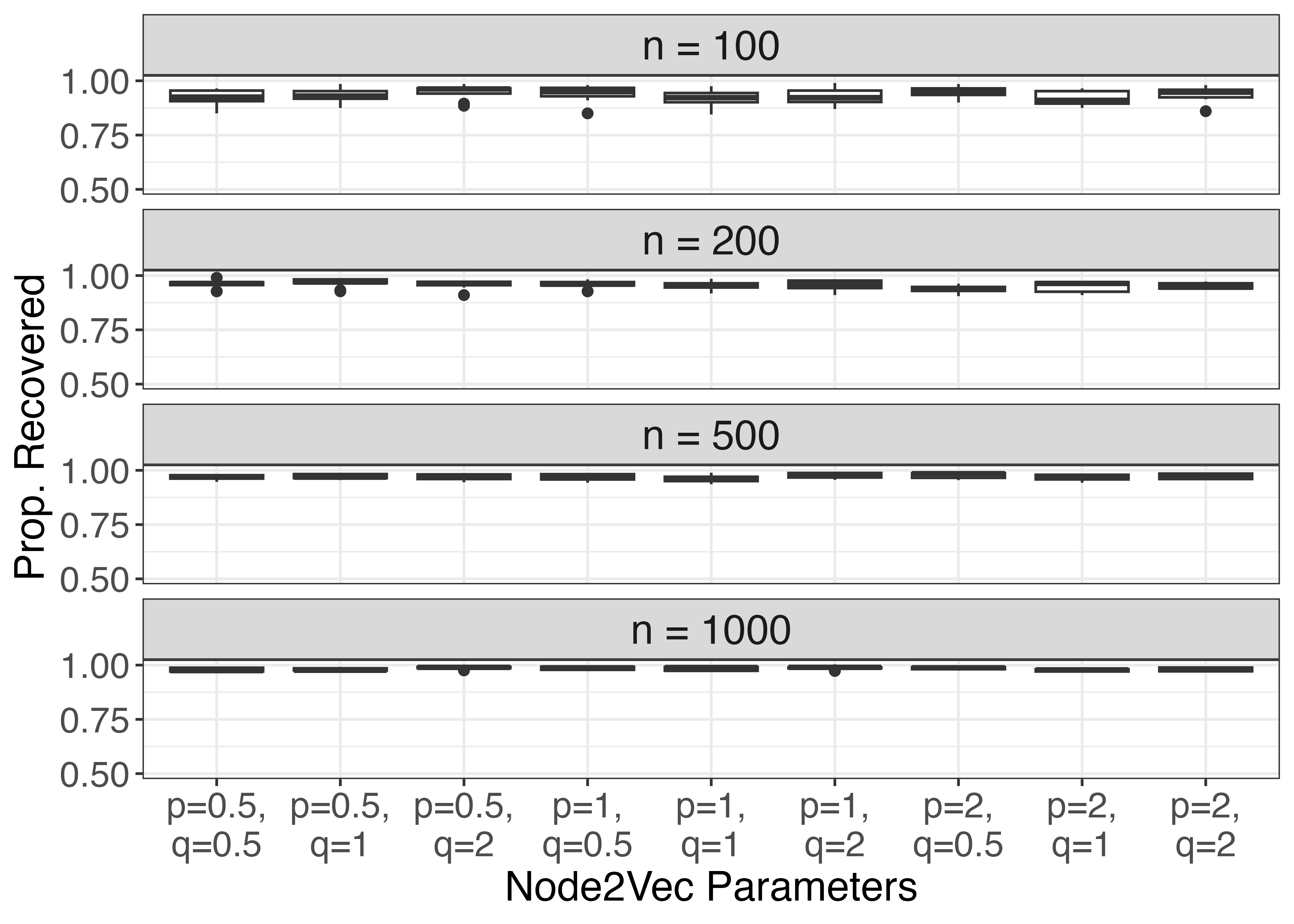}
        \caption{$\beta = 0.01$}
    \end{subfigure}%
    ~ 
    \begin{subfigure}[t]{0.5\textwidth}
        \centering
        \includegraphics[width=\textwidth]{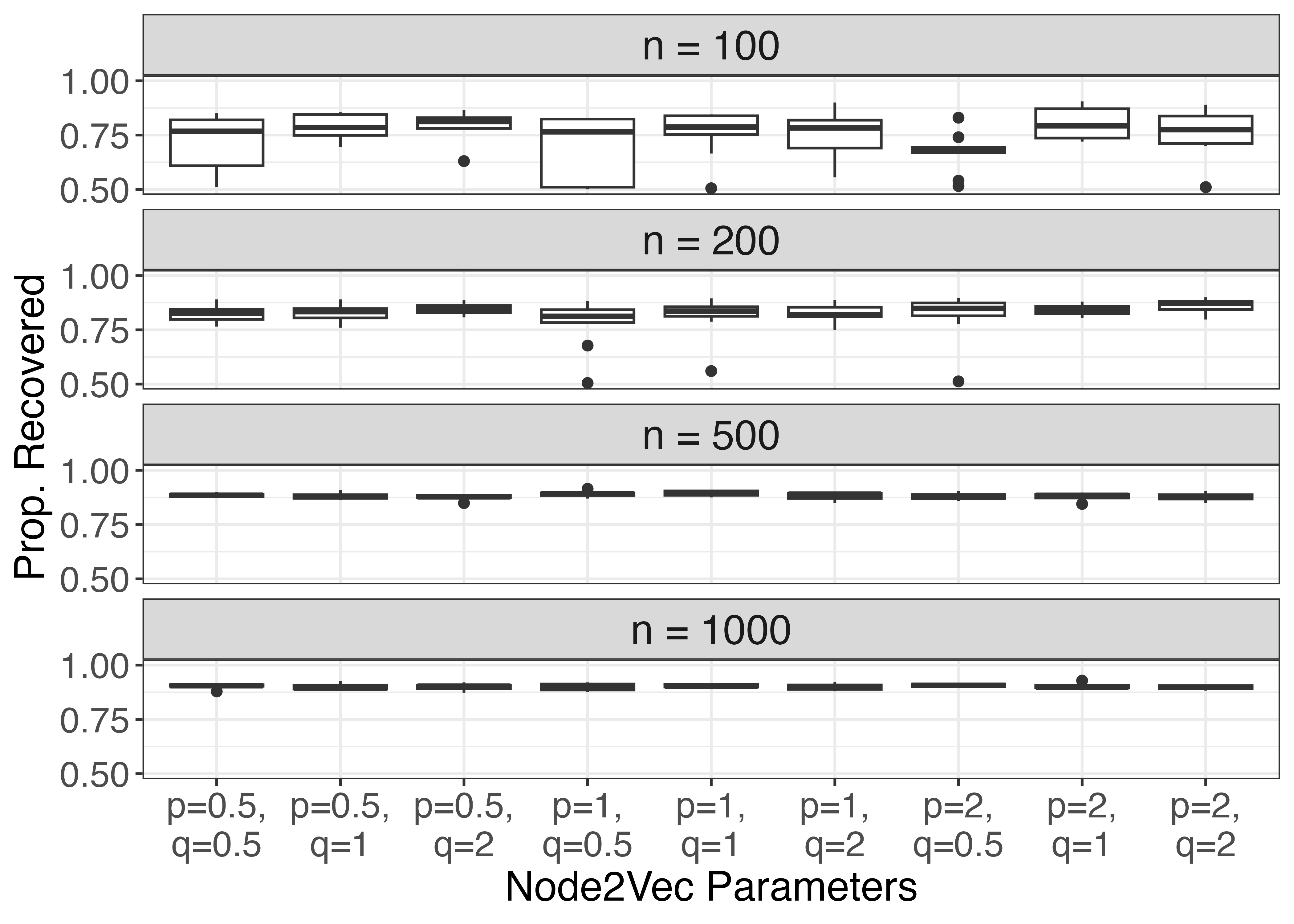}
        \caption{$\beta=0.2$}
    \end{subfigure}
    \caption{Varying the node2vec sampling parameters for DC-SBMs with $\beta=0.01$ (left) and $\beta = 0.2$ (right). Community recovery is harder when $\beta$ is larger and this is seen for all values of $p$ and 
    $q$ for small networks. As the number of nodes increases we get good community recovery for all
    choices of $p$ and $q$.}
    \label{fig:pq_pars}
\end{figure}




\subsection{Performance on Real Networks}
We wish to further examine the performance of this community detection procedure 
for real networks, with known community structure. We also wish to compare 
this procedure to spectral clustering, which is widely used in practice for community detection.
We use two publicly available networks containing known community structure.
We first consider a network
of emails between 1005 members of a large 
research institution, available as part of the Stanford Network
Analysis Project \citep{snapnets}.
There are 25571 directed edges between the nodes in this network,
with known ground truth communities consisting of 42 departments present in this institution.
We also consider a widely used dataset of directed edges between 1490 U.S political blogs,
collected before the 2004 elections \citep{polblogs}. Here the directed 
edges correspond to hyperlinks, with ground truth communities corresponding
to whether the blogs has been identified as liberal or conservative.

For each of these datasets we compare the community recovery of Node2Vec and traditional spectral clustering,
using the normalized graph Laplacian.
As is common in the literature, we remove the direction from these edges and take the largest connected component,
forming symmetric adjacency matrices with 986 and 1222 nodes respectively.
We then use the previously described procedure to perform community detection using Node2Vec. We consider a range
of embedding dimensions ($d=16,32,64,128,256$) and unigram sampling parameter
($\alpha=-1,0.0, 0.25, 0.5, 0.75, 1.0$), while keeping all
other parameters fixed at the defaults considered before. With the true number of communities
known, we then compare the estimated communities from 10 simulations for each of these parameter settings,
along with performing 10 simulations of spectral clustering for each of these settings.

\begin{figure}[ht]
    \centering
    \includegraphics[width=0.75\textwidth]{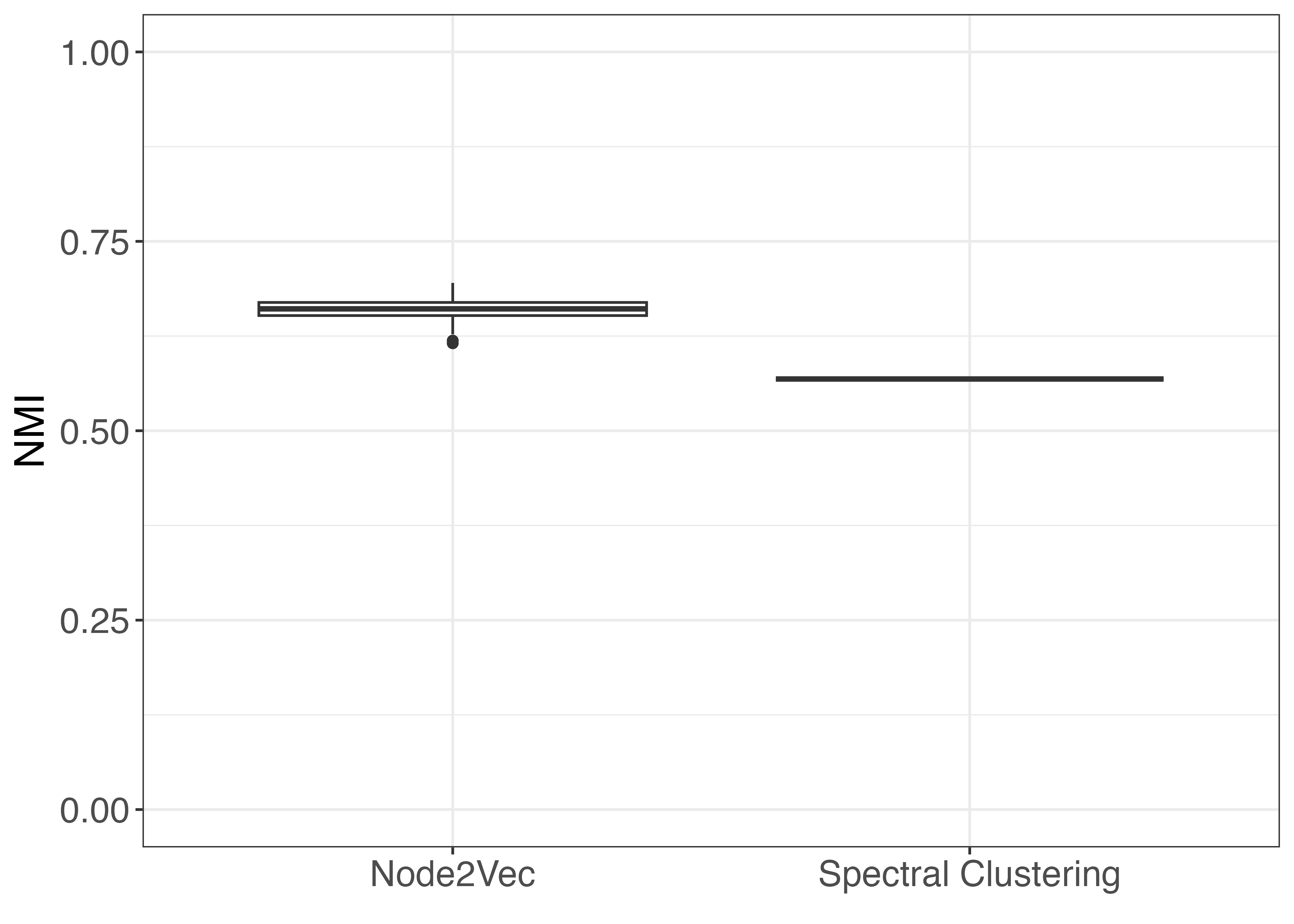}
    \caption{Community recovery for the Email data, using both Node2Vec and Spectral Clustering.
    Node2Vec can better recover the true communities.}
    \label{fig:snap1}
\end{figure}

\begin{figure}[ht]
    \centering
    \includegraphics[width=0.75\textwidth]{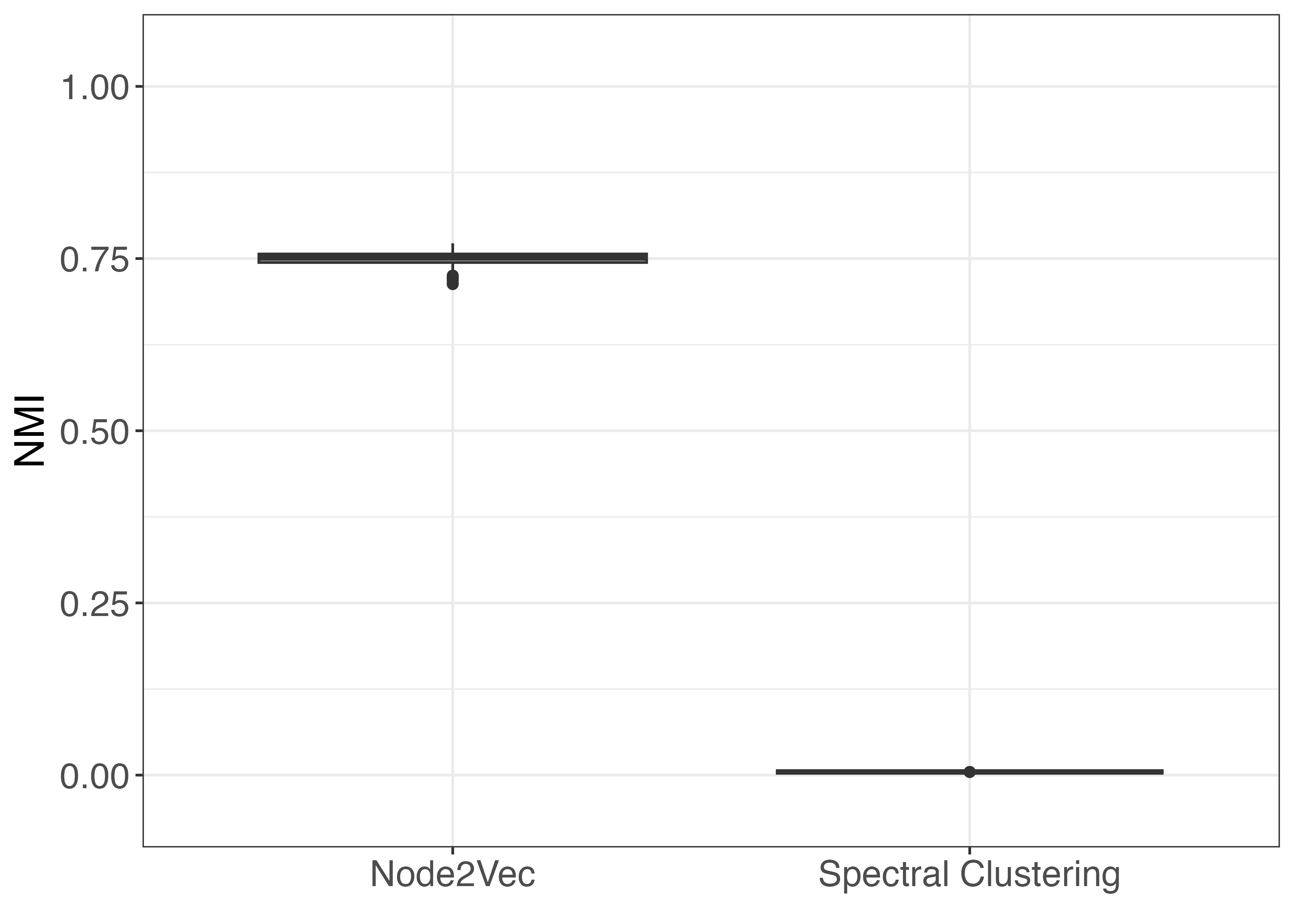}
    \caption{Community recovery for the Political Blog data, using both Node2Vec and Spectral Clustering.
    Node2Vec can better recover the true communities.}
    \label{fig:blogs1}
\end{figure}

In Figure~\ref{fig:snap1} we compare the performance of Node2Vec and spectral clustering for the Email network
and in Figure~\ref{fig:blogs1} we use the Political Blogs network.
We measure community recovery in terms of the normalized mutual information (NMI)
between the estimated and true communities. Other metrics such as
the adjusted rand index (ARI) showing similar results.
In each case the communities estimated by Node2Vec are substantially closer to the true communities
than those estimated by spectral clustering. 
As highlighted by \citet{karrer_stochastic_2011} for the political blog data, 
models which do not account for degree heterogeneity can struggle 
to recover the underlying community structure. As shown in Figure~\ref{fig:blogs1},
spectral clustering is unable to recover the
communities due to this heterogeneity, while clustering 
using the Node2Vec embedding shows strong performance at community recovery.

We also further expand on the role of the embedding parameters
in the performance of Node2Vec on these real networks.
In Figure~\ref{fig:snap2} we examine community recovery for the Email data
as we vary the embedding dimension $d$ and the unigram sampling 
parameter $\alpha$. As we vary each of these parameters we see 
good community recovery in all settings. 
For this dataset all choices of embedding dimension and unigram parameter
give good NMI scores.

\begin{figure}
    \centering
    \begin{subfigure}{0.45\textwidth}
       \centering 
       \includegraphics[width=\textwidth]{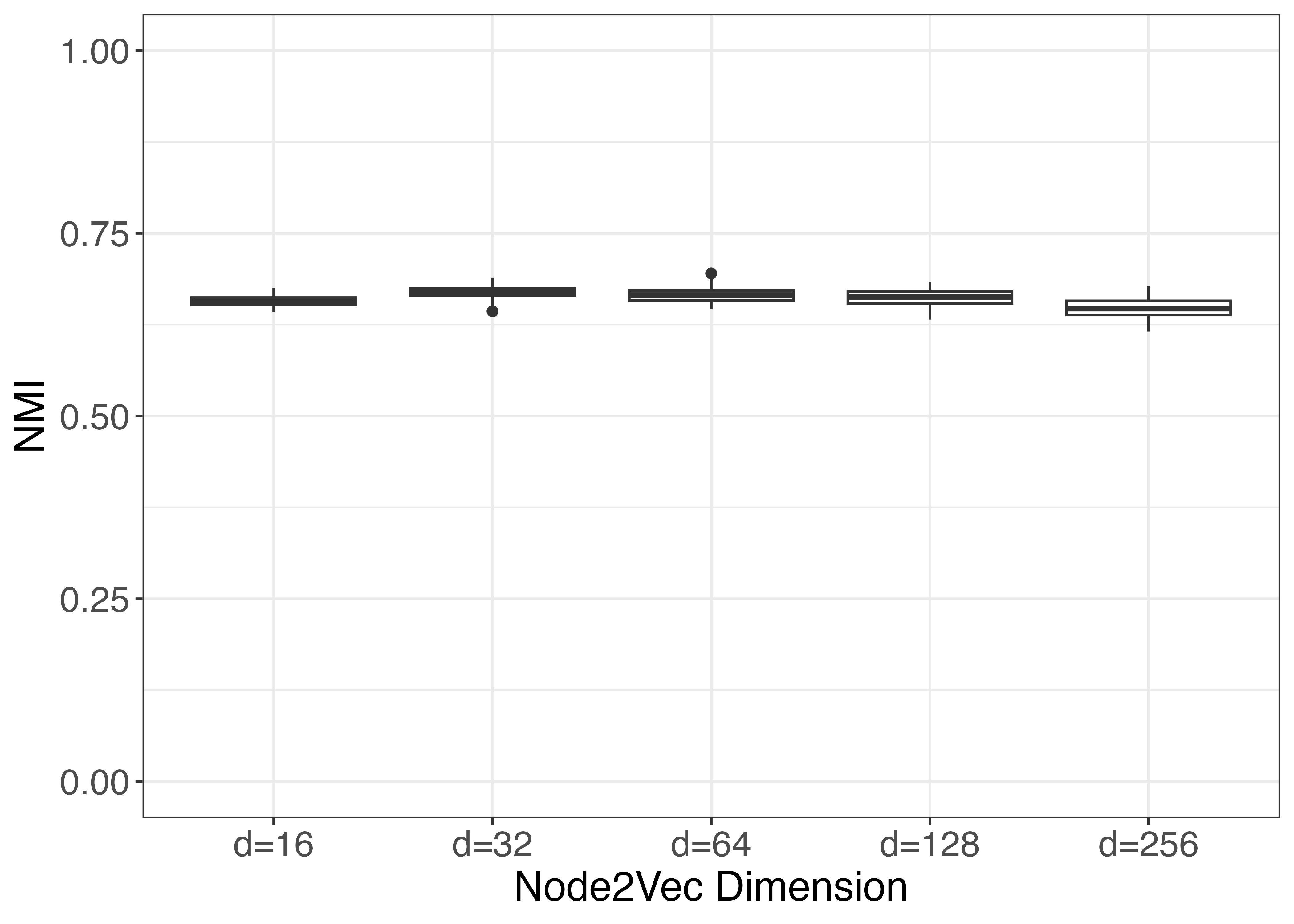}
       \caption{Varying the embedding dimension used.}
    \end{subfigure}
    \hfill
    \begin{subfigure}{0.45\textwidth} 
    \includegraphics[width=\textwidth]{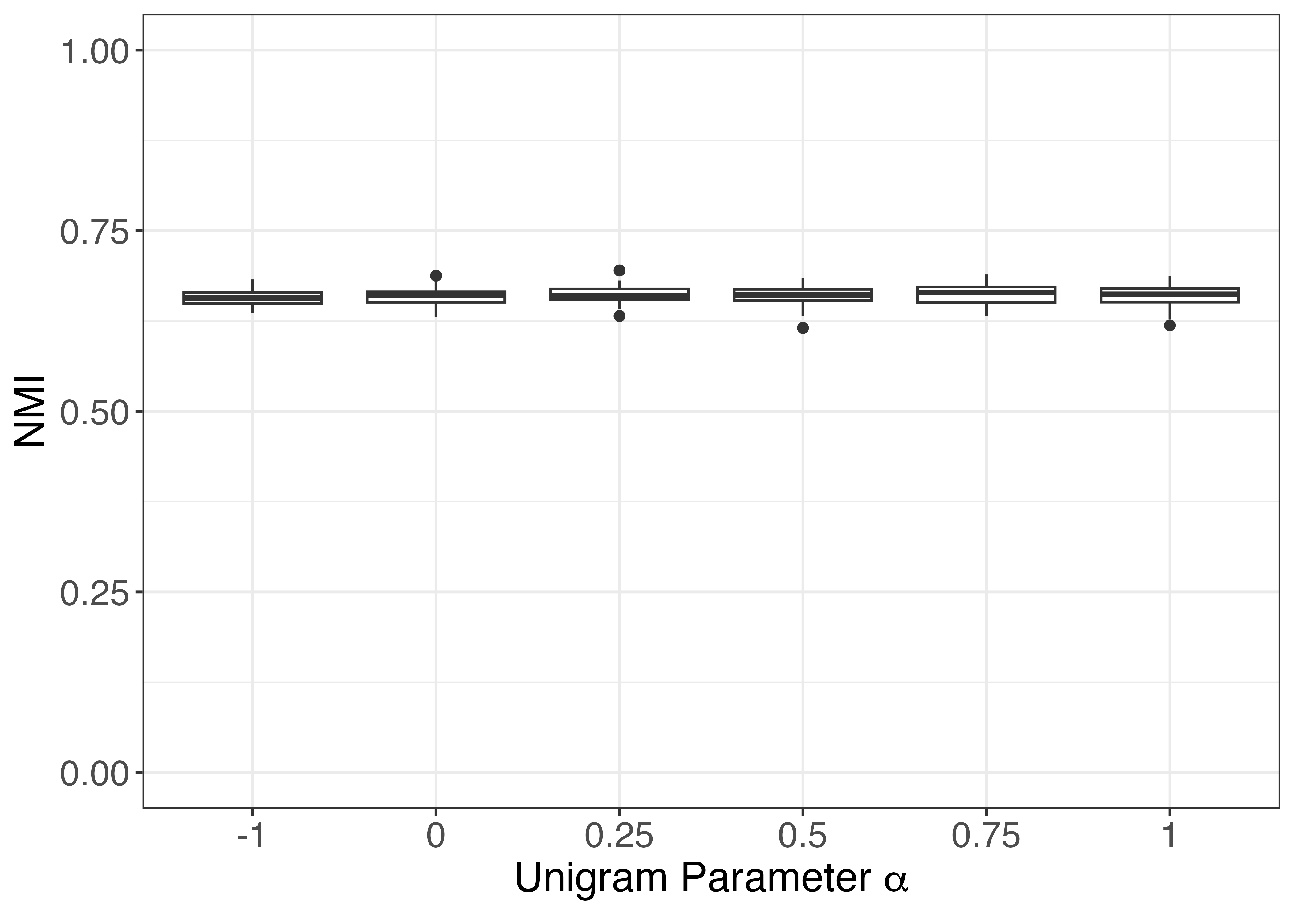}
    \caption{Varying the Unigram Parameter $\alpha$.}
    \end{subfigure}
    \caption{The effect of different Node2Vec parameters on community recovery, measure in terms of 
    Normalized Mutual Information (NMI), for the Email Data.}
    \label{fig:snap2}
\end{figure}


\section{Additional Notation}
We give a brief recap of some of the notation introduced in the main paper, along with some more notation which is used purely within the Supplemntary Material. Throughout, we will suppose that the graph $\mcG = (\mcV, \mcE)$ is
drawn according to the following generative model: each vertex 
$u \in \mcV$ have latent variables $\lambda_u = (c(u), \theta_u)$ where
$c(u) \in [\kappa]$ is a community assignment, and $\theta_u$ is a degree-heterogenity 
correction factor. We then suppose that the edges $a_{uv} \in \{ 0, 1\}$ 
in the graph $\mcG_n$ on $n$ vertices arise independently with probability
\begin{equation}
    \mathbb{P}(a_{uv} = 1 \,|\, \lambda_u, \lambda_v ) 
    = \rho_n \theta_u \theta_v P_{c(u), c(v)}
\end{equation}
for $u < v$, with $a_{uv} = a_{vu}$ by symmetry for $u > v$\footnote{To prevent notation overloading when $A$ is used to indicate constants, we use $a_{uv}$
to describe the 
presence or absence of an edge between nodes $u$ and $v$ in the supplement, rather than $A_{uv}$ which was used in the main text.}. 
The factor
$\rho_n$ accounts for sparsity in the network. The above model corresponds 
to a degree corrected stochastic block model \citep{karrer_stochastic_2011};
we highlight that the
case where $\theta_u$ is constant across all $u \in \mcV$ corresponds
to the original stochastic block model \citep{holland_stochastic_1983}.
For convenience, we will write
\begin{equation}
    W(\lambda_u, \lambda_v) = \theta_u \theta_v P_{c(u), c(v)} \qquad \text{ so } 
    \qquad \mathbb{P}(A_{uv} = 1 \,|\, \lambda_u, \lambda_v ) 
    = \rho_n W(\lambda_u, \lambda_v).
\end{equation}
We then introduce the notation
\begin{equation}
    W(\lambda_i, \cdot) := \mathbb{E}[ W(\lambda_i, \lambda_j) \,|\, \lambda_i], \qquad 
    \mcE_W(\alpha) := \mathbb{E}[ W(\lambda_i, \cdot)^{\alpha}] \text{ for } \alpha > 0.
\end{equation}
Note that under the assumptions that the community assignments are drawn i.i.d
from a $\mathrm{Categorical}(\pi)$ random variable, and the degree correction factors
are drawn i.i.d from a distribution $\vartheta$ independently of the community
assignments, we have 
\begin{align}
    W(\lambda_i, \cdot) & = 
    \theta_i \cdot \mathbb{E}[\theta] \cdot 
        \mathbb{E}_{j \sim \mathrm{Cat}(\pi)}[ P_{c(i), j} \,|\, c(i)] 
     = \theta_i \cdot \mathbb{E}[\theta] \cdot \sum_{j = 1}^{\kappa} \pi_j P_{c(i), j}, \\
    \mcE_W(\alpha) & = \mathbb{E}[\theta^{\alpha}] \cdot \mathbb{E}[\theta]^{\alpha} \cdot
    \sum_{i = 1}^{\kappa} \pi_i \Big( \sum_{j=1}^{\kappa} \pi_j P_{i, j}   \Big)^{\alpha}
\end{align}
For convenience, we will write $\widetilde{P}_{c(i)} = \sum_{j=1}^{\kappa} 
\pi_j P_{c(i), j}$.

Recall that node2vec attempts to minimize the objective 
\begin{align*}
    \emprisk(U, V) := \sum_{i \neq j} \Big\{ - &\psamp \log( \sigma(  \embedip ) ) 
    \\ & - \nsamp \log(1 - \sigma( \embedip )) \big\} 
\end{align*}
where $U, V \in \mathbb{R}^{n \times d}$, with $u_i, v_j \in \mathbb{R}^d$ denoting
the $i$-th and $j$-th rows of $U$ and $V$ respectively, and
$\sigma(x) := (1 + e^{-x})^{-1}$ denoting the sigmoid function. Here $\mcP$
and $\mcN$ correspond to the positive and negative sampling schemes induced by the random walk and unigram mechanisms respectively. 

\section{Proof of Theorems 1 and 2}
\subsection{Proof overview}

To give an overview of the proof approach, we work by forming successive approximations
to the function $\mcL_n(U, V)$ where we have uniform convergence of the approximation
error as $n \to \infty$ over either level
sets of the function considered, or the overall domain of 
optimization of the embedding matrices $U$ and $V$. We break
these approximations up into multiple steps:
\begin{enumerate}
    \item Theorems~\ref{app:thm:n2v_positive},~\ref{app:thm:n2v_negative},~\ref{app:thm:deepwalk_sampling}
    and Proposition~\ref{thm:gram_converge:samp_weight_approx_lossfn} - We 
    begin by working with an approximation $\widehat{\mcL}_n(U, V)$ of
    $\mcL_n(U, V)$, where the sampling weights $\psamp$ and
    $\nsamp$ are replaced by functions of the latent variables
    $(\lambda_i, \lambda_j)$ of the vertices $i$ and $j$, along with
    $a_{ij}$ in the case of $\psamphat$. 
    \item The resulting approximation $\widehat{\mcL}_n(U, V)$ has a dependence on the adjacency matrix of the network. We argue that this loss function converges uniformly to its average over the adjacency matrix when the vertex latent variables remain fixed; this is the contents of Theorem~\ref{thm:gram_converge:adjacency_average}.
    \item So far, the loss function only looks between interactions of $u_i$ and $v_j$ for $i \neq j$. For theoretical purposes, it is more
    convenient to work with a loss function where the term with $i = j$ is
    included. This is handled within Lemma~\ref{app:thm:add_diag}.
    \item Now that we have an averaged version of the loss function to work with, we are able to examine the minima of this loss function, and find
    that there is a unique minima (in the sense that for any pair
    of optima matrices $U^*$ and $V^*$, the matrix $U^* (V^*)^T$ is
    unique). Moreover, in certain circumstances we can give closed
    forms for these minima. This is the contents of Section~\ref{sec:app:embed_converge:minima}.
    \item This is then all combined together in order to give Theorems~\ref{thm:gram_converge:gram_converge}~and~\ref{thm:gram_converge:embed_converge_supp}, 
    which correspond to Theorems~1 and
    2 of the main text. 
\end{enumerate}
We recap that we consider three scenarios - referred to as Scenario (i), (ii) and (iii) throughout - when proving the following result:
\begin{enumerate}[label=(\roman*)]
    \item We use DeepWalk ($p = q =1$ in node2vec), and the 
    graph is drawn according to a SBM with $\rho_n \gg \log(n)/n$;
    \item We use node2vec, and the graph is drawn according to a SBM with $\rho_n = n^{-\alpha}$
    for some $\alpha < \alpha'$, where $\alpha'$ depends on node2vec's hyperparameters;
    \item We use DeepWalk and a unigram parameter of $\alpha = 1$, and the graph is drawn according to a DCSBM with $\rho_n \gg \log(n)/n$ where the degree heterogeneity
    parameters $\theta_u \in [C^{-1}, C]$ for some $C > \infty$.
\end{enumerate}
Generally speaking, the approach is the exact same for all three scenarios. As we have 
a closed formula in the case where we examine DeepWalk, we will consistently provide the details
for the DeepWalk case first, and then discuss afterwards how the results and proofs change (if at all)
when considering node2vec in generality. Throughout, we also contextualize the 
proof by examining what it says
for a SBM$(n, \kappa, \tilde{p}, \tilde{q}, \rho_n)$ model. This corresponds to a balanced network with 
$\pi = (\kappa^{-1},\ldots, \kappa^{-1})$.

\subsection{Replacing the sampling weights}

Before giving an approximation to $\mcL_n(U, V)$, we need to first
come up with approximate forms of $\psamp$ and $\nsamp$. The next three results
give examples of this.
In this section we prove three main results. The first two give us guarantees for
the sampling probabilities of vertex pairs $(u, v)$ for node2vec for any choice of
the hyperparameters $(p, q)$. In particular they will allow us to argue that when
the underlying graph arises from a SBM, the sampling probabilities asymptotically
depend only on the underlying communities. The last specializes
this to the case of DeepWalk (where $p = q =1$), which has enough structure to allow
us to get some additional information, such as closed formula for these sampling probabilities,
which can be used in the case where the graph arises through a DCSBM.

\begin{theorem}
    \label{app:thm:n2v_positive}
    There exists $\alpha$ sufficiently small, depending on the
    walk length $k$, such that if $\rho_n = n^{-\alpha}$ then there exists
    a symmetric measurable (with respect to the sigma field generated by
    $W$) function $f_{\mcP}(\lambda, \lambda')$ which is bounded
    below away from zero, and bounded above by $C \rho_n^{-1}$ for some constant $C < \infty$, 
    such that
    \begin{equation}
        \max_{i \neq j} \Bigg| \frac{ n^2 \psamp}{ a_{ij} \psamphat} - 1  \Bigg| = o_p(1).
    \end{equation}
\end{theorem}

\begin{theorem}
    \label{app:thm:n2v_negative}
    There exists $\alpha$ sufficiently small, depending on the
    walk length $k$, such that if $\rho_n = n^{-\alpha}$ then there exists
    a symmetric measurable (with respect to the sigma field generated by
    $W$) function $f_{\mcP}(\lambda, \lambda')$ which is bounded
    below away from zero, and bounded above by some constant $C < \infty$, such that
    \begin{equation}
        \max_{i \neq j} \Bigg| \frac{ n^2 \psamp}{ \nsamphat} - 1  \Bigg| = o_p(1).
    \end{equation}
\end{theorem}

The proof of these two results are given in Appendix~\ref{app:sec:n2v_positive} and 
\ref{app:sec:n2v_negative} respectively. We note that while in principle we could give
a closed formula for $f_{\mcP}$ and $f_{\mcN}$ in this scenario, they are sufficiently
intractable to inspection that doing so would not provide any benefit.

In the case of DeepWalk where $p = q = 1$, the calculations involved are tractable
enough such that we can improve the sparsity constraints, give closed forms for 
the measurable functions discussed above, and also provide rates of convergence.
\begin{theorem}
    \label{app:thm:deepwalk_sampling}
    Denote 
    \begin{align}
        \psamphat &:= \frac{ 2k}{\rho_n \mcE_W(1)}, \\
        \nsamphat & := \frac{ l (k+1)}{\mcE_W(1) \mcE_W(\alpha)} 
            \big( W(\lambda_i, \cdot) W(\lambda_j, \cdot)^{\alpha} + 
            W(\lambda_i, \cdot)^{\alpha} W(\lambda_j, \cdot) \big).
    \end{align}
    Then we have that
    \begin{align}
        \max_{i \neq j} \Bigg| \frac{ n^2 \psamp}{ a_{ij} \psamphat} - 1  \Bigg| = O_p\Big( 
            \Big( \frac{ \log n}{n \rho_n}  \Big)^{1/2} 
            \Big), \\ 
            \max_{i \neq j} \Bigg| \frac{ n^2 \nsamp}{ \nsamphat} - 1  \Bigg| = O_p\Big( 
                \Big( \frac{ \log n}{n \rho_n}  \Big)^{1/2} 
                \Big).
    \end{align}
\end{theorem}

\begin{proof}
    This is a consequence of \citep[Proposition~26]{davison_asymptotics_2023}.
    We highlight the referenced result
    supposes that for the negative sampling scheme, vertices for which $a_{ij} = 0$ are
    rejected, whereas this does not happen here. Other than for the factor 
    of $(1 - a_{ij})$ in the quoted result, the proof is otherwise unchanged, which
    gives the statement above for $\nsamp$. 
\end{proof}

With this, we then get the following result:

\begin{proposition}
    \label{thm:gram_converge:samp_weight_approx_lossfn}
    Denote
    \begin{equation}
        \empriskhat(U, V) :=  \frac{1}{n^2} \sum_{i \neq j} \Big\{ - \psamphat a_{ij} \log( \sigma(  \embedip ) ) 
        - \nsamphat \log(1 - \sigma( \embedip )) \big\} 
    \end{equation}
    and define the set
    \begin{equation}
        \Psi_{\tilde{A}} := \Big\{ U, V \in \mathbb{R}^{n \times d} \,|\, \
            \mcL_n(U, V) \leq \tilde{A} \mcL_n( 0_{n \times d}, 0_{n \times d} ) \Big\} \subseteq
            \mathbb{R}^{n \times d} \times \mathbb{R}^{n \times d}
    \end{equation}
    for any constant $\tilde{A} > 1$, where $0_{n \times d}$ denotes
    the zero matrix in $\mathbb{R}^{n \times d}$. Then for any set
    $X \subseteq \mathbb{R}^{n \times d} \times \mathbb{R}^{n \times d}$
    containing the pair of zero matrices $O_{n \times d}$, we have under Scenario i) and iii) that
    \begin{align}
        \sup_{(U, V) \in \Psi_A \cap X} \big| \emprisk(U, V) - \empriskhat(U, V) \big| 
        = O_p\Big( 
            \tilde{A} \cdot \Big( \frac{ \log n}{n \rho_n}  \Big)^{1/2}  
            \Big), \\ 
        \mathbb{P}\Big(   \argmin_{(U, V) \in X} \emprisk(U, V) \cup \argmin_{(U, V) \in X} \empriskhat(U, V) 
        \subseteq \Psi_{\tilde{A}} \cap X \Big) = 1 - o(1).
    \end{align}
    In Scenario (ii), the $O_p(\cdot)$ bound is replaced by an $o_p(1)$ bound.
\end{proposition}

\begin{proof}
    The proof is essentially equivalent to Lemma~32 of \citep{davison_asymptotics_2023} 
    up to changes in notation, and so we do not repeat the details.
\end{proof}

Note that
    in practice we can choose $A$ to be any constant greater than $1$ 
    but fixed with $n$ - e.g
    $A = 10$, and have the result hold. We will do so going forward.

\subsection{Averaging over the adjacency matrix of the graph}

Following the proof outline, the next step is to argue that
$\mcL_n(U, V)$ is close to its expectation when we average over
the adjacency matrix of the graph $\mcG_n$. We begin with showing what occurs in the DeepWalk
case (Scenarios (i) and (iii)), and at the end of the section we discuss how the proof changes for the more general node2vec case.
Note that we have
\begin{equation}
    \mathbb{E}[ \empriskhat(U, V) \,|\, \lambda ] = \frac{1}{n^2} \sum_{i \neq j} \Big\{
        - \psamphat \rho_n W(\lambda_u, \lambda_v) \log(\sigma(\embedip)) 
        - \nsamphat \log(1 - \sigma(\embedip)) \Big\}
\end{equation}
and so
\begin{align}
    E_n(U, V) & := \frac{\mcE_W(1)}{2k} \Big( 
        \empriskhat(U, V) - \mathbb{E}[\empriskhat(U, V) \,|\, \lambda ] \Big) \\
    & = \frac{1}{n^2} \sum_{i \neq j} 
    \Big( \rho_n^{-1} a_{ij} - W(\lambda_i, \lambda_j) \Big) 
    \cdot ( - \log \sigma (\embedip)).
\end{align}
Note that $\mathbb{E}[ E_n(U, V) \,|\, \lambda ] = 0$, and so it therefore
suffices to control $E_n(U, V) - \mathbb{E}[E_n(U, V) \,|\, \lambda]$
uniformly over embedding matrices $U, V \in \mathbb{R}^{n \times d}$.
This is the contents of the next theorem.

\begin{theorem}
    \label{thm:gram_converge:adjacency_average}
    Begin by defining the set 
   \begin{align}
        B_{2, \infty}(\tilde{A}_{2, \infty}) := \big\{
            U \in \mathbb{R}^{n \times d} \,:\, \| U \|_{2, \infty} \leq \tilde{A}_{2, \infty}
        \big\}.
   \end{align}
   Then we have the bound
   \begin{equation}
        \sup_{U, V \in B_{2, \infty}(\tilde{A}_{2, \infty}) } \big| E_n(U, V) \big| = O_p
        \Big( 
            \tilde{A}_{2, \infty}^2 \Big( \frac{ d }{ n \rho_n} \Big)^{1/2} 
        \Big).
   \end{equation}
   In particular, we also have that
   \begin{equation}
        \sup_{U, V \in B_{2, \infty}(\tilde{A}_{2, \infty}) } \big|  
        \empriskhat(U, V) - \mathbb{E}[\empriskhat(U, V) \,|\, \lambda ] \big|
        = O_p\Big(   \frac{ \tilde{A}_{2, \infty}^2 k}{\mcE_W(1)} 
        \Big( \frac{ d }{ n \rho_n} \Big)^{1/2}   \Big).
   \end{equation}
\end{theorem}

\begin{proof}
    Begin by noting that for any set $C \subseteq \mathbb{R}^{n \times d} \times
    \mathbb{R}^{n \times d}$ for which $0_{n \times d} \times 0_{n \times d} \in C$,
    we have that 
    \begin{align}
        \sup_{(U, V) \in C} | E_n(U, V) | & \leq 
            \sup_{(U, V) \in C} \big| E_n(U, V) - E_n( 0_{n \times d}, 0_{n \times d} ) \big| 
            + | E_n(0_{n \times d}, 0_{n \times d}) | \\
            & \leq \sup_{(U, V), (\tilde{U}, \tilde{V}) \in C} 
            \big| E_n(U, V) - E_n( \tilde{U}, \tilde{V} ) \big| 
            + | E_n(0_{n \times d}, 0_{n \times d}) |.
    \end{align}
    We therefore need to control these two terms. We begin with the second; note
    that as
    \begin{equation}
        \label{app:eq:main_chaining_1}
        E_n(0_{n \times d}, 0_{n \times d}) =  
        \frac{1}{n^2} \sum_{i \neq j} 
    \Big( \rho_n^{-1} a_{ij} - W(\lambda_i, \lambda_j) \Big) 
        \cdot \frac{1}{n^2} 
    \end{equation}
    it follows by Lemma~\ref{app:thm:bern_2} that this term is $O_p( (n^2 \rho_n)^{-1/2})$.
    For the first term, we make use of a chaining bound. Note that if we write
    $T_{ij} = - \log \sigma( \embedip)$ and $S_{ij} = - \log \sigma( \langle 
    \tilde{u}_i, \tilde{v}_j \rangle )$ for $i, j \in [n]$, then we have that 
    \begin{equation}
        E_n(U, V) - E_n( \tilde{U}, \tilde{V} ) = 
        \frac{1}{n^2} \sum_{i \neq j} 
        \Big( \rho_n^{-1} a_{ij} - W(\lambda_i, \lambda_j) \Big) 
        \cdot (T_{ij} - S_{ij}).
    \end{equation}
    Because the function $x \mapsto -\log\sigma(x)$ is 1-Lipschitz, it follows that
    \begin{equation}
        \| T - S \|_F^2 \leq \| UV^T - \tilde{U} \tilde{V}^T \|_F^2, 
        \qquad \| T - S \|_{\infty} \leq \| UV^T - \tilde{U} \tilde{V}^T \|_{\infty}
    \end{equation}
    and consequently we have that
    \begin{align}
        \mathbb{P}\big( |E_n(U, V) & - E_n( \tilde{U}, \tilde{V} )| \geq u \big) 
        \\
        & \leq 
        2 \exp\Big( - \min\Big\{    
            \frac{u^2}{ 128 \rho_n^{-1} n^{-4} \| UV^T - \tilde{U} \tilde{V}^T \|_F^2}, \frac{u}{16 \rho_n^{-1} n^{-2} \|  UV^T - \tilde{U} \tilde{V}^T \|_{\infty}}
        \Big\}  \Big)
    \end{align}
    as a result of Lemma~\ref{app:thm:bern_2}. Now, as
    $U, V \in B_F(A_F) \cap B_{2, \infty}(\tilde{A}_{2, \infty})$, by
    Lemma~\ref{app:thm:set_bounds} if we define the metrics
    \begin{align}
        d_F((U_1, V_1), (U_2, V_2)) & 
            := \| U_1 - U_2 \|_F + \| V_1 - V_2 \|_F, \\
        d_{2, \infty}((U_1, V_1), (U_2, V_2)) & 
            := \| U_1 - U_2 \|_{2, \infty} + \| V_1 - V_2 \|_{2, \infty},
    \end{align}
    then we have that
    \begin{align}
        \mathbb{P}\big( |&E_n(U, V) - E_n( \tilde{U}, \tilde{V} )| \geq u \big) 
        \\
        & \leq 
        2 \exp\Big( - \min\Big\{    
            \frac{u^2}{ 128 \rho_n^{-1} n^{-4} A_F^2 
            d_F((U, V), (\tilde{U}, \tilde{V}))^2   }, 
            \frac{u}{16 \rho_n^{-1} n^{-2} \tilde{A}_{2, \infty}    
            d_{2, \infty}((U, V), (\tilde{U}, \tilde{V}))
            }
        \Big\}  \Big).
    \end{align}
    As a result of Corollary~\ref{app:thm:chain_corr}, it therefore follows that
    \begin{equation}
        \label{app:eq:main_chaining_2}
        \sup_{(U, V), (\tilde{U}, \tilde{V}) \in T \times T} \big| E_n(U, V) - 
        E_n(\tilde{U}, \tilde{V}) \big| = O_p\Big(  
            \tilde{A}_{2, \infty}^2 \Big( \frac{ d }{ n \rho_n} \Big)^{1/2}  +
            \tilde{A}_{2, \infty}^2 \frac{ d }{ n \rho_n}  \Big) 
    \end{equation}
    The desired conclusion follows by combining the bounds 
    \eqref{app:eq:main_chaining_1} and \eqref{app:eq:main_chaining_2}.
\end{proof}

For the more abstract node2vec case under Scenario (ii), we highlight that we can take
\begin{equation}
    E_n(U, V) = \frac{1}{n^2} \sum_{i \neq j} \rho_n f_{\mcP}(\lambda_i, \lambda_j)
    \Big( \rho_n^{-1} a_{ij} - W(\lambda_i, \lambda_j) \Big) 
    \cdot ( - \log \sigma (\embedip)).
\end{equation}
Now, as $f_{\mcP}(\lambda_u, \lambda_v)$ is a function of the community assignments only
within the SBM case, we can replace this by a matrix of constants $f_{\mcP, c, c'}$ for
$c, c' \in [\kappa]$, and therefore the error term can be decomposed into a sum
\begin{equation}
    \sum_{c_1, c_2} (\rho_n f_{\mcP, c_1, c_2}) \sum_{\substack{i \neq j \\ i: c(u) = c_1 \\ j: c(u) = c_2}}  \Big( \rho_n^{-1} a_{ij} - W(\lambda_i, \lambda_j) \Big) 
    \cdot ( - \log \sigma (\embedip)),
\end{equation}
where we recall that $\max_{c_1, c_2} (\rho_n f_{\mcP, c_1, c_2}) < \infty$
as guaranteed by Theorem~\ref{app:thm:n2v_positive}.
Each of these terms (of which there are finitely many) can be controlled using the exact same
argument as in Theorem~\ref{thm:gram_converge:adjacency_average}, and so the 
conclusion of the Theorem also holds with the same overall rate of convergence in
Scenario (ii).

\subsection{Adding in a diagonal term} 

Currently the sum in $\mathbb{E}[\widehat{\mcL}_n(U, V) \,|\, \lambda]$ is
defined only terms $i, j$ with $i \neq j$ - it is more convenient to work
with the version where the diagonal term is added in:
\begin{align}
    \mcR_n(U, V) := \frac{1}{n^2} \sum_{i, j \in [n]}
    \Big\{  
        - \psamphat &\rho_n W(\lambda_u, \lambda_v) \log(\sigma(\embedip)) \\
        & - \nsamphat \log(1 - \sigma(\embedip))
    \Big\}.
\end{align}
We show that this does not significantly change the size of the
loss function. 

\begin{lemma}
    \label{app:thm:add_diag}
    With the same notation as in Theorem~\ref{thm:gram_converge:adjacency_average},
    we have that
    \begin{align*}
        \sup_{U, V \in B_{2, \infty}(\tilde{A}_{2, \infty}) } \big| 
            \mcR_n(U, V) & - \mathbb{E}[ \widehat{\mcL}_n(U, V) \,|\, \lambda]
            \big| \\
            & = O_p\Big( \frac{1}{n} \tilde{A}_{2, \infty}^2 \Big( 
               \| \rho_n f_{\mcP}(\lambda, \lambda') W(\lambda, \lambda') \|_{\infty} +
               \| f_{\mcN}(\lambda, \lambda') \|_{\infty} 
     \Big)  \Big).
    \end{align*}
    In particular, in the case of DeepWalk we have that
    \begin{align*}
        \sup_{U, V \in B_{2, \infty}(\tilde{A}_{2, \infty}) } \big| 
            \mcR_n(U, V) & - \mathbb{E}[ \widehat{\mcL}_n(U, V) \,|\, \lambda]
            \big| = O_p\Big( \frac{1}{n} \tilde{A}_{2, \infty}^2 \Big( 
                \frac{2k \|W\|_{\infty}}{ \mcE_W(1) } + \frac{ 2l(k+1) \|W \|_{\infty}^2}{
                    \mcE_W(1) \mcE_W(\alpha)
                } \Big)  \Big).
    \end{align*}
\end{lemma}

\begin{proof}
    Begin by noting that
    \begin{equation}
    \begin{aligned}
        0 & \leq 
        \mcR_n(U, V) - \mathbb{E}[ \widehat{\mcL}_n(U, V) \,|\, \lambda] \\
        & = \frac{1}{n^2} \sum_{i=1}^n \big\{  
            - \psamphat \rho_n W(\lambda_u, \lambda_v) \log(\sigma( 
                \langle u_i, v_i \rangle
            )) 
        - \nsamphat \log(1 - \sigma( \langle u_i, v_i \rangle )) \big\}. 
    \end{aligned}
    \end{equation}
    Note that we can bound 
    \begin{equation}
        - \log( \sigma( \langle u_i, v_j))  \leq | \langle u_i, v_i \rangle
        \leq \| u_i \|_2 \| v_i \|_2 
    \end{equation}
    and similarly $- \log( 1 - \sigma(\langle u_i, v_i \rangle)) \leq
    | \langle u_i, v_i \rangle | \leq \| u_i \|_2 \|v_i \|_2$. 
    Moreover, we have the bounds
    \begin{equation}
        \psamphat \rho_n W(\lambda_i, \lambda_j) \leq 
        \| \rho_n f_{\mcP}(\lambda, \lambda') W(\lambda, \lambda') \|_{\infty} < \infty,
        \nsamphat \leq  \| f_{\mcN}(\lambda, \lambda') \|_{\infty} < \infty
    \end{equation}
    under our assumptions. 
    As a result, because $U, V \in \mcB_{2, \infty}(\tilde{A}_{2, \infty})$, we end
    up with the final bound
    \begin{equation}
        \big|  \mcR_n(U, V) - \mathbb{E}[ \widehat{\mcL}_n(U, V) \,|\, \lambda] \big|
        \leq \frac{1}{n} \tilde{A}_{2, \infty}^2 \Big( 
               \| \rho_n f_{\mcP}(\lambda, \lambda') W(\lambda, \lambda') \|_{\infty} +
               \| f_{\mcN}(\lambda, \lambda') \|_{\infty} 
     \Big)
    \end{equation}
    which gives the stated result as the RHS is free of $U$ and $V$.
\end{proof}

\subsection{Chaining up the loss function approximations}

By chaining up the prior results, we end up with the following result:

\begin{proposition}
    \label{thm:gram_converge:loss_converge}
    There exists a non-empty set $\Psi_n$ for each $n$ such that,
    for any set $X \subseteq \mathbb{R}^{n \times d} \times
    \mathbb{R}^{n \times d}$ containing $0_{n \times d} \times 0_{n \times d}$,
    we have for DeepWalk that 
    \begin{equation}
        \begin{aligned}
        \sup_{(U, V) \in \Psi_n \cap
        B_{2, \infty}(\tilde{A}_{2, \infty})
        } \big| \emprisk(U, V) - \mcR_n(U, V) \big| 
        = O_p\Big( 
             \Big( \frac{ \log n}{n \rho_n}  \Big)^{1/2}  
             +  \tilde{A}_{2, \infty}^2 \Big( \frac{ d }{n \rho_n}  \Big)^{1/2}  
            \Big)
        \end{aligned}
    \end{equation}
    and
    \begin{equation}
        \begin{aligned}
        \mathbb{P}\Big(  
            \argmin_{(U, V) \in B_{2, \infty}(\tilde{A}_{2, \infty}) \cap X } \emprisk(U, V) 
            \cup 
            \argmin_{(U, V) \in B_{2, \infty}(\tilde{A}_{2, \infty}) \cap X } \mcR_n(U, V) 
        & \subseteq \Psi_A \cap B_{2, \infty}(\tilde{A}_{2, \infty}) \cap X \Big) \\
        & = 1 - o(1).
    \end{aligned}
    \end{equation}
    For node2vec, the same result holds when we replace the $(\log n / n \rho_n)^{1/2}$
    term with an $o_p(1)$ term and add the constraint that $d \ll n \rho_n$.
    The same result also holds when we constrain $U = V$, but otherwise
    keep everything else unchanged.
\end{proposition}

\subsection{Minimizers of $\mcR_n(U, V)$}
\label{sec:app:embed_converge:minima}

Recall that we have earlier defined 
\begin{equation}
\begin{aligned}
    \mcR_n(U, V) := \frac{1}{n^2} \sum_{i, j \in [n]}
    \Big\{  
        - \psamphat & \rho_n W(\lambda_u, \lambda_v) \log(\sigma(\embedip)) 
        \\ & - \nsamphat \log(1 - \sigma(\embedip))
    \Big\}.
    \end{aligned}
\end{equation}
We now want to reason about the minima of these functions. To do so, note that
the optimization domain is non-convex - firstly due to the rank constraints
on the matrix $UV^T$, and secondly due to the fact that the loss function
is invariant to any mapping $(U, V) \to (U M, V M^{-1})$ for any invertible
$d \times d$ matrix $M$. To handle the second part, we consider the global minima
of this function when parameterized only in term of the matrix $UV^T$. We will
then see that the minima matrix is already low rank. 

We first begin by giving
some basic facts about the function $\mcR_n(U, V)$ when parameterized as
a function of $UV^T$. 

\begin{lemma}
    \label{thm:gram_converge:minimizers_unconstrained}
    Define the modified function
    \begin{equation}
        \mcR_n(M) := \frac{1}{n^2} \sum_{i, j \in [n]}
        \Big\{  
            - \psamphat \rho_n W(\lambda_u, \lambda_v) \log(\sigma(M_{ij})) 
            - \nsamphat \log(1 - \sigma(M_{ij}))
        \Big\}.
    \end{equation}
    over all matrices $M \in \mathbb{R}^{n \times n}$. Then we have
    the following:
    \begin{enumerate}[label=\alph*)]   
        \item The function $\mcR_n(M)$ is strictly convex in $M$.
        \item The global minimizer of $\mcR_n(M)$ is given by
        \begin{equation}
            M^*_{ij} = \log\Big(  \frac{ 
                \psamphat \rho_n W(\lambda_i, \lambda_j)  }{   \nsamphat   }     \Big)
        \end{equation}
        and satisfies $\nabla_M \mcR_n(M) = 0$. 
        \item When restricted to a cone of semi-positive definite matrices
        $M \in \mcM_n^{\curlyeqsucc 0}$, there exists a unique minimizer
        to $\mcR_n(M)$ over this set, which we call $M^{\curlyeqsucc 0}$.
        Moreover, $M^{\curlyeqsucc 0}$ has the property that
        $\langle \nabla_M \mcR_n(M^{\curlyeqsucc 0}) , M^{\curlyeqsucc 0} - M \rangle
        \leq 0$ for all $M \in \mcM_n^{\curlyeqsucc 0}$.
    \end{enumerate}
\end{lemma}

\begin{proof}
    For part a), this follows by the fact that the functions
    $-\log(\sigma(x))$ and $-\log(1 - \sigma(x))$ are positive
    and strictly convex functions of $x \in \mathbb{R}$, the fact
    that $\psamphat \rho_n W(\lambda_i, \lambda_j)$ and
    $\nsamphat$ are positive quantities which are 
    bounded above (see e.g
    Lemma~\ref{app:thm:add_diag}), and the fact that the sum of
    strictly convex functions is strictly convex. For part b),
    this follows by noting that each of the $M_{ij}^*$ are pointwise
    minima of the functions
    \begin{equation}
        r_{ij}(x) =  - \psamphat \rho_n W(\lambda_u, \lambda_v) \log(\sigma(x))) 
        - \nsamphat \log(1 - \sigma(x))
    \end{equation}
    defined over $x \in \mathbb{R}$. Indeed, note that
    \begin{equation}
        \frac{d r_{ij}}{dx} =  (-1 + \sigma(x)) \psamphat \rho_n W(\lambda_u, \lambda_v) 
        + \sigma(x) \nsamphat,
    \end{equation}
    so setting this equal to zero, rearranging and making use of the
    equality $\sigma^{-1}(a/(a+b)) = \log(a/b)$ gives the stated result.
    Part c) is a consequence of strong convexity, the optimization domain
    being convex and self dual, and the KKT conditions. 
\end{proof}

To understand the form of the
the global minimizer of $\mcR_n(M)$ in the DeepWalk case, by substituting in the values 
for $\psamphat$ and $\nsamphat$ we end up with
\begin{align}
    M_{ij}^* = & \log\Big( \frac{ 2 P_{c(i), c(j)} \mcE_W(\alpha)}{    
        (1 + k^{-1}) \mathbb{E}[\theta] \mathbb{E}[\theta]^{\alpha} \big( 
        \theta_j^{\alpha - 1} \widetilde{P}_{c(i)} 
        \widetilde{P}_{c(j)}^{\alpha}
            + \theta_i^{\alpha - 1} \widetilde{P}_{c(i)}^{\alpha} 
            \widetilde{P}_{c(j)}
            \big) }   \Big) \\
    = & \log\Big( \frac{ 2 \mcE_W(\alpha) }{ 
        (1 + k^{-1}) \mathbb{E}[\theta] \mathbb{E}[\theta]^{\alpha} } \cdot 
        \frac{  P_{c(i), c(j)} }{ 
        \widetilde{P}_{c(i)} \widetilde{P}_{c(j)}   
        \cdot \big(   
            \theta_i^{\alpha - 1} \widetilde{P}_{c(i)}^{\alpha - 1}
        + \theta_j^{\alpha - 1} \widetilde{P}_{c(j)}^{\alpha - 1}
        \big)}   \Big) 
\end{align}

In particular, from the above formula we get the following
lemma as a consequence:

\begin{lemma}
    \label{thm:gram_converge:nice_minimizers_unconstrained}
    Suppose that Scenarios (i) or (iii) holds, so that either a) $\theta_i$ is constant for all $i$, or
    b) $\alpha = 1$. Then if we write $\Pi_C \in \mathbb{R}^{n \times \kappa}$
    for the matrix where $(\Pi_C)_{il} = 1[ c(i) = l]$, and define the matrix
    \begin{equation}
        (\widetilde{M}^*_\alpha)_{lm} = \log\Big( 
            \frac{ 2 \mcE_W(\alpha) }{ 
        (1 + k^{-1}) \mathbb{E}[\theta] \mathbb{E}[\theta]^{\alpha} } \cdot 
        \frac{  P_{lm}  }{ \widetilde{P}_m \widetilde{P}_l^{\alpha} 
        + \widetilde{P}_m^{\alpha} \widetilde{P}_l   }   \Big) \; \text{ for } l, m \in [\kappa], 
    \end{equation}
    then we have that $M^* = \Pi_C \widetilde{M}^*_\alpha \Pi_C^T$. In particular, as soon
    as the matrix $\Pi_C$ is of full rank (which occurs with asymptotic
    probability 1), then the rank of $M^*$ equals the rank of $\widetilde{M}^*_\alpha$.
    Moreover, as soon as $d$ is greater than or equal to the rank of
    $\widetilde{M}^*_\alpha$, $(U, V)$ is a minimizer of $\mcR_n(U, V)$ if
    and only if $UV^T = M^*$. 

    Under Scenario (ii), the same result applies noting that $f_{\mcP}$ and
    $f_{\mcN}$ are functions only of the underling communities, and so if we abuse
    notation and write e.g $f_{\mcP}(l, m)$ to indicate the value of 
    $f_{\mcP}(\lambda_i, \lambda_j)$ when $c(i) = l$ and $c(j) = m$, one can take
    \begin{equation}
        (\widetilde{M}^*)_{lm} = \log\Big(  \frac{ 
                f_{\mcP}(l, m) \rho_n P_{l, m}  }{ f_{\mcN}(l, m) }  \Big)
    \end{equation}
    and have the above result hold. 
\end{lemma}

We discuss in Appendix~\ref{app:sec:ugly_stuff} 
what happens when we apply DeepWalk in the DCSBM regime when $\alpha \neq 1$. 
To give an example of what $M^*$ looks like, we write it
down in the case of a $\mathrm{SBM}(n, \kappa, \tilde{p}, \tilde{q}, \rho_n)$ model, which is frequently
used to illustrate the behavior of various community detection algorithms. Such 
a model assumes that the community assignments
$\pi_l = 1/\kappa$ for all $l \in [\kappa]$, and that
\begin{equation}
    P_{kl} = \begin{cases} \tilde{p} & \text{ if } k = l, \\ \tilde{q} & \text{ if } k \neq l.
    \end{cases}
\end{equation}
In this case, we have that
\begin{align}
    \widetilde{P}_l = \frac{\tilde{p} + \kappa(\tilde{q}-1)}{\kappa} \text{ for } l \in [\kappa],
    \qquad 
    \mcE_W(\alpha) = \mathbb{E}[\theta]^{\alpha} \mathbb{E}[\theta^{\alpha}] \cdot \Big(  
        \frac{\tilde{p} + (\kappa- 1)\tilde{q} }{\kappa} \Big)^{\alpha}.
\end{align}
Substituting these values into the matrix $\widetilde{M}_{\alpha}^*$ gives
\begin{equation}
    (\widetilde{M}^*_\alpha)_{lm} = 
        \log \Big( \frac{ \mathbb{E}[\theta^{\alpha}] }{ \mathbb{E}[\theta] (1 + k^{-1})   }  
        \cdot \frac{ \kappa \tilde{p}}{\tilde{p} + (\kappa - 1) \tilde{q} }
        \Big) \delta_{lm} + 
        \log \Big( \frac{ \mathbb{E}[\theta^{\alpha}] }{ \mathbb{E}[\theta] (1 + k^{-1})   }
        \cdot \frac{  \kappa \tilde{q} }{ \tilde{p} + (\kappa - 1) \tilde{q} }   \Big) (1 - \delta_{lm}). 
\end{equation}
We highlight this is a matrix of the form
$\alpha \delta_{lm} + \beta (1 - \delta_{lm})$, and so it is straightforward to
describe the spectral behavior of the matrix 
(see Lemma~\ref{thm:other_results:nice_mat_lemma}). 

\subsubsection{Minimizers in the constrained regime $U = V$}
\label{app:sec:gram_converge:constraints}

In the case where we have constrained $U = V$, it is not possible in
general to write down the closed form of the minimizer of $\mcR_n(M)$
over $\mcM_n^{\curlyeqsucc 0}$. However, it is still possible to draw 
enough 
conclusions about the form of the minimizer in order to
give guarantees for community detection. We begin with the 
proposition below. We state the next two results for DeepWalk only,
but note that the first generalizes to the node2vec case immediately. 

\begin{proposition}
    \label{thm:gram_converge:minimizers_constrained}
    Suppose that $\theta_i$ is constant across all $i$. Supposing that
    $\widetilde{M} \in \mathbb{R}^{\kappa \times \kappa}$ is of the form
    $\widetilde{M} = \widetilde{U} \widetilde{U}^T$ for matrices
    $\widetilde{U} \in \mathbb{R}^{\kappa \times d}$, define the
    function
    \begin{equation}
        \widetilde{\mcR}_n(\widetilde{M}) = 
        \sum_{l, m \in [\kappa]} \hat{p}_n(l) \hat{p}_n(m)
        \Big\{   
            - 2k P_{lm} 
            \log\sigma( \langle u_l, u_m \rangle) 
            - \{  \widetilde{P}_{l}  \widetilde{P}_{m}^{\alpha}
            + \widetilde{P}_{m}  \widetilde{P}_{l}^{\alpha}
               \} \log(1 - \sigma( \langle u_l, u_m \rangle))
        \Big\}
    \end{equation}
    where we define $\hat{p}_n(l) := n^{-1} | \{ i \,:\, c(i) = l \} |$
    for $l \in [\kappa]$. 
    Then $\widetilde{\mcR}_n(\widetilde{M})$ is strongly convex, and
    moreover has a unique minimizer as soon as $d \geq \kappa$. 
    
    Moreover, any minimizer of
    $\mcR_n(M)$ over matrices $M$ of the form $M = UU^T$
    where $U \in \mathbb{R}^{n \times d}$ must take the form 
    $M = \Pi_C M^* \Pi_C^T$ where $(\Pi_C)_{il} = 1[c(i) = l]$
    where $M^*$ is a minimizer of $\widetilde{\mcR}_n(\widetilde{M})$.
    In particular, once $d \geq \kappa$, there is a unique minimizer to
    $\mcR_n(M)$.
\end{proposition}

\begin{proof}
    The properties of $\widetilde{\mcR}_n(\widetilde{M})$ are immediate by
    similar arguments to Lemma~\ref{thm:gram_converge:minimizers_unconstrained}
    and standard facts in convex analysis.
    We begin by noting that if we substitute in the values
    \begin{align}
        \rho_n W(\lambda_i, \lambda_j) \psamphat & = 
            \frac{2k P_{c(i), c(j)}}{\mcE_W(1)}, \\
        \nsamphat & = \frac{l(k+1)}{ \mcE_W(1) \mcE_W(\alpha)} 
        \big(  
                \widetilde{P}_{c(i)}  \widetilde{P}_{c(j)}^{\alpha}
                + \widetilde{P}_{c(j)}  \widetilde{P}_{c(i)}^{\alpha}
                \big),
    \end{align}
    for $\psamphat$
    and $\nsamphat$, then we can write that (recalling that
    $M_{ij} = \langle u_i, u_j \rangle$)
    \begin{align}
        \mcR_n(M) & := \frac{1}{n^2} \sum_{i, j \in [n]} \Big\{  
            - 2k P_{c(i), c(j)} \log \sigma(\langle u_i, u_j \rangle)
            \\
            & \qquad \qquad \qquad
            - \frac{ l(k+1)}{ \mcE_W(1) \mcE_W(\alpha)} \big(  
                \widetilde{P}_{c(i)}  \widetilde{P}_{c(j)}^{\alpha}
                + \widetilde{P}_{c(j)}  \widetilde{P}_{c(i)}^{\alpha}
                \big) \log(1 - \sigma(\langle u_i, u_j \rangle))
        \Big\} \\
        & := \sum_{l, m \in [\kappa]} \hat{p}_n(l) \hat{p}_n(m)
        \Big\{   
            - 2k P_{lm} \frac{1}{|\mcC_l| |\mcC_m|} \sum_{i \in \mcC_l, j \in \mcC_m}
            \log\sigma( \langle u_i, u_j \rangle) \\
            & \qquad \qquad \qquad 
            - \{  \widetilde{P}_{c(i)}  \widetilde{P}_{c(j)}^{\alpha}
            + \widetilde{P}_{c(j)}  \widetilde{P}_{c(i)}^{\alpha}
               \} \frac{1}{|\mcC_l| |\mcC_m|} \sum_{i \in \mcC_l, j \in \mcC_m}
            \log(1 - \sigma( \langle u_i, u_j \rangle))
        \Big\}
    \end{align}
    where for $l \in [\kappa]$ we define
    $\hat{p}_n(l) := n^{-1} | \{ i \,:\, c(i) = l \}|$, along with the
    sets $\mcC_l = \{ i\,:\, c(i) = l\}$. Now, note that as
    the functions $-\log(\sigma(x))$ and $-\log(1-\sigma(x))$ are strictly
    convex, by Jensen's inequality we have that e.g
    \begin{equation}
        \frac{1}{|\mcC_l| |\mcC_m|} \sum_{i \in \mcC_l, j \in \mcC_m}
            - \log\sigma( \langle u_i, u_j \rangle)
            \geq - \log\sigma\Big( \Big\langle 
            \frac{1}{|\mcC_l|} \sum_{i \in \mcC_l} u_i, 
            \frac{1}{|\mcC_m|} \sum_{j \in \mcC_m} u_j \Big\rangle \Big)
    \end{equation}
    (where we also used bilinearity of the inner product)
    where equality holds above if and only if the $u_i$ are constant are across all
    indices $i$. In particular, any minimizer of $\mcR_n(M)$ must
    have the $u_i$ constant across $i \in \mcC_l$ for each $l \in [\kappa]$, which
    defines the function $\tilde{\mcR}_n(\widetilde{M})$. This gives the claimed statement.
\end{proof}

In certain cases, we are able to give a closed form to the
minimizer. We illustrate this for the case of the SBM$(n, \kappa, \tilde{p}, \tilde{q}, \rho_n)$ model.

\begin{proposition}
    Let $\widetilde{M}^*$ be the unique minimizer of
    $\tilde{\mcR}_n(\widetilde{M})$ as introduced
    in Proposition~\ref{thm:gram_converge:minimizers_constrained}.
    In the case of a SBM$(n, \kappa, \tilde{p}, \tilde{q}, \rho_n)$ model, we have that 
    $ \kappa^{-2} \| \widetilde{M}^* - M^* \|_1 
    = O_p( (\kappa \log \kappa / n)^{1/4} )$,
    where $M^*$ is of the form 
    \begin{equation}
    (M^*)_{ij} = \alpha^* \delta_{ij} - \frac{\alpha^*}{  \kappa - 1 } (1 - \delta_{ij})
    \end{equation}
    for some $\alpha^* = \alpha^*(\tilde{p}, \tilde{q}) \geq 0$. Moreover, $\alpha^* > 0$ iff $\tilde{p} > \tilde{q}$.
\end{proposition}

\begin{proof}
    We begin by arguing that the objective 
    function $\widetilde{\mcR}_n(\widetilde{M})$
    converges uniformly to the objective
    \begin{equation}
        \bar{\mcR}_n(\widetilde{M})
        := \frac{1}{\kappa^2} \sum_{l, m \in [\kappa]} 
        \Big\{   
            - 2k P_{lm} 
            \log\sigma( \langle u_l, u_m \rangle) 
            - \{  \widetilde{P}_{m}  \widetilde{P}_{l}^{\alpha}
            + \widetilde{P}_{l}  \widetilde{P}_{m}^{\alpha}
               \} \log(1 - \sigma( \langle u_l, u_m \rangle))
        \Big\}
    \end{equation}
    over a set containing the minimizers of both functions. Note that
    this function is also strictly convex, and has a unique minimizer as soon
    as $d \geq \kappa$. To do so, we highlight that as we have that
    \begin{equation}
        \max_{k \neq l} \Big| \frac{ \hat{p}_n(l) \hat{p}_n(k) - \kappa^{-2} }{ 
            \kappa^{-2} } \Big|
        = O_p\Big( \Big(  \frac{ \kappa \log \kappa}{n}  \Big)^{1/2}   \Big)
    \end{equation}
    by standard concentration results for Binomial random
    variables (e.g Proposition~47 of \cite{davison_asymptotics_2023}), 
    it follows that 
    \begin{equation}
        \label{eq:blahhhh}
        \big| \bar{\mcR}_n(\widetilde{M}) - \widetilde{\mcR}_n(\widetilde{M}) \big|
        \leq \bar{\mcR}_n(\widetilde{M}) 
        \cdot O_p\Big( \Big(  \frac{ \kappa \log \kappa}{n}  \Big)^{1/2}   \Big).
    \end{equation}
    Consequently, $\widetilde{\mcR}_n(\widetilde{M})$ converges to
    $\bar{\mcR}_n(\widetilde{M})$ uniformly over any level set of 
    $\bar{\mcR}_n(\widetilde{M})$, which necessarily contains the minima of
    $\bar{\mcR}_n(\widetilde{M})$. If one does so over the set (for example)
    \begin{equation}
        A = \{  \widetilde{M} \,:\, \bar{\mcR}_n(\widetilde{M}) \leq 10 \bar{\mcR}_n(0) \}
    \end{equation}
    (for example), then as $\bar{\mcR}_n(0)$ is constant across $n$,
    we have uniform convergence of \eqref{eq:blahhhh} over the set $A$
    at a rate of $O_p\big(  ( \log \kappa / n p)^{1/2}   \big)$.
    This argument can be reversed, which therefore
    ensures uniform convergence (over the same set) which contains the minimizers
    (with the minimizer of $\tilde{\mcR}_n(M)$ being contained within this set
    with asymptotic probability $1$)
    at a rate of $O_p\big(  ( \kappa \log \kappa / n)^{1/2}   \big)$.

    With this, we note that an application of Lemma~\ref{thm:eg:other:curvature}
    gives that for any matrices $\widetilde{M}_1$ and $\widetilde{M}_2$ we have that
    \begin{align}
        \bar{\mcR}_n( \widetilde{M}_1 ) & \geq  \bar{\mcR}_n( \widetilde{M}_2 )
        + \langle \Delta \bar{\mcR}_n(\widetilde{M}_2) , \widetilde{M}_1
        - \widetilde{M}_2 \rangle \\
        &
        \qquad +
        \frac{C}{\kappa^2}
        \sum_{i, j \in [\kappa]} 
        \min\{  | (\widetilde{M}_2)_{ij} - (\widetilde{M}_1)_{ij} |^2,
        2 | (\widetilde{M}_2)_{ij} - (\widetilde{M}_1)_{ij} | \}.
    \end{align}
    where to save on notation, we define
    \begin{equation}
        C := \frac{1}{4} e^{-\| \widetilde{M}_2 \|_{\infty}} 
        \min_{l,m} \{ 2k P_{lm}, \widetilde{P}_m \widetilde{P}_l^{\alpha} \}.
    \end{equation}
    In particular, if $\widetilde{M}_2 = \bar{M}^*$ is an optimum of
    $\bar{\mcR}_n(\widetilde{M})$, then by the
    KKT conditions (similarly as in 
    Lemma~\ref{thm:gram_converge:minimizers_unconstrained})
    we have that
    \begin{equation}
        \bar{\mcR}_n( \widetilde{M}_1 )- \bar{\mcR}_n( \bar{M}^* )
         \geq 
        \frac{C}{\kappa^2}
        \sum_{i, j \in [\kappa]} 
        \min\{  | ( \bar{M}^* )_{ij} - (\widetilde{M}_1)_{ij} |^2,
        2 | ( \bar{M}^* )_{ij} - (\widetilde{M}_1)_{ij} | \}.
    \end{equation}
    In particular, if we then let $\widetilde{M}^*$ be any minimizer of 
    $\widetilde{\mcR}_n(\widetilde{M})$, then we have that
    \begin{align}
        \frac{C}{\kappa^2}
        \sum_{i, j \in [\kappa]} & 
        \min\{  | ( \bar{M}^* )_{ij} - (\widetilde{M}_1)_{ij} |^2,
        2 | ( \bar{M}^* )_{ij} - (\widetilde{M}_1)_{ij} | \} \\
        & \leq \bar{\mcR}_n( \widetilde{M}_1 )- \bar{\mcR}_n( \bar{M}^* )
        \leq \bar{\mcR}_n( \widetilde{M}_1 )- \tilde{\mcR}_n( \bar{M}^* )
        + \tilde{\mcR}_n(\widetilde{M}^*) - \bar{\mcR}_n( \bar{M}^* ) \\
        & \leq 2 \sup_{M \in A} 
            \big|  \tilde{\mcR}_n(M) - \bar{\mcR}_n(M) \big|
    \end{align}
    on an event of asymptotic probability $1$. Consequently,
    it follows by Lemma~\ref{thm:eg:other:curvature_lower_bound} that
    \begin{equation}
        \frac{1}{\kappa^2} \| \bar{M}^* - \widetilde{M}^* \|_1 
        = O_p\big(  ( \kappa \log \kappa / n)^{1/4}   \big).
    \end{equation}

    We now need to find the minimizing positive semi-definite matrix
    which optimizes $\bar{\mcR}_n(\widetilde{M})$. To do so, 
    we will argue that
    one can find $\alpha$ for which
    \begin{equation*}
        \widehat{M}_{ij} = \alpha \delta_{ij} - \frac{\alpha}{\kappa - 1} (1 - \delta_{ij}), 
        \quad 
        \nabla \bar{\mcR}_n(\widehat{M}) = C 1_{\kappa} 1_{\kappa}^T, \quad 1_{\kappa} = (1, \cdots, 1)^T
    \end{equation*}
    for some positive constant $C$, as then the KKT conditions 
    for the constrained optimization problem will hold. 
    Indeed, for any positive definite matrix $M$, as by definition of $\widehat{M}$
    we have that $\langle \nabla \bar{\mcR}_n(\widehat{M}), \widehat{M} \rangle = 0$ 
    as all of the
    eigenvectors of $\widehat{M}$ are orthogonal to the unit vector $1_{\kappa}$
    (Lemma~\ref{thm:other_results:nice_mat_lemma}). It 
    consequently follows that as $\nabla \bar{\mcR}_n(\widehat{M})$ 
    is itself positive definite,
    we get that $\langle -\nabla \bar{\mcR}_n(\widehat{M}), \widehat{M} - M \rangle 
    = \langle \nabla \bar{\mcR}_n(\widehat{M}), M \rangle \geq 0$. We now need to verify the existence of a constant
    $\alpha$ for which this condition holds. We 
    note that as $\widehat{M}_{ij}$ is constant across $i = j$, and 
    also constant across $i \neq j$, to verify
    the condition that $\nabla \bar{\mcR}_n(\widehat{M})$ is 
    proportional to $1_{\kappa} 1_{\kappa}^T$, 
    it suffices to check whether the on and off diagonal terms
    of $\nabla \bar{\mcR}_n(\widehat{M})$ are equal to each other. 
    This gives the equation
    \begin{align*}
        \sigma(\alpha) \cdot \Big( k \tilde{p} & + l(k+1) \frac{\tilde{p}+(\kappa - 1)\tilde{q}}{\kappa}\Big)  \\
        & = k(\tilde{p}-\tilde{q}) + \sigma(-\alpha/(\kappa- 1)) \Big( k\tilde{q}  + l(k+1) \frac{\tilde{p}+(\kappa - 1)\tilde{q}}{\kappa} \Big) 
    \end{align*}
    By applying Lemma~\ref{thm:other_results:poly_solve}, this has
    a singular positive solution in $\alpha$ if and only if $k(\tilde{p} - \tilde{q}) \geq k(\tilde{p}-\tilde{q})/2$,
    which holds iff $\tilde{p} \geq \tilde{q}$. In the case where $\tilde{p} < \tilde{q}$, it follows that the solution
    has $\alpha = 0$.
\end{proof}

\subsection{Strong convexity properties of the minima matrix}

\begin{proposition}
    \label{thm:gram_converge:strong_convexity}
    Define the modified function
    \begin{equation}
        \mcR_n(M) := \frac{1}{n^2} \sum_{i, j \in [n]}
        \Big\{  
            - \psamphat \rho_n W(\lambda_u, \lambda_v) \log(\sigma(M_{ij})) 
            - \nsamphat \log(1 - \sigma(M_{ij}))
        \Big\}.
    \end{equation}
    over all matrices $M \in \mathbb{R}^{n \times n}$. Then we have for
    any matrices $M_1, M_2 \in \mathbb{R}^{n \times n}$ with
    $\| M_1 \|_{\infty}, \|M_2\|_{\infty} \leq \tilde{A}_{\infty}$ that
    \begin{align}
        \mcR_n(M_1) \geq \mcR_n(M_2) 
        + \langle \nabla \mcR_n(M_2), M_1 - M_2 \rangle 
        + \frac{\widetilde{C} e^{-\tilde{A}_{\infty}}}{2} \cdot \frac{1}{n^2} \|M_1 - M_2 \|_F^2
    \end{align}
    where $\widetilde{C} = \min_{l, m}\{ 2k P_{l, m}, \tilde{P}_l^{\alpha} 
    \tilde{P}_m \}$ for Scenarios (i) and (iii), and 
    $\widetilde{C} = \min\{ \| \rho_n f_{\mcP}(\lambda, \lambda') \|_{-\infty}, 
    \| f_{\mcN}(\lambda, \lambda') \|_{-\infty} \} > 0$ for Scenario (ii). Moreover,
    \begin{enumerate}[label=\roman*)]
        \item If $\mcR_n(M)$ is constrained over a set 
        $\mcX = \{ M = UV^T \,:\, U, V \in \mathbb{R}^{n \times d}, \| M \|_{\infty}
        \leq \tilde{A}_{\infty} \}$, and there exists $M^*$ in $\mcX$ such that
        $\nabla \mcR_n(M^*) = 0$, then we have that
        \begin{equation}
            \frac{1}{n^2} \| M^* - M \|_F^2 \leq 2\widetilde{C}^{-1} e^{\tilde{A}_{\infty}}
            \cdot \big( \mcR_n(M) - \mcR_n(M^*)  \big) \text{ for all } M \in \mcX.
        \end{equation}
        \item If $\mcR_n(M)$ is constrained over a set
        $\mcX^{\geq 0} = \{  M = UU^T \,:\, U \in \mathbb{R}^{n \times d}, 
        \| M \|_{\infty} \leq \tilde{A}_{\infty}$ \}, and there exists $M^*$ in $\mcX^{\geq 0}$
        such that $\langle \nabla \mcR_n(M^*), M - M^* \rangle \geq 0$ for all
        $M \in \mcX^{\geq 0}$, then we get the same inequality as in part i) above.
    \end{enumerate}
\end{proposition}

\begin{proof}
    The first inequality 
    follows by an application of Lemma~\ref{thm:eg:other:curvature},
    with the second and third parts following by applying the conditions
    stated and rearranging.
\end{proof}

\subsection{Convergence of the gram matrices of the embeddings}

By combining together Proposition~\ref{thm:gram_converge:strong_convexity}
and Proposition~\ref{thm:gram_converge:loss_converge} we end up with the
following result: 

\begin{theorem}
    \label{thm:gram_converge:gram_converge}
    Suppose that the conditions of
    Lemma~\ref{thm:gram_converge:nice_minimizers_unconstrained}
    hold. (In particular, recall that $d \geq \kappa$.)
    Then there exist constants $\tilde{A}_{\infty}$ and $\tilde{A}_{2, \infty}$
    (depending on the parameters of the model and sampling scheme)
    and a matrix $M^* \in \mathbb{R}^{\kappa \times \kappa}$
    (also depending on the parameters of the model and the sampling
    scheme) such that for any minimizer $(U^*, V^*)$ of $\mcL(U, V)$ over the set
    \begin{equation}
        X = \{ (U, V) \,:\, \| U \|_{\infty},  \| V \|_{\infty} \leq \tilde{A}_{\infty},
        \| U \|_{2, \infty}, \| V \|_{2, \infty} \leq \tilde{A}_{2, \infty} \},
    \end{equation}
    we have that
    \begin{equation}
        \frac{1}{n^2} \sum_{i, j \in [n]}
        \big( \langle u_i^*, v_j^* \rangle - M^*_{c(i), c(j)} \big)^2
        = C \cdot \begin{cases} O_p( (\tfrac{ \max\{\log n, d\} }{n \rho_n} )^{1/2} ) 
        & \text{under Scenarios (i) and (iii);} \\ 
            o_p(1) & \text{under Scenario (ii);}
        \end{cases}
    \end{equation}
    for some 
    constant $C$ depending on the model, the node2vec hyperparameters, $\tilde{A}_{\infty}$ and $\tilde{A}_{2, \infty}$. In the case
    where we constrain $U = V$, the same result holds provided
    the conditions of Proposition~\ref{thm:gram_converge:minimizers_constrained}
    hold. 
\end{theorem}

\begin{proof}
    We note that by Lemma~\ref{thm:gram_converge:nice_minimizers_unconstrained},
    there exists a minimizer $\widetilde{M}^*$ for
    $\mcR_n(M)$ of the form $\widetilde{M}^* = \Pi M^* \Pi^T$
    for a matrix $M^* \in \mathbb{R}^{\kappa \times \kappa}$. We can then
    take $\tilde{A}_{\infty}$ and $\tilde{A}_{2, \infty}$ as $2 \| M^* \|_{\infty}$
    and $2 \| M^* \|_{2, \infty}$. We highlight that we can do this
    even when $d > \kappa$, as we can embed $M^*$ into the block diagonal matrix
    $\mathrm{diag}( M^*, O_{d - \kappa, d - \kappa} )$, which preserves
    both the norms above. Lemma~\ref{thm:gram_converge:minimizers_unconstrained} and 
    Proposition~\ref{thm:gram_converge:strong_convexity} then guarantee that
    \begin{equation}
        \frac{1}{n^2} \| U^* (V^*)^T - \widetilde{M}^* \|_F^2 
        \leq \tilde{C} \cdot 
        \big(  \mcR_n(UV^T) - \mcR_n(\widetilde{M}^*) \big)
    \end{equation}
    for some constant $\tilde{C}$ depending only on the quantities mentioned
    in the theorem statement.
    As $\mcX$ is a subset of $\mcB_{2, \infty}(\tilde{A}_{2, \infty})$, and $(U^*, V^*)$
    is a minimizer of $\mcL(U, V)$, we end up getting that 
    \begin{align}
        \big(  \mcR_n(UV^T) & - \mcR_n(\widetilde{M}^*) \big) \\
        & 
        \leq \mcR_n(UV^T) - \mcL_n(U^*, V^*) + \mcL_n(M^*) - \mcR_n(\widetilde{M}^*) \\
        & \leq 2 \sup_{(U, V) \in X} \big| \mcR_n(U, V) - \mcL_n(U, V) \big|
    \end{align}
    from which we can apply Proposition~\ref{thm:gram_converge:loss_converge}
    to then give the claimed result.
\end{proof}

We give some brief intuition as to the size of the constants
involved here, to understand any potential hidden dependencies
involved in them. Of greatest concern are the constants
$\tilde{A}_{\infty}$ and $\tilde{A}_{2, \infty}$ (as the remaining constants
are explicit throughout the proof, and depend
only on the hyperparameters of the sampling schema and the model
in a polynomial fashion). 
Note that in the case
where $k$ is large and we have a SBM$(n, \kappa, \tilde{p}, \tilde{q}, \rho_n)$ model
and we apply the DeepWalk scheme, from
the discussion after 
Lemma~\ref{thm:gram_converge:nice_minimizers_unconstrained}, 
the minimizing matrix $M^*$ takes the form
\begin{equation}
    (M^*)_{lm} \approx \log\Big( \frac{\kappa \tilde{p}}{\tilde{p} + (\kappa - 1) \tilde{q}} \Big) \delta_{lm}
    + \log\Big(  \frac{\kappa \tilde{q}}{\tilde{p} + (\kappa - 1) \tilde{q}}      \Big) (1 - \delta_{lm}).
\end{equation}
Supposing for simplicity that $\tilde{p} > \tilde{q}$, it follows that we can take
can take $\tilde{A}_{\infty}$ to be of the order $O(\log(\tilde{p}/\tilde{q}))$ when $\kappa$ is large.
In the rate from Proposition~\ref{thm:gram_converge:strong_convexity}, this gives
a rate of $O(\tilde{p}/\tilde{q})$ from the $e^{\tilde{A}_{\infty}}$ factor; note that
the dependence on
the parameters of the models here are not unreasonable. As for $\tilde{A}_{2, \infty}$,
we first highlight the fact that 
\begin{equation}
(\kappa - 1) \log\Big( \frac{\kappa \tilde{q}}{\tilde{p}+ (\kappa - 1)\tilde{q} } \Big)
\to \frac{\tilde{p}-\tilde{q}}{\tilde{q}} \text{ as } \kappa \to \infty.
\end{equation}
By Lemma~\ref{thm:other_results:nice_mat_lemma} we can
therefore take $\tilde{A}_{2, \infty}$ to be a scalar multiple of $|\log(\tilde{p}/\tilde{q})|^{1/2}$, 
avoiding any implicit
dependence on $\kappa$ or the embedding dimension $d$.

\subsection{Convergence of the embedding vectors}

We can then get results guaranteeing the convergence of the
individual embedding vectors (rather than their gram matrix)
up to rotations, as stated by the following theorem.

\begin{theorem}
    \label{thm:gram_converge:embed_converge_supp}
    Suppose that the conclusion of
    Theorem~\ref{thm:gram_converge:gram_converge} holds, and further suppose that $d$
    equals the rank of the matrix $M^*$. Then there exists a matrix
    $\tilde{U}^* \in \mathbb{R}^{\kappa \times d}$ such that
    \begin{equation}
        \min_{Q \in O(d)} \frac{1}{n} \sum_{i=1}^n 
        \| u_i^* - \tilde{u}_{c(i)}^* Q \|_2^2 =  C \cdot \begin{cases} O_p( (\tfrac{ \max\{\log n, d\} }{n \rho_n} )^{1/2} ) 
        & \text{under Scenarios (i) and (iii);} \\ 
            o_p(1) & \text{under Scenario (ii);}
        \end{cases}
    \end{equation}
\end{theorem}

\begin{proof}
    We handle the cases where $U \neq V$ and $U = V$ separately. For the case
    where $U \neq V$, we note that without loss of generality we can
    suppose that $UU^T = VV^T$, in which case we can apply 
    Lemma~\ref{thm:mat:procrustes_bound}
    and Theorem~\ref{thm:gram_converge:gram_converge} to give the
    stated result. To do so, we note that by 
    Lemma~\ref{thm:mat:assignment_mat_svals} we have that
    $n^{-1} \sigma_d( \Pi M^* \Pi^T) \geq c \sigma_d(M^*)$ for some
    constant $c$ with asymptotic probability $1$, as a result of the fact that
    $n_k(\Pi) \geq 1/2 n \pi_k$ with asymptotic probability $1$ uniformly
    across all communities $k \in [\kappa]$. As moreover we have that
    $n^{-1} \| UV^T - \Pi M^* \Pi^T \|_{\text{op}} \leq 
    n^{-1} \| UV^T - \Pi M^* \Pi^T \|_{F} = o_p(1)$, the 
    condition
    that $\| UV^T - \Pi M^* \Pi^T \|_{\text{op}} \leq 1/2 \sigma_d(\Pi
    M^* \Pi^T)$ holds with asymptotic probability 1,
    we have verified the conditions in Lemma~\ref{thm:mat:procrustes_bound}, giving the
    desired result. In the case where we constrain $U = V$, the same argument
    holds, except we no longer need to verify the condition that
    $\| UU^* - M^* \|_{\text{op}}$ is sufficiently small, and so we have concluded
    in this case also.
\end{proof}

In the case of a SBM$(n, \kappa, \tilde{p}, \tilde{q}, \rho_n)$ model it is actually able
to give closed form expressions for the embedding vectors which are
converged to by factorizing the minima matrix $M^*$ in the way described
by the above proof. These details are given in Lemma~\ref{thm:other_results:nice_mat_lemma}.

\section{Proof of Theorem 3, Corollary 4 and Lemma 5}

\subsection{Guarantees for community detection}

We begin with a discussion of how we can get guarantees for 
community detection via approximate
k-means clustering method, using the convergence criteria
for embeddings
we have derived already. To do so, suppose we have a matrix
$U \in \mathbb{R}^{n \times d}$ corresponding of $n$ columns of $d$-dimensional
vectors. Defining the set
\begin{equation}
    M_{n, K} := \{ \Pi \in \{0, 1\}^{n \times K} \,:\, \text{each row of $\Pi$
    has exactly $K - 1$ zero entries} \},
\end{equation}
we seek to find a factorization $U \approx \Pi X$ for matrices $\Pi \in M_{n, K}$
and $X \in \mathbb{R}^{K \times d}$. To do so, we minimize the objective 
\begin{equation}
    \mathcal{L}_k(\Pi, X) = \frac{1}{n} \| U - \Pi X \|_F^2
\end{equation}
In practice, this minimization problem is NP-hard \citep{aloise_np-hardness_2009},
but we can find
$(1 + \epsilon)$-approximate solutions in polynomial time \citep{kumar_linear_2005}.
As a result,
we consider any minimizers $\hat{\Pi}$ and $\hat{X}$ such that
\begin{equation}
    \mcL_k(\hat{\Pi}, \hat{X}) \leq (1 + \epsilon) \min_{\Pi, X} \mcL_k(\Pi, X).
\end{equation}
We want to examine the behavior of k-means clustering on the matrix $U$, when it
is close to a matrix $U^*$ which has an exact factorization $U^* = \Pi^* X^*$
for some matrices $\Pi^* \in M_{n, K}$ and $X^* \in \mathbb{R}^{K \times d}$.
We introduce the notation
\begin{equation}
    G_k(\Pi) := \{ i \in [n] : \Pi_{ik} = 1 \}, \qquad n_k(\Pi) := |G_k(\pi)|
\end{equation}
for the columns of $U$ which are assigned as closest
to the $k$-th column of $X$ as according to the matrix $\Pi$. 

We make use of the following theorem from \citet{lei_consistency_2015},
which we restate for ease of use.

\begin{proposition}[Lemma~5.3 of \citet{lei_consistency_2015}]
    \label{thm:ml:kmeans}
    Let $(\hat{\Pi}, \hat{X})$ be any $(1+\epsilon)$-approximate
    minimizer to the k-means problem given a matrix $U \in \mathbb{R}^{n \times d}$.
    Suppose that $U^* = \Pi^* X^*$
    for some matrices $\Pi^* \in M_{n, \kappa}$ and $X^* \in \mathbb{R}^{\kappa \times d}$.
    Fix any $\delta_k \leq \min_{l \neq k} \| X^*_{l\cdot} - X^*_{k \cdot} \|_2$,
    and suppose that the condition
    \begin{equation}
        \label{eq;thm:ml:kmeans:condition}
        (16 + 8\epsilon) \| U - U^* \|_F^2 / \delta_k^2 < n_k(\Pi^*) \text{ for all }
        k \in [\kappa]
    \end{equation}
    holds. Then there exist subsets $S_k \subseteq G_k(\Pi^*)$ and a
    permutation matrix $\sigma \in \mathbb{R}^{\kappa \times \kappa}$ 
    such that the following holds:
    \begin{enumerate}[label=\roman*)] 
        \item For $G = \bigcup_k (G_k(\Pi^*) \setminus S_k)$, we have that
        $(\Pi^*)_{G\cdot} = \sigma \Pi_{G\cdot}$. In words, outside of the sets
        $S_k$ we recover the assignments given by $\Pi^*$ up to a re-labelling
        of the clusters.
        \item The inequality 
        $\sum_{k=1}^\kappa |S_k| \delta_k^2 \leq (16 + 8\epsilon) \| U - U^* \|_F^2$
        holds.
    \end{enumerate}
\end{proposition}

In particular, we can then apply this to our consistency results
with the embeddings learned by node2vec. Recall that we are interested
in the following metrics measuring recovery of communities by any
given procedure:
\begin{align}
    L(c, \hat{c}) & := 
    \min_{\sigma \in \mathrm{Sym}(\kappa)} \frac{1}{n} \sum_{i=1}^n 1[ \hat{c}(i) 
    \neq \sigma(c(i))], \\
    \widetilde{L}(c, \hat{c}) &:= \max_{k \in [\kappa]} 
    \min_{\sigma \in \mathrm{Sym}(\kappa)} \frac{1}{|\mcC_k|} \sum_{i \in \mcC_k}
    1[ \hat{c}(i) \neq \sigma(k) ].
\end{align}
These measure the overall misclassification rate and worst-case class
misclassification rate respectively. 

\begin{corollary}
    \label{thm:ml:kmeans_embed_supp}
    Suppose that we have embedding vectors $\omega_i \in \mathbb{R}^d$ 
    for $i \in [n]$ such that
    \begin{equation}
        \label{eq:thm:ml:kmeans_embed:condition_supp}
        \min_{Q \in O(d)} \frac{1}{n} \sum_{i=1}^n \| \omega_i - \eta_{C(i)} Q \|_2^2
        = O_p(r_n)
    \end{equation}
    for some rate function $r_n \to 0$ as $n \to \infty$ 
    and vectors $\eta_l \in \mathbb{R}^d$ for $l \in [\kappa]$.
    Moreover suppose that $\delta := \min_{l \neq k} \| \eta_l - \eta_k \|_2 > 0$.
    Then if $\hat{c}(i)$ are the community assignments produced by applying
    a $(1+\epsilon)$-approximate k-means clustering to the matrix 
    whose columns are the $\omega_i$, we have that $L(c, \hat{c}) = O_p(
        \delta^{-2} r_n)$
    and $\widetilde{L}(c, \hat{c}) = O_p(\delta^{-2} r_n)$. If the RHS 
    of \eqref{eq:thm:ml:kmeans_embed:condition_supp} is instead $o_p(1)$,
    then we replace $O_p(r_n)$ by $o_p(1)$ in the statements for 
    $L(c, \hat{c})$ and $\widetilde{L}(c, \hat{c})$.
    
\end{corollary}

\begin{proof}
    We apply Proposition~\ref{thm:ml:kmeans} with $\Pi^*$ corresponding
    to the matrix of community assignments according to $c(\cdot)$, and
    $X^*$ the matrix whose columns are the $Q \eta_l$ for $l \in [\kappa]$
    where $Q \in O(d)$ attains the minimizer in \eqref{eq:thm:ml:kmeans_embed:condition_supp}.
    Letting $U$ be the matrix whose columns are the $\omega_i$ and taking
    $\delta_k = \delta$, the condition \eqref{eq;thm:ml:kmeans:condition}
    to verify becomes
    \begin{equation}
        \frac{16 + 8 \epsilon}{\delta^2} 
        \frac{1}{n} \sum_{i=1}^n \| \omega_i - Q \eta_{c(i)} \|_2^2 
        < \frac{ |\mcC_k|}{n} \text{ for all } k \in [\kappa].
    \end{equation}
    As $r_n \to 0$ and $|\mcC_l| / n > c > 0$ for some constant $c$ 
    uniformly across vertices $l \in [\kappa]$ with asymptotic
    probability $1$ (as a result of the community generation mechanism,
    the communities are balanced), the above event will be
    satisfied with asymptotic probability 1. The desired conclusion
    follows by making use of the inequalities
    \begin{equation}
        L(c, \hat{c}) \leq \frac{1}{n} \sum_{k \in [\kappa]} |S_k|, 
        \qquad \widetilde{L}(c, \hat{c}) \leq \max_{k \in [\kappa]}
        \frac{1}{|\mcC_k|} |S_k| \leq \Big( \max_{k \in [\kappa]} \frac{n}{|\mcC_k|}
        \Big)
        \cdot \frac{1}{n} \sum_{l \in [\kappa]} |S_l|
    \end{equation}
    which hold by the first consequence in Proposition~\ref{thm:ml:kmeans}, and
    then applying the bound 
    \begin{equation}
        \frac{1}{n} \sum_{k \in [\kappa]} |S_k| \leq \frac{16+8\epsilon}{\delta^2}
        \cdot \frac{1}{n} \sum_{i=1}^n \| \omega_i - Q \eta_{c(i)} \|_2^2. \qedhere
    \end{equation} 
\end{proof}

We note that in order to apply this theorem, we require the
further separation criterion of $\delta > 0$. As a result of 
Lemma~\ref{thm:other_results:nice_mat_lemma}, we can
guarantee this for the SBM$(n, \kappa, \tilde{p}, \tilde{q}, \rho_n)$
model when either a) DeepWalk is trained in the unconstrained setting,
or b) we are in the constrained setting with $\tilde{p}> \tilde{q}$. As we know that
the embedding vectors converge to the zero vector on average
when we are in the constrained setting with $\tilde{p} \leq \tilde{q}$, 
as a result we know that community detection is possible in 
the constrained setting iff $\tilde{p} > \tilde{q}$, which gives Corollary~5 of
the main paper.

\subsection{Guarantees for node classification and link prediction}
\label{app:sec:ml:node_classification}

We now discuss what guarantees we can make when using the embedding vectors
for classification. In this section, we suppose that we have a guarantee
\begin{equation}
    \label{eq:ml:convergence_condition}
    \frac{1}{n} \min_{Q \in O(d)} \sum_{i=1}^n \| u_i - \eta_{C(i)} Q \|_2^2 
    \leq C(\tau) r_n 
    \qquad \text{ holds with probability } \geq 1 - \tau
\end{equation}
for some constant $C(\tau)$ and rate function $r_n \to 0$ as $n \to \infty$.
This is the same as saying that the LHS is $O_p(r_n)$ - it will happen
to be more convenient to use this formulation. We also suppose
that there exists a positive constant $\delta > 0$ for which
\begin{equation}
    \delta 
    \leq \min_{k \neq l} \| \eta_k - \eta_l \|_2.
\end{equation}
We begin 
with a lemma which discusses the underlying geometry when we take a small
sample of the embedding vectors. 

\begin{lemma}
    Suppose we sample $K$ embeddings from the set $(u_i)_{i \in [n]}$,
    which we denote as $u_{i_1}, \ldots, u_{i_K}$.
    Define the sets
    \begin{equation}
        S_l = \{ i \in \mcC_l \,:\, \| u_i - \eta_{C(i)} \|_2 < \delta/4 \}.
    \end{equation}
    Then there exists $n_0(K, \delta, \tau')$ such that if $n \geq n_0$, with
    probability $1 - \tau'$ we have that $u_{i_j} \in S_{c(i_j)}$ for all $j \in [K]$.
\end{lemma}

\begin{proof}
    Without loss of generality, we will suppose that $Q = I$. 
    For each $l \in [\kappa]$, define the sets $S_l = \{ i \in \mcC_l \,:\,
    \| u_i - \eta_{l} \|_2 \leq \delta/4 \}$. Then by the condition
    \eqref{eq:ml:convergence_condition}, by Markov's inequality we know
    that with probability $1 - \tau$ we have that
    \begin{equation}
        \label{eq:helpful_bound}
        \frac{1}{n} \sum_{l \in [\kappa]} | \mcC_l \setminus S_l | 
        \leq 4 \delta^{-2} C(\tau/2) r_n.
    \end{equation}
    We now suppose that we sample $K$ embeddings 
    uniformly at random; for convenience, we suppose
    that they are done so with replacement. Then the probability that
    all of the embeddings are outside the set $\bigcup_l (\mcC_l \setminus S_l) $
    is given by $(1 - \frac{1}{n} \sum_l |\mcC_l \setminus S_l|)^{K} 
    \geq 1 - \frac{K}{n} \sum_l |\mcC_l \setminus S_l|$. In particular,
    this means with probability no less than $1 - \tau - 4K \delta^{-1} 
    C(\tau) r_n$, if we sample $K$ embeddings with indices $i_1, \ldots, i_K$
    at random from the set
    of $n$ embeddings, they lie within the sets $S_{C(i_1)}, \ldots, S_{C(i_K)}$
    respectively. The desired result then follows by noting that we take $\tau = 
    \tau'/2$, and choose $n$ such that $4 \delta^{-2} C(\tau/2) r_n < \tau'/2$.
\end{proof}

To understand how this lemma can give insights into the downstream use
of embeddings, suppose that we have access to an oracle which provides
the community assignments of a vertex when requested, but otherwise
the community assignments are unseen.

We note that in practice, only a small number of labels are needed to be provided
to embedding vectors in order to achieve good classification results 
(see e.g the experiments in \citet{hamilton_inductive_2017,velickovic_deep_2018}).
As a result, we can imagine keeping $K$ fixed in the regime where $n$ is large.
Moreover, the constant $\delta$ simply reflects the underlying geometry of the
learned embeddings, and $\tau'$ is a tolerance we can choose such that the
stated result is very likely to hold (by e.g choosing $\tau' = 10^{-2}$ or $10^{-3}$).
As a consequence, the above lemma tells us 
with high probability, we can 
\begin{enumerate}[label=\roman*)]
    \item learn a classifier which is able to distinguish
    between the sets $S_l$ given use of the sampled
    embeddings $u_{i_1}, \ldots, u_{i_K}$ and the labels
    $c(i_1), \ldots, c(i_K)$, provided the classifier is flexible
    enough to separate $\kappa$ disjoint convex sets; and 
    \item as a consequence of \eqref{eq:helpful_bound}, this classifier
    will correctly classify a large
    proportion of vertices within the correct sets $S_l$.
\end{enumerate}
The same argument applies if instead we have classes assigned to embedding vectors
which form a coarser partitioning of the underlying community assignments.
The importance of the above result is that in order to understand 
the behavior of embedding methods for classification,
it suffices to understand which geometries particular classifiers are able to
separate - for example, when the number of classes equals $2$, this reduces down
to the classic concept of linear separability, in which case a logistic classifier
would suffice.

We end with a discussion as to the task of link prediction, which asks
to predict whether two vertices are connected or not given a partial observation 
of the network. To do so, we suppose that from the observed network, 
we delete half of the edges in the network, and then train node2vec
on the resulting network. Note that the node2vec mechanism only makes explicit use of known edges within the network.
This corresponds to training the node2vec
model on the data with sparsity factor $\rho_n \to \rho_n /2$; in
particular, this leaves the underlying asymptotic representations
unchanged and slows the rate of convergence by a factor of 2. 
With this, a link prediction classifier is formed by the following process:
\begin{enumerate}
    \item Take a set of edges $J \subseteq  \{ (i, j) \,:\, a_{ij} = 1 \}$ for which the node2vec algorithm was not trained on, 
    and a set of non-edges $\tilde{J} \subseteq \{ (i, j) \,:\, a_{ij} = 0 \}$.
    As in practice networks are sparse, these sets are not sampled randomly
    from the network, but are assumed to be sampled in a balanced fashion
    so that the sets $J$ and $\tilde{J}$ are roughly balanced in size. One way of doing so is to pick a number of edges in advance, say $E$, and then
    sample $E$ elements from the set of
    edges and non-edges in order to form $J$ and $\tilde{J}$ respectively.
    \item Form edge embeddings $e_{ij} = f(u_i, u_j)$ given some 
    symmetric function $f(x, y)$ and node embeddings $u_i$. Two popular choices
    of functions are the average function $f(x, y) = (x + y)/2$
    and the Hadamard product $f(x, y) = (x_i y_i)_{i \in [d]}$.
    \item Using the features $e_{ij}$ and the labels provided by the sets $J$ and $\tilde{J}$, build a classifier using your
    favorite ML algorithm.
\end{enumerate}
By our convergence guarantees, we know that the asymptotic distribution of
the edge embeddings $e_{ij}$ will approach some vectors $\eta_{c(i), c(j)} \in \mathbb{R}^d$, giving at most $\kappa^2$ distinct vectors overall. Note that these embedding vectors in of themselves contain little information about whether
the edges are connected; that said, even given perfect information of the communities and the connectivity matrix $P$, one can only form probabilistic guesses as to whether two vertices are connected. That said, by clustering
together the link embeddings we can identify together
edges as having vertices belonging to a particular pair of communities. 
With knowledge of the sampling mechanism, it is then possible to
backout estimates for $p$ and $q$ by counting the overlap
of the sets $J$ and $\tilde{J}$ in the neighbourhoods of the clustered
node embeddings. 

We note that in practice, ML classification algorithms such as
logistic regression are used instead. This instead depends on the typical
geometry of the sets $J$ and $\widetilde{J}$. Suppose we
have a SBM$(n, 2, \tilde{p}, \tilde{q}, \rho_n)$ model. In this case, the set $J$
will approximately consist of $\tilde{p}/2(\tilde{p}+\tilde{q}) \times E$ vectors from $\eta_{11}$,
$\tilde{p}/2(\tilde{p}+\tilde{q}) \times E$ vectors from $\eta_{22}$, $\tilde{q}/2(\tilde{p}+\tilde{q}) \times E$ vectors from $\eta_{12}$
and $\tilde{q}/2(\tilde{p}+\tilde{q}) \times E$ vectors from $\eta_{21}$. In contrast, the set
$\tilde{J}$ will approximately have $E/4$ of each of $\eta_{11}$,
$\eta_{12}$, $\eta_{21}$ and $\eta_{22}$. As a result, in the case where $\tilde{p} \gg \tilde{q}$, a linear classifier (for example) will be biased towards classifying more frequently vectors with $c(i) = c(j)$, which is at least directionally correct.

So far, we have not talked about the particular mechanism used to form
link embeddings from the node embeddings. The Hadamard product is
popular, but particularly difficult to analyze given our results, as
it does not remain invariant to an orthogonal rotation of the embedding
vectors. In contrast, the average link function retains this information.
In the SBM$(n, 2, \tilde{p}, \tilde{q}, \rho_n)$, it ends up giving embeddings
which will asymptotically depend on only whether $c(i) = c(j)$ or not 
(i.e, whether the vertices belong to the same community or not).




\section{Intermediate results}
\subsection{Sampling probabilities for node2vec}

In this section, we derive asymptotic results for the sampling probabilities of edges
within node2vec. We begin by recapping the second-order random walk defined for node2vec. 
To do so, we define
a random process $(X_n)_{n \geq 1}$ via the second-order Markov property
\begin{equation}
    \mathbb{P}\big( X_n = u \,|\, X_{n-1} = s, X_{n-2} = v \big)
    \propto \begin{cases}
        0 & \text{ if } (u, s) \not\in \mcE, \\ 
        1/p & \text{ if } d_{u, v} = 0 \text{ and } (u, s) \in \mcE, \\ 
        1 & \text{ if } d_{u, v} = 1 \text{ and } (u, s) \in \mcE, \\
        1/q & \text{ if } d_{u, v} = 2 \text{ and } (u, s) \in \mcE.
    \end{cases}
\end{equation}
where $d_{u, s}$ denotes the length of the shortest path between $u$ and $s$.
Given the extra information that $(u, s)$ is an edge, $d_{u, v} = 0$ occurs iff
$u = v$, $d_{u, v} = 1$ occurs iff $(u, v)$ is an edge, and $d_{u, v} = 2$ occurs iff
$(u, v)$ is not an edge (as given that $(v, s)$ is an edge, the shortest path must be 
$v \to s \to u$). With this, we select positive samples by selecting $k$ concurrent edges
within the walk (via taking a walk of length $k+1$). 

To initialize the random walk, we note
that for the second order walk we need to specify a distribution on the first two vertices;
for DeepWalk where this collapses down to a first order walk, we only need to specify a 
distribution on ther first vertex. To do so generally, we consider an initial distribution
of selecting the first vertex via $\pi(u) = \deg(u) / \sum_v \deg(v) = \deg(u) / 2 E_n$ with $E_n$ being
the number of edges in the graph (single counting $(u, v) \in \mcE$ and $(v, u) \in \mcE$), and
select the second vertex uniformly at random from those connected to the first. (Note that this is
the transition kernel used for DeepWalk, and so we handle both cases via this argument.) One
can show this is equivalent to selecting an edge uniformly at random. 

For the negative sampling mechanism, we consider the vertices which arose as part
of the positive sampling process - which we denote $V(\mcP)$ - and then sample $l$ vertices 
independently according to the unigram distribution
\begin{equation}
    \mathrm{Ug}_{\alpha}(v \,|\, u, \mcG_n ) = \frac{\deg(v)^{\alpha} }{
    \sum_{v' \neq u} \deg(v)^{\alpha}
    }
\end{equation}
where $u \in V(\mcP)$. We note that the case where $\alpha \to 0$ corresponds to 
the uniform distribution on vertices not equal to $u$.  

\subsubsection{Proof of Theorem~\ref{app:thm:n2v_positive}}
\label{app:sec:n2v_positive}

In this section and the next, it will be convenient to use the notation $\sim_p$ to indicate that
two positive random variables $X_n$ and $Y_n$ are asymptotic in the sense that $|X_n/Y_n - 1| = o_p(1)$
when $n \to \infty$.
If we say such a bound happens uniformly over some free variables - say $X_{n,k} \sim_p Y_{n, k}$
uniformly over $k$ - then this means $\max_k |X_{n, k}/Y_{n, k} - 1| = o_p(1)$. We also make extensive
use of the result that if $X_n^{(i)} \sim_p r_n Y_n^{(i)}$ for $i \in \{0, 1\}$ and $Y_n^{(i)} \in [C^{-1}, C]$
for $C > 1$, then $X_n^{(0)} + X_n^{(1)} \sim_p r_n (Y_n^{(0)} + Y_n^{(1)})$. Indeed, if we write $X_n^{(i)} = Y_n^{(i)} r_n (1 + \epsilon_n^{(i)}$ where $\epsilon_n^{(1)} = o_p(1)$, then 
\begin{equation}
    X_n^{(0)} + X_n^{(1)} = r_n (Y_n^{(0)} + Y_n^{(1)} ) \cdot \Big( 1 + \frac{ Y_n^{(0)}  }{ Y_n^{(0)} + Y_n^{(1)}   } \epsilon_n^{(0)} + \frac{ Y_n^{(1)}  }{ Y_n^{(0)} + Y_n^{(1)}   } \epsilon_n^{(1)}  \Big)
\end{equation}
from which the claimed result follows as the terms weighting the $\epsilon_n^{(1)}$ can be bounded below away from zero, and are bounded above by $1$. We also note that $X_n^{(0)} - X_n^{(1)} = O_p(r_n)$, meaning that the order of magnitude of terms cannot increase (only decrease) by subtracting them.  

As we are interested in the sampling probability of edges within node2vec, it
will be convenient to instead study the first order Markov process 
$Y_n = (X_{n}, X_{n-1})$, as then we instead study the sampling probability of individual
states in a regular Markov chain. We note that normally we use the notation $(u, v)$ to refer an unordered
pair belonging to an edge in a graph, but for the Markov process $(Y_n)_{n \geq 1}$ the order
matters, we will write $Y_n = e_{v \to u}$ whenever $X_n = u$ and $X_{n-1} = v$. In such a
scenario, the random walk is therefore defined on the state space
\begin{equation*}
    S = \bigcup_{(u, v) \in \mcE} \big\{ e_{u \to v}, e_{v \to u} \big\}.
\end{equation*}
with the law of $Y$ given by 
\begin{align}
    \mathbb{P}\big( Y_n = e_{t \to u} \,|\, Y_{n-1} = e_{v \to s} \big) & = 0 \text{ if } 
    t \neq s, \\
    \mathbb{P}\big( Y_n = e_{s \to u} \,|\, Y_{n-1} = e_{v \to s} \big) & \propto 
    \begin{cases}
        0 & \text{ if } (s, u) \not\in \mcE \\ 
         \frac{1[u = v]}{p} + 1[u \neq v] (a_{uv} +  \frac{ 1 - a_{uv} }{q} )& \text{ otherwise.}
    \end{cases}
\end{align}
One can calculate the normalizing factor for the probability distribution as being 
\begin{equation}
    \Big( \frac{1}{p} - \frac{1}{q} \Big) 
        + \frac{1}{q} \deg(s) 
        + \Big( 1 - \frac{1}{q} \Big) \sum_{u \in \mcV \setminus \{ v \} } a_{su} a_{uv},
\end{equation}
from which we observe that when $p=q=1$ we recover the simple random walk defined by
DeepWalk, as then the probability an edge is selected with source node $u$ is uniform over
edges $(u, v)$ where $v$ is a neighbour of $u$. 

With this in mind, we define the transition matrix
\begin{equation}
    P_{v \to s, s \to u} = \frac{ a_{su} \cdot \{  1[u = v] \cdot 1/p + 1[u \neq v] (a_{uv} +  1/q \cdot (1 - a_{uv})   \}   }{  \Big( \frac{1}{p} - \frac{1}{q} \Big) 
    + \frac{1}{q} \deg(s) 
    + \Big( 1 - \frac{1}{q} \Big) \sum_{u \in \mcV \setminus \{ v \} } a_{su} a_{uv}   }
\end{equation}
governing the transition probabilities on the above chain. We note that 
by \citep[Proposition~72]{davison_asymptotics_2023} and Theorem~\ref{app:thm:path_concentration} respectively that 
\begin{align}
    \deg(s) & \sim_p n \rho_n W(\lambda_s, \cdot), \label{app:samp:n2v_deg_asymptotics} \\
     \sum_{u \in \mcV \setminus \{v\} } 
    a_{su} a_{uv} & \sim_p n \rho_n^2 T(\lambda_s, \lambda_v) \text{ where } T(\lambda_s, \lambda_v) 
    := \mathbb{E}_{\lambda \sim \mathrm{Unif}[0, 1]}[ W(\lambda_u, \lambda) W(\lambda, \lambda_v)
    \,|\, \lambda_u, \lambda_v] \label{app:samp:n2v_vcount_asymptotics}
\end{align}
uniformly over all $s, u, v$. 
As a result, we define
\begin{equation}
    \widetilde{P}_{v \to s, s \to u} = \frac{  
        a_{su} \cdot \{
            q^{-1} + (1 - q^{-1}) a_{vu} + \delta_{uv} (p^{-1} - q^{-1})
        \}
    }{  \Big( \frac{1}{p} - \frac{1}{q} \Big) 
    + \frac{1}{q} n \rho_n W(\lambda_s, \cdot)  
    + \Big( 1 - \frac{1}{q} \Big) n \rho_n^2 T(\lambda_s, \lambda_v)   }.
\end{equation}
where $\delta_{uv} := 1[u = v]$ and the numerator is the same as in $P_{v \to s, s \to u}$ 
(only written in a more
convenient to use fashion), and the denominator makes use of the asymptotic statements
\eqref{app:samp:n2v_deg_asymptotics} and \eqref{app:samp:n2v_vcount_asymptotics}. 
As a result, we have that $P_{v \to s, s \to u} \sim_p \widetilde{P}_{v \to s, s \to u}$
uniformly over $v, s, u$. In particular,
we have that $\widetilde{P}_{v \to s, s \to u} = \Theta_p( a_{su} (n \rho_n)^{-1} )$ uniformly
over all triples of indices $(v, s, u)$. 

Let $A_j(u \to v) = \{ Y_j = e_{u \to v} \}$. We then note that the sampling probability
of $(u, v)$ being sampled within the first $k + 1$ steps of the second order random
walk is given by
\begin{equation}
    \mathbb{P}\Big(   \bigcup_{j \leq k} A_j(u \to v) \cup A_j(v \to u) \,|\, \mcG_n \Big).
\end{equation}
To ease on the notation going forward, we write $\mathbb{P}_n(\cdot) := \mathbb{P}(\cdot \,|\, \mcG_n)$.
By the inclusion-exclusion principle, we can write this probability as equalling
\begin{align}
    \label{app:samp:inc_exc}
    \sum_{\substack{l, m \geq 1 \\ l + m \leq k}} (-1)^{k + m+ 1} \sum_{\substack{1 \leq i_1 < i_2 < \cdots 
    < i_l \leq k \\ 1 \leq j_1 < j_2 < \cdots < j_m \leq m}}  
    \mathbb{P}_n\Big( \bigcap_{k \leq l} A_{i_k}(u \to v) \cap
    \bigcap_{k \leq m} A_{j_k}(v \to u)  \Big).
\end{align}
We note that the number of terms in this sum is bounded above by $(2k)!$ (some terms will
be zero, as we cannot select $e_{u \to v}$ two times in a row), and so
for asymptotic purposes we can focus on the individual terms. 

We now address the individual probabilities making up this sum. Intuitively, we want to
show the following: that the terms for which $(l, m) \neq (1, 0)$ or $(0, 1)$
are asymptotically negligible, and that asymptotically these terms are functions only
of $(\lambda_u, \lambda_v)$. We fix a particular instance of
the $i_1, \ldots, i_l$ and $j_1, \ldots, j_m$, and denote
$\beta_1 < \beta_2 < \cdots < \beta_{l+m}$ for the ordering of these indices. As we use 
indices $i_k$ to denote the direction $u \to v$ and $j_k$ for the direction $v \to u$, we
write 
\begin{equation}
    A_i(u \to v) =: A_{\beta}(u, v, 0), \qquad A_j(v \to u) =: A_{\beta}(u, v, 1)
\end{equation}
where the third argument (which we refer to as the orientation herein)
indicates which of the first two arguments are used as the source
node for the edge. For each $\beta_k$ for $k \leq l+m$, we write $o_k$ to denote this
orientation. As a result, it suffices for us to analyze
\begin{equation}
    \mathbb{P}_n\Big( \bigcap_{k \leq l+m} A_{\beta_k}(u, v, o_k) \Big)
\end{equation}
over all sequences $1 \leq \beta_1 < \beta_2 < \cdots < \beta_{l+m} \leq k$ and 
orientations $(o_k)_{k=1}^{l+m}$. For this, we then note that by the Markov property
of the random walk, we are able to write this probability as
\begin{align}
    \Bigg[ \prod_{k \leq l+m-1} & \mathbb{P}_n\Big( A_{\beta_{k+1}}(u, v, o_{k+1}) \,|\, 
    A_{\beta_{k}}(u, v, o_{k}) \Big) \Bigg] \cdot \mathbb{P}_n\big( A_{\beta_{1}}(u, v, o_{1}) 
    \big) \\ 
    & = \Bigg[ \prod_{k \leq l+m-1} \mathbb{P}_n\Big( A_{\beta_{k+1} - \beta_k + 1}(u, v, o_{k+1}) \,|\, 
    A_{1}(u, v, o_k) \Big) \Bigg] \cdot \mathbb{P}_n\big( A_{\beta_{1}}(u, v, o_{1}) 
    \big)
\end{align}
Focusing now on the terms in the product, if $\beta_{k+1} - \beta_k = 1$, then\
this term equals zero if $o_k = o_{k=1}$, or
otherwise equals e.g $P_{u \to v, v \to u}$ which is 
$O_p((n\rho_n)^{-1})$ as discussed above. If the walk is longer, then by
the same argument as in \citep[Proposition~73]{davison_asymptotics_2023}, by conditioning on the second step in the walk one
can show this probability is asymptotically of the same order of a walk of length 
$\beta_{k+1} - \beta_k - 1$ initialized from the uniform distribution on the edges of
$\mcG_n$. As a result, we therefore only need to analyze events of the form
\begin{equation}
    \mathbb{P}_n\big( A_{\beta}(u, v, o)\big)
\end{equation}
which will allow us to then show that the events of the form $(l, m) = (1, 0)$ or $(0, 1)$ are
the only ones we need to consider in the asymptotic expansion. Going forward, we assume that $o=0$, as the sum \eqref{app:samp:inc_exc}
is symmetric
in the orientation $o$ and the arguments are unchanged.

To do so, we begin by writing writing 
$\pi' = ( a_{uv} / |\mcE| )_{u, v}$ for 
the initial distribution provided to $Y_1$. To analyze 
$p_n(u, v, \beta) := \mathbb{P}_n\big( A_{\beta}(u, v, 0))$, note that when $\beta = 1$
we trivially have that this probability equals $a_{uv} / |\mcE|$
and we know that $|\mcE| \sim_p n^2 \rho_n \mcE_W(1)$.
In the case where
$\beta \geq 2$, we consider the set of sequences $\alpha = (\alpha_0, \ldots,
\alpha_{\beta - 2})\in \mcV^{\beta - 1}$, where we then have that
\begin{align}
    p_n(u, v, 2) & = \frac{1}{|\mcE|} \sum_{\alpha_0} a_{\alpha_0, u} P_{\alpha_0 \to u, u \to v} \\
    p_n(u, v, \beta) & =  \frac{1}{|\mcE|} \sum_{\alpha}
    a_{\alpha_0, \alpha_1} \cdot \prod_{j=1}^{\beta } P_{\alpha_{j-1} \to \alpha_{j},
    \alpha_j \to \alpha_{j+1}} \cdot P_{\alpha_{\beta-2} \to \alpha_{\beta-1},
    \alpha_{\beta-1} \to u } P_{\alpha_{\beta-1} \to u,
    u \to v}
\end{align}
for $\beta \geq 3$.

To study these sums, we begin by noting that they are asymptotic to their versions
where we replace $P \to \widetilde{P}$. Indeed, we note that if we have positive sequences
$(a_i)$ and $(b_i)$, then
    \begin{equation}
        \Big| \frac{ \sum_j a_j}{\sum_j b_j} - 1 \Big| = \frac{ | \sum_j b_j (a_j / b_j - 1) |}{ \sum_j b_j}
        \leq \max_j \Big| \frac{a_j}{b_j} - 1 \Big|,
    \end{equation}
and so the fact that we know $P \sim_p \widetilde{P}$ uniformly, means that we can apply this to
obtain asymptotic formulae for their sums also. 
With this, if we write $N(\lambda_s, \lambda_t)$ 
for the denominator 
of $\widetilde{P}_{t \to s, s \to u}$, 
$p_n(u, v, \beta)$ can be asymptotically be decomposed into a linear combination of terms
(bounded in number by a function of $k$ independent of $n$)
of the form
\begin{equation}
    \label{app:samp:formal_sum}
    \frac{ c(p, q) a_{uv}}{|\mcE|} \sum_{\alpha \in \mcV^{\beta - 1}} \Bigg\{ \Big(  \prod_{2 \leq i \leq \beta} 
        N(\lambda_{\tilde{\alpha}_{i-1}}, \lambda_{\tilde{\alpha}_i})
     \Big)^{-1} \cdot 
     \prod_{i \leq \beta - 1} a_{ \tilde{\alpha}_{i-1}, \tilde{\alpha}_i }
     \cdot \prod_{j \in J} a_{ \tilde{\alpha}_{j-1}, \tilde{\alpha}_{j+1} }
     \cdot \prod_{k \in K} \delta_{ \tilde{\alpha}_{k-1}, \tilde{\alpha}_{k+1}  }
     \Bigg\} 
\end{equation}
where:
\begin{itemize}
    \item we write $\tilde{\alpha}$ for the concatenation $(\alpha, u, v)$, meaning
    $\tilde{\alpha}$ is of length $\beta + 1$, with $\tilde{\alpha}_k = \alpha_k$ for $k \leq \beta - 1$, 
$\tilde{\alpha}_{\beta} = u$ and $\tilde{\alpha}_{\beta + 1} = v$;
    \item $c(p, q) = (q^{-1})^{\beta - |J| - |K|}(1 - q^{-1})^{|J|}(p^{-1} - q^{-1})^{|K|}$ is a polynomial in $p^{-1}$ and $q^{-1}$;
    \item $J$ and $K$ are possibly empty subsets of $\{1, \ldots, \beta\}$ 
    which are disjoint.
\end{itemize}
The more tedious part to handle is when the set $K$ is non-empty; as each delta function
acts to contract the sum along one variable, doing so allows us to rewrite
\eqref{app:samp:formal_sum} as
\begin{equation}
    \label{app:samp:formal_sum2}
    \frac{a_{uv}}{|\mcE|} c(p, q) \sum_{\alpha \in \mcV^{\beta - 1 - |K|}} \Bigg\{ \Big(  \prod_{2 \leq i \leq \beta - |K|} 
        N(\lambda_{\tilde{\alpha}_{i-1}}, \lambda_{\tilde{\alpha}_i})^{n_i}
     \Big)^{-1} \cdot 
     \prod_{i \leq \beta - 1 - |K|} a_{ \tilde{\alpha}_{i-1}, \tilde{\alpha}_i }
     \cdot \prod_{j \in \tilde{J}} a_{ \tilde{\alpha}_{j-1}, \tilde{\alpha}_{j+1} }
     \Bigg\} 
\end{equation}
after a) performing some relabeling of the indices and modification to the set $J$, to
give a new set $\tilde{J}$ which is a subset of $\{1, \ldots, \beta - |K|\}$ and b)
introducing
some multiplicities $n_i$ which sum to $\beta - 1$. By
Theorem~\ref{app:thm:path_concentration} we uniformly have that this quantity is
asymptotic, uniformly over all the free variables in the expression, to
\begin{equation}
    \frac{ \rho_n^{|\tilde{J}|}  }{  (n \rho_n)^{|K|} } \cdot \frac{ a_{uv} c(p, q) \rho_n^{-1}}{n^2 \mcE_W(1)} 
    \cdot \mathbb{E}\Bigg[  
        \frac{ 
            \prod_{i \leq \beta - 1 - |K|} W(\lambda'_{i-1}, \lambda'_i) \prod_{j \in \tilde{J}} 
            W(\lambda'_{j-1}, \lambda'_{j+1})
        }{ \prod_{2 \leq i \leq \beta - |K|} N'( \lambda'_{i-1}, \lambda'_i)^{n_i}   }
    \,|\, \lambda_u, \lambda_v \Bigg]
\end{equation}
where we write $\lambda' = (\widetilde{\lambda}_0, \ldots, 
    \widetilde{\lambda}_{\beta - 2 - |K|}, \lambda_u, \lambda_v)$ and $\widetilde{\lambda}$ is an
    independent copy of $\lambda$, and $N'(\lambda_u, \lambda_v) := (n \rho_n)^{-1} N(\lambda_u, \lambda_v)$. 
As $n \rho_n \to \infty$ under the prescribed conditions, we only need to consider
leading terms of the order $\rho_n^{-1} \ n^2$, which shows that the sampling probability is asymptotic (uniformly over all vertices) to $\rho_n^{-1} \ n^2$ for some function $g_{\mcP}(\lambda_u, \lambda_v)$. To argue that this function is bounded above away from zero, we note that the terms where $|J| + |K| > 0$ will be asymptotically negligible, and the remainder of the terms give a positive weighted sum. 

\subsubsection{Proof of Theorem~\ref{app:thm:n2v_negative}}
\label{app:sec:n2v_negative}

To understand the selection probability for the vertex pair $(u, v)$ to be selected
via negative sampling, define the events
\begin{equation}
    A_i(u) = \{ X_i = u \}, \qquad B_i(v | u) = \{ v \text{ selected via negative sampling from u} \} 
\end{equation}
so then 
\begin{equation}
    \mathbb{P}( (u, v) \in \mcN(\mcG_n) \,|\, \mcG_n)
    = \mathbb{P}\Big( \bigcup_{i=0}^k (A_i(u) \cap B_i(v|u)) \cup (A_i(v) \cap B_i(u|v) ) \,|\, \mcG_n \Big).
\end{equation}
We note that
\begin{equation}
    \mathbb{P}(A_i(u) \cap B_i(v|u) \,|\, \mcG_n) 
    = \mathbb{P}(A_i(u) \,|\, \mcG_n) \cdot \mathbb{P}( \mathrm{Binomial}(l, \mathrm{Ug}_{\alpha}(v | u))
    \geq 1 \,|\, \mcG_n).
\end{equation}
As a result, we need to begin by understanding the asymptotic probabilities of $\mathbb{P}(A_i(v) \,|\, 
\mcG_n)$ and the unigram sampling probability. We begin with understanding the first
probability. If $i \in \{0, 1\}$, then we have that $\mathbb{P}( A_i(v) \,|\, \mcG_n) = \mathrm{deg}(v) / 2 E_n \sim_p W(\lambda_v, \cdot) / n \mcE_W(1)$ uniformly in $v$ \citep[Proposition~72]{davison_asymptotics_2023}. For $i \geq 2$, we have that
\begin{equation}
    \mathbb{P}(A_i(v) \,|\, \mcG_n) = \sum_u \mathbb{P}( A_i(u \to v) \,|\ \mcG_n)
\end{equation}
using the same notation as in Appendix~\ref{app:sec:n2v_positive}. Consequently, via the same
arguments as in Appendix~\ref{app:sec:n2v_positive}, it will be asymptotic to a positive linear combination
of statistics of the form 
\begin{equation}
    \frac{ c(p, q) }{ |\mcE| } \sum_{\alpha \in \mcV^{\beta}}
     \Bigg\{ \Big(  \prod_{2 \leq i \leq \beta - |K|} 
        N(\lambda_{\tilde{\alpha}_{i-1}}, \lambda_{\tilde{\alpha}_i})^{n_i}
     \Big)^{-1} \cdot 
     \prod_{i \leq \beta - |K|} a_{ \tilde{\alpha}_{i-1}, \tilde{\alpha}_i }
     \cdot \prod_{j \in \tilde{J}} a_{ \tilde{\alpha}_{j-1}, \tilde{\alpha}_{j+1} }
     \Bigg\} 
\end{equation}
where we write $\tilde{\alpha} = (\alpha, v)$ for $\alpha \in \mcV^{\beta}$. Using the same relabeling
and arguments as given in Appendix~\ref{app:sec:n2v_positive} will be asymptotic to
\begin{equation}
 \frac{ \rho_n^{|\tilde{J}|}  }{  (n \rho_n)^{|K|} } \cdot \frac{ c(p, q) }{n \mcE_W(1) } \cdot
 \mathbb{E}\Bigg[  
        \frac{ 
            \prod_{i \leq \beta - |K|} W(\lambda'_{i-1}, \lambda'_i) \prod_{j \in \tilde{J}} 
            W(\lambda'_{j-1}, \lambda'_{j+1})
        }{ \prod_{2 \leq i \leq \beta - |K|} N'( \lambda'_{i-1}, \lambda'_i)^{n_i}   }
    \,|\, \lambda_v \Bigg]
\end{equation}
uniformly in all the free variables involved, where $\lambda' = (\widetilde{\lambda}_0, \ldots, \widetilde{\lambda}_{\beta - 1 - |K|}, \lambda_v)$ and $\widetilde{\lambda}$ is an independent copy of $\lambda$. (We note that while Theorem~\ref{app:thm:path_concentration} is expressed in terms of concentration of quantities around functions which depend on both $\lambda_u$ and $\lambda_v$, the exact same reasoning will apply for statistics which only end up depending on $\lambda_v$.) In particular by taking the highest order terms of this expansion, we have that there exists some measurable function $g_{i}(\cdot)$ which is bounded below and above, for each $i$, such that
$\mathbb{P}(A_i(u) \,|\, \mcG_n) \sim_p n^{-1} g_i(\lambda_u)$ uniformly in $u$.

As for the unigram sampling term, we note that by \citep[Proposition~77]{davison_asymptotics_2023} we have that
\begin{equation}
\mathbb{P}( \mathrm{Binomial}(l, \mathrm{Ug}_{\alpha}(v | u)) \sim_p \frac{ l W(\lambda_u, \cdot)^{\alpha}}{ n \mcE_W(\alpha)}
\end{equation}
uniformly in the vertices $v, u$. With this, we note that the same arguments via self-intersection allow us to argue that
\begin{equation}
    \mathbb{P}( (u, v) \in \mcN(\mcG_n) \,|\, \mcG_n) 
    \sim_p \frac{l}{n^2} \sum_{i=0}^k \frac{l}{\mcE_W(\alpha)} ( g_i(\lambda_u) W(\lambda_v, \cdot)^{\alpha}
    + g_i(\lambda_v) W(\lambda_u, \cdot)^{\alpha} )
\end{equation}
which gives the claimed result.



\subsection{Chaining and bounds on Talagrand functionals}

In this section, let $L > 0$ denote a universal constant 
(which may differ across occurrences)
and $K(\alpha)$ a universal constant which depends on a 
variable $\alpha$ (but for fixed $\alpha$ also differs across occurrences).
For a metric space $(T, d)$, we define the \emph{diameter} of $T$ as
\begin{equation}
    \Delta(T) := \sup_{t_1, t_2 \in T} d(t_1, t_2).
\end{equation}
We also define the entropy and covering numbers respectively by
\begin{align}
    N(T, d, \epsilon) & := \min\big\{ n \in \mathbb{N} \,|\, F \subseteq T, |F| \leq n,
    d(t, F) \leq \epsilon \text{ for all } t \in T \big\}, \\
    e_n(T) & := \inf\big\{ 
        \sup_{t \in T} d(t, T_n) \,|\, T_n \subseteq T, |T_n| \leq 2^{2^n}
    \big\} = \inf \big\{ \epsilon > 0 \,|\, N(t, d, \epsilon) \leq 2^{2^n} \big\}.
\end{align}
We then define the \emph{Talagrand $\gamma_{\alpha}$ functional} 
\citep{talagrand_upper_2014}
of the metric
space $(T, d)$ by 
\begin{equation}
    \gamma_{\alpha}(T, d) = \inf \sup_{t \in T} \sum_{n \geq 0} 2^{n / \alpha} 
    \Delta\big( A_n(t) \big)
\end{equation}
where the infimum is taking over all \emph{admissable sequences};
these are increasing sequences $(\mcA_n)_{ n\geq 0}$ of $T$ such that 
$| \mcA_0 | = 1$ and $| \mcA_n | \leq 2^{2^n}$ for all $n$, with $A_n(t)$ being
the unique element of $\mcA_n$ which contains $t$. We will shortly see that
this quantity helps to control the supremum of empirical processes on 
the metric space $(T, d)$. We first give some generic properties for the above
functional.


\begin{lemma}
    \label{app:thm:gamma}
    \begin{enumerate}[label=\alph*)]
        \item Suppose that $d$ is a metric on $T$, and $M > 0$ is a constant.
        Then $\gamma_{\alpha}(T, Md) = M \gamma_{\alpha}(T, d)$. If
        $U \subseteq T$, then $\gamma_{\alpha}(U, d) \leq \gamma_{\alpha}(T, d)$.
        \item Suppose that $(T_1, d_1)$ and $(T_2, d_2)$ are metric spaces,
        so $d = d_1 + d_2$ is a metric on the product space $T = T_1 \times T_2$.
        Then $\gamma_{\alpha}(T, d) \leq K(\alpha) (\gamma_{\alpha}(T_1, d_1) +
        \gamma_{\alpha}(T_2, d_2))$.
        \item We have the upper bounds 
        \begin{equation}
            \gamma_{\alpha}(T, d) \leq K(\alpha) \sum_{n \geq 0} 2^{n / \alpha} e_n(T)
            \leq K(\alpha) \int_0^{\infty} \big( \log N(T, d, \epsilon)  \big)^{1/\alpha}
            \, d \epsilon.
        \end{equation}
        \item Suppose that $\| \cdot \|$ is a norm on $\mathbb{R}^m$, $d$ is the
        metric induced by $\| \cdot \|$, and 
        $B_A = \{ x \,:\, \| x \| \leq A \}$. Then one has the bound
        $N(B_A, d, \epsilon) \leq \max\{(3 A / \epsilon)^m, 1\}$, and consequently
        $\gamma_{\alpha}(B_A, d) \leq K(\alpha) A m^{1/\alpha}$. 

    \end{enumerate}
\end{lemma}

\begin{proof}
    The first statement in a) is immediate, and the second
    part is Theorem~2.7.5 a) of \citet{talagrand_upper_2014}. 
    
    For
    part b), suppose that $\mcA^{i}_n$ are admissable sequences for $(T_i, d_i)$
    such that
    \begin{equation}
        \sup_{t_i \in T_i} \sum_{n \geq 0} 2^{n / \alpha} \Delta(A_n^i(t)) \leq
        2 \gamma_{\alpha}(T_i, d_i) \text{ for } i = 1, 2.
    \end{equation}
    If we then form the 
    sequence of sets $\mcB_n := \{A_1 \times A_2 \,:\, A_i \in \mcA^{i}_{n-1}\}$
    for $n \geq 1$ and $\mcB_0 = T_1 \times T_2$, 
    we have that $\mcB_n$ is a partition of $T$ for each $n$, $|\mcB_0| = 1$ and 
    $| \mcB_n| = | \mcA^{1}_{n-1} | \cdot | \mcA_{n-1}^{2} | \leq 2^{2^n}$
    for each $n$, meaning that $\mcB_n$ is an admissable sequence for the metric
    space $(T, d)$. Moreover, note that we have
    \begin{equation}
        \Delta( (A_1 \times A_2)(t_1, t_2) ) = \Delta( A_1(t_1) ) + \Delta(A_2(t_2)) 
    \end{equation}
    for all sets $A_1 \subseteq T_1$, $A_2 \subseteq T_2$ and $t_1 \in T_1$, 
    $t_2 \in T_2$. As a result, if write $B_n(t_1, t_2) = A_{n-1}^1(t_1) \times 
    A_{n-1}^2(t_2)$ for the unique set in $\mcB_n$ for which the point $(t_1, t_2)$
    lies within it, then we have that 
    \begin{equation}
        \sum_{n \geq 0} 2^{n / \alpha} \Delta(B_n(t_1, t_2)) \leq
        2^{\alpha} \Big(    
            \sum_{n \geq 1} 2^{(n - 1) / \alpha} \Delta(A_{n-1}^i(t_1)) +
            \sum_{n \geq 1} 2^{(n - 1) / \alpha} \Delta(A_{n-1}^i(t_2))
        \Big).
    \end{equation}
    In particular, taking supremum over all $t \in T$ then gives the result,
    as the resuling LHS is lower bounded by $\gamma_{\alpha}(T, d)$, and the 
    resulting RHS is upper
    bounded by $2(\gamma_{\alpha}(T_1, d_1) + \gamma_{\alpha}(T_2, d_2))$.

    For part c), the first inequality is Corollary~2.3.2 in \citet{talagrand_upper_2014}. 
    As for the second inequality, 
    note that if $\epsilon \leq e_n(T)$, then $N(T, d, \epsilon) > 2^{2^n}$
    and consequently $N(T, d, \epsilon) \geq 2^{2^n} + 1$ (recall that
    both quantities are integers). Writing $N_n = 2^{2^n}$, this implies that
    \begin{equation}
        \big(  \log(1 + N_n) \big)^{1/\alpha} (e_n(T) - e_{n+1}(T)) \leq 
        \int_{e_{n+1}(T)}^{e_n(T)} \big( \log N(T, d, \epsilon)  \big)^{\alpha} \, 
        d \epsilon.
    \end{equation}
    As $\log(1+ N_n) \leq 2^n \log(2)$ for all $n \geq 0$, summation over
    all $n \geq 0$ implies that 
    \begin{equation}
        ( \log 2 )^{1/\alpha} \sum_{n \geq 0} 2^{n / \alpha }
        (e_n(T) - e_{n+1}(T)) \leq \int_0^{e_0(T)} 
        \big( \log N(T, d, \epsilon)  \big)^{\alpha} \, d\epsilon. 
    \end{equation}
    As we have that
    \begin{equation}
        \sum_{n \geq 0} 2^{n / \alpha} \big( e_n(T) - e_{n+1}(T)) 
        \geq (1 - 2^{1/\alpha}) \sum_{n \geq 0} 2^{n/\alpha} e_n(T),
    \end{equation}
    combining this and the prior inequality gives the stated result. 

    For part d), we can calculate that
    \begin{align}
        \int_0^{\infty} \big(  \log N(B_A, d, \epsilon) \big)^{1/\alpha} \, d \epsilon 
        \leq \int_0^{3A} m^{1/\alpha} \big( \log( 3A / \epsilon) \big)^{1\alpha} \, d\epsilon 
        \leq 3 A m^{1/\alpha}  \int_0^1 ( \log(1/y) )^{1/\alpha} \, dy.
    \end{align}
    For the remaining integral, note that if we make the substitution
    $ y = \exp(-t^{\alpha})$, then the integral equals
    \begin{equation}
        \int_0^1 ( \log(1/y) )^{1/\alpha} \, dy = 
        \alpha \int_0^{\infty} t^{\alpha} e^{-t^{\alpha}} \, dt,
    \end{equation}
    which we recognize as the mean of an $\mathrm{Exp}(1)$ random variabe in the
    case where $\alpha = 1$, and the variance of an unnormalized
    $\mathrm{N}(0, 2)$ density in the case where $\alpha = 2$, and
    so in both cases the integral is finite. The desired conclusion follows.
\end{proof}

Before stating a corollary of this result involving bounds on the
$\gamma$-functional of some of the
sets introduced in Theorem~\ref{thm:gram_converge:adjacency_average},
we discuss some of the properties of these sets.

\begin{lemma}
    \label{app:thm:set_bounds}
    Define the sets
    \begin{align}
        \mcB_F(A) & := \big\{ 
            U \in \mathbb{R}^{n \times d} \,|\, \| U \|_F \leq A
        \big\}, \\
        \mcB_{2, \infty}(A) & := \big\{ 
            U \in \mathbb{R}^{n \times d} \,|\, \| U \|_{2, \infty} \leq A
        \big\}.
    \end{align}
    Moreover, define the metrics
    \begin{align}
        d_F((U_1, V_1), (U_2, V_2)) & 
            := \| U_1 - U_2 \|_F + \| V_1 - V_2 \|_F \\
        d_{2, \infty}((U_1, V_1), (U_2, V_2)) & 
            := \| U_1 - U_2 \|_{2, \infty} + \| V_1 - V_2 \|_{2, \infty}
    \end{align}
    defined on the space $\mathbb{R}^{n \times d} \times \mathbb{R}^{n \times d}$
    of pairs of $n \times d$ matrices. Then we have that for
    $U_1, U_2, V_1, V_2 \in \mcB_F(A_F) \cap \mcB_{2, \infty}(\tilde{A}_{2, \infty})$
    that
    \begin{equation}
        \| U_1 V_1^T - U_2 V_2^T \|_F \leq A_F d_F((U_1, V_1), (U_2, V_2)), 
        \qquad \| U_1 V_1^T - U_2 V_2^T \|_{\infty} \leq \tilde{A}_{2, \infty}
        d_{2, \infty}((U_1, V_1), (U_2, V_2)).
    \end{equation}
    Moreover, if $U \in \mcB_{2, \infty}(A)$, then $U \in \mcB_F(\sqrt{n} A)$
    also, and consequently if $U \in \mcB_{2, \infty}(\tilde{A}_{2, \infty})$
    then we have that $U \in \mcB_{2, \infty}(\tilde{A}_{2, \infty}) \cap 
    \mcB_F( \sqrt{n} \tilde{A}_{2, \infty})$. 
\end{lemma}

\begin{proof}
    Begin by noting that, if $U_1, V_1, U_2, V_2 \in \mathbb{R}^{n \times d}$
    are matrices, then we have that
    \begin{align*}
        \| U_1 V_1^T - U_2 V_2^T \|_F & 
        = \| U_1 (V_1 - V_2)^T + (U_1 - U_2) V_2^T \|_F 
        \leq \| U_1 \|_F \|V_1 - V_2 \|_F + \| U_1 - U_2 \|_F \| V_2 \|_F
    \end{align*}
    and similarly 
    \begin{align*} 
        \| U_1 V_1^T - U_2 V_2^T \|_{\infty} & 
        = \| U_1 (V_1 - V_2)^T + (U_1 - U_2) V_2^T \|_{\infty} 
        \leq \| U_1 \|_{2, \infty} \|V_1 - V_2 \|_{2, \infty} 
        + \| U_1 - U_2 \|_{2, \infty} \| V_2 \|_{2, \infty}.
    \end{align*}
    As a result, we therefore have that in the case where
    $U_1, V_1, U_2, V_2$ all have $\| \cdot \|_{F} \leq A_F$, then
    \begin{equation}
        \| U_1 V_1^T - U_2 V_2^T \|_F \leq A_F \big( \| U_1 - U_2 \|_F + \| V_1 
        - V_2 \|_F   \big)
    \end{equation}  
    and similarly if each of $U_1, V_1, U_2, V_2$ have $\| \cdot \|_{2, \infty}
    \leq \tilde{A}_{2, \infty}$ then
    \begin{equation}
        \| U_1 V_1^T - U_2 V_2^T \| \leq \tilde{A}_{2, \infty} \big(  
            \| U_1 - U_2 \|_{2, \infty} + \| V_1 - V_2 \|_{2, \infty}
        \big),
    \end{equation}
    giving the first result of the lemma. The second part follows by noting that
    \begin{equation}
        \sum_{i=1}^n \sum_{j=1}^d | u_{ij} |^2 \leq n \max_{i \in [n] }\sum_{j=1}^d 
        | u_{ij} |^2 
    \end{equation}
    and taking square roots. 
\end{proof}

\begin{corollary}
    \label{app:thm:gamma_set_bounds}
    With the same notation as in Lemma~\ref{app:thm:set_bounds}, 
    and writing $T = \mcB_F(A_F) \cap \mcB_{2, \infty}(\tilde{A}_{2, \infty})$,
    we have that for any constant $C > 0$ that 
    \begin{align}
        \gamma_{\alpha}(T \times T, C d_F) & \leq \gamma_{\alpha}(B_F(A_F), C d_F)
        \leq K(\alpha) \cdot C A_F (nd)^{1/\alpha} \leq K(\alpha) \cdot C \tilde{A}_{2, \infty}
        n^{1/2 + 1/\alpha} d^{1/\alpha}, \\
        \gamma_{\alpha}(T \times T, C d_{2, \infty}) & \leq \gamma_{\alpha}(B_{2, \infty}(\tilde{A}_{2, \infty}), C d_F)
        \leq K(\alpha) \cdot C \tilde{A}_{2, \infty} (nd)^{1/\alpha}.
    \end{align}
\end{corollary}

\begin{proof}
    This is a combination of
    Lemma~\ref{app:thm:gamma} and Lemma~\ref{app:thm:set_bounds}
\end{proof}

We now state a result which illustrates the usefulness of the above quantity when
trying to control the supremum of empirical processes on a metric space $(T, d)$.

\begin{theorem}
    \label{app:thm:chaining}
    Suppose $(X_t){t \in T}$ is a mean-zero stochastic process, where
    $d_1$ and $d_2$ are two metrics on $T$. Suppose for all $s, t \in T$
    we have the inequality
    \begin{equation}
        \mathbb{P}\big( |X_s - X_t | \geq u \big)
        \leq 2 \exp\Big( - \min\big\{    
            \frac{u^2}{d_2(s, t)^2}, \frac{u}{d_1(s, t)}
        \big\} \Big).
    \end{equation}
    Then we have that
    \begin{equation}
        \mathbb{P}\big(  
            \sup_{s, t \in T}|X_s - X_t| \geq 
            L u \big( \gamma_2(T, d_2) + \gamma_1(T, d_1)  \big)\big) \leq L \exp(-u).
    \end{equation}
\end{theorem}

\begin{proof}
    This can be found within
    the proof of Theorem~2.2.23 in \citet{talagrand_upper_2014}.
\end{proof}

\begin{corollary}
    \label{app:thm:chain_corr}
    With the notation of Theorem~\ref{thm:gram_converge:adjacency_average}, 
    Lemma~\ref{app:thm:set_bounds} and Corollary~\ref{app:thm:gamma_set_bounds}, 
    if we have the bound
    \begin{align}
        \mathbb{P}\big( |&E_n(U, V) - E_n( \tilde{U}, \tilde{V} )| \geq u \big) 
        \\
        & \leq 
        2 \exp\Big( - \min\Big\{    
            \frac{u^2}{ 128 \rho_n^{-1} n^{-4} A_F^2 
            d_F((U, V), (\tilde{U}, \tilde{V}))^2   }, 
            \frac{u}{16 \rho_n^{-1} n^{-2} \tilde{A}_{2, \infty}    
            d_{2, \infty}((U, V), (\tilde{U}, \tilde{V}))
            }
        \Big\}  \Big)
    \end{align}
    then as a consequence we can deduce that
    \begin{equation}
        \sup_{(U, V), (\tilde{U}, \tilde{V}) \in T \times T} \big| E_n(U, V) - 
        E_n(\tilde{U}, \tilde{V}) \big| = O_p\Big(  
            \tilde{A}_{2, \infty}^2 \Big( \frac{ d }{ n \rho_n} \Big)^{1/2}  +
            \tilde{A}_{2, \infty}^2 \frac{ d }{ n \rho_n}  \Big) 
    \end{equation}
\end{corollary}

\begin{proof}
    This is a consequence of Corollary~\ref{app:thm:gamma_set_bounds}
    and Theorem~\ref{app:thm:chaining}.
\end{proof}
\subsection{Matrix Algebra}


\begin{proposition}
    \label{thm:mat:procrustes_bound}
    Suppose that we have matrices $U, X \in \mathbb{R}^{n \times d}$
    with $n \geq d$, and 
    suppose that $X$ is a full rank matrix so $\sigma_d(XX^T) > 0$. Then
    we have that
    \begin{equation}
        \min_{Q \in O(d)} \frac{1}{n}  \| U - XQ \|_F^2
        \leq \frac{ n^{-2} \| UU^T - XX^T \|_F^2     }{ 
            \sqrt{2} (\sqrt{2} - 1) n^{-1} \sigma_d(XX^T) 
        }.
    \end{equation}
    Now instead suppose we have matrices $U, V \in \mathbb{R}^{n \times d}$
    and a matrix $M \in \mathbb{R}^{n \times d}$ of rank $d$. Let 
    $M = U_M \Sigma V_M^T$ be a SVD of $M$. Moreover
    suppose that $U^T U = V^T V$, and $\| UV^T - M \|_{\text{op}} \leq 
    \sigma_d(M) / 2$. Then we have that
    \begin{equation}
        \min_{Q \in O(d)} \frac{1}{n} \| U - U_M \Sigma^{1/2} Q \|_F^2 \leq 
        \frac{ 2 n^{-2} \| UV^T - M \|_F^2 }{(\sqrt{2} - 1) n^{-1} \sigma_d(M)}.
    \end{equation}     
\end{proposition}

\begin{proof}
    The first part of the theorem statement is 
    Lemma~5.4 of \citet{tu_low-rank_2016}. For the second part,
    we note that by Proposition~\ref{thm:matrix:uvvu}, we can let
    $U = U_M \Sigma^{1/2} Q $ and $V = V_M \Sigma^{1/2} Q$ for some
    orthonormal matrix $Q$, where $\tilde{U} \tilde{\Sigma} \tilde{V}^T$ is the SVD
    of $UV^T$. As a result, we can therefore apply without loss of generality
    Lemma~5.14 of \citet{tu_low-rank_2016}, which then gives the
    desired statement.
\end{proof}

\begin{proposition}
    \label{thm:matrix:uvvu}
    Suppose that $U, V \in \mathbb{R}^{n \times d}$ are matrices such that
    $UV^T  = M$ for some rank $d$ matrix $M \in \mathbb{R}^{n \times n}$.
    Moreover suppose that $U^T U = V^T V$. Let $M = U_M \Sigma V_M^T$
    be the SVD of $M$. Then there exists an orthonormal matrix
    $Q \in O(d)$ such that $V = V_M \Sigma^{1/2} Q$. In particular,
    the symmetry group of the mapping $(U, V) \to UV^T$ under
    the constraint $U^T U = V^T V$ is exactly the orthogonal group
    $O(d)$. 
\end{proposition}

\begin{proof}
    Begin by noting that the condition $U^T U = V^T V$ forces there to exist
    an orthonormal matrix $R \in O(n)$ such that $RU = V$
    (e.g by Theorem~7.3.11 of \citet{horn_matrix_2012}). As a consequence,
    we therefore have that $M = R^{-1} VV^T$. This is a polar decomposition
    of $M$, and therefore as the semi-positive definite factor is unique,
    we have that $VV^T = (V_M \Sigma^{1/2}) (V_M \Sigma^{1/2})^T$, where
    $M = U_M \Sigma V_M^T$ is the SVD of $M$, and we highlight that the
    polar decomposition of $M$ is usually represented by $M = (U_M V_M^{-1})
    \cdot (V_M \Sigma V_M^T)$. As $VV^T = (V_M \Sigma^{1/2}) (V_M \Sigma^{1/2})^T$,
    again by e.g Theorem~7.3.11 of \citet{horn_matrix_2012} 
    we have that there exists an orthonormal matrix $Q \in O(d)$ such that
    $V = V_M \Sigma^{1/2} Q$, giving the desired result. 
\end{proof}

\begin{lemma}
    \label{thm:mat:assignment_mat_svals}
    Suppose $X \in \mathbb{R}^{n \times n}$ is a symmetric matrix
    such that $X = \Pi A \Pi^T$ where $A \in \mathbb{R}^{d \times d}$ 
    is of full rank, and
    $\Pi \in \mathbb{R}^{n \times d}$ is the assignment matrix for a partition of
    $[n]$; that is, there exists a partition of $[n]$ into $d$ sets 
    $B(1), \ldots, B(d)$ such that $\Pi_{il} = 1[i \in B(l)]$. Suppose further
    that $\Pi$ is of full rank. Then we have that 
    $\sigma_d(X) \geq \sigma_d(A) \times \min_l |B(l)|$. 
\end{lemma}

\begin{proof}
    Let $\Delta = \mathrm{diag}(|B(1)|^{1/2}, \ldots, |B(d)|^{1/2})$. Then note that
    we can write 
    \begin{equation}
        X = ( \Pi \Delta^{-1}) \cdot \Delta A \Delta \cdot (\Pi \Delta)^{-1}
    \end{equation}
    where $(\Pi \Delta^{-1})$ is an orthonormal matrix. As a result,
    we can simply concentrate on the spectrum of the matrix $\Delta A \Delta $.
    As the smallest singular value of a matrix product is less than the product
    of the smallest singular values, the stated result follows.
\end{proof}

\subsection{Concentration inequalities}

\begin{theorem}
\label{app:thm:path_concentration}
    Suppose that $H$ is a graph on a vertex set $\{ r_1, \ldots, r_l, v_1, \ldots, v_m \}$ where the vertices $r_i$ are referred to as root vertices, and the remaining vertices as free vertices. We refer to such a graph as a \emph{rooted graph}. Suppose that all the edges in $H$ have at least one free vertex as an endpoint. Write $\bold{x} = (x_1,
    \ldots, x_m)$ for the collection of $m$ variables $x_i$, and 
    let $Y$ be a statistic of the form 
    \begin{equation}
        Y = \sum_{x_1, \ldots, x_m \in [n] } g_{\bold{x}} \prod_{i \sim_H j} 
        t_{x_i, x_j}
    \end{equation}
    where the random variables $t_{x_i, x_j}$ are independent and $\{0, 1\}$ valued with $c_p \leq \mathbb{P}(t_{x_i, x_j} = 1) \leq 1 - c_p$ for all $x_i, x_j$; the coefficients $c_g \leq g_{x} \leq \| g \|_{\infty} < \infty$ for some $c_g > 0$; and $i \sim_H j$ iff $(i, j)$ is an edge within the graph $H$. Suppose that
    $\rho_n = n^{-\alpha}$ for some $\alpha < 1 / m'(H)$ where $m'(H) = \max_{2 \leq j \leq k} (j-1)/(v(j) - 2)$, $v(j) = \min_{|A| \geq j} v(A)$
    and $v(A)$ for a set of edges $A$ indicates the number of vertices in $A$.
    Then there exist constants $c, \delta, \Delta$ which depend only on $c_g$, $c_p$, $\|g\|_{\infty}$, $H$ and $\alpha$ such that
    \begin{equation}
        \mathbb{P}\big( \big| Y - \mathbb{E}[Y] \big| \geq \mathbb{E}[Y] \sqrt{ \lambda (n^2 \rho_n)^{-1} } \big) \leq \exp(-c\lambda)
    \end{equation}
    for all $\Delta \leq \lambda \leq n^{\delta}$.
\end{theorem}

\begin{proof}
    Without loss of generality suppose that $\| g \|_{\infty} = 1$.
    The proof is essentially the same as \citet[Corollary~6.4]{vu_conc_2002}, where we extend the result derived for the asymptotics of subgraph counts to that of a weighted count of rooted subgraph counts. To do so, we introduce 
    some notation introduced within \cite{vu_conc_2002}. If $H$ has $k$ edges, and $A$ is a set of pairs $\{x_i, x_j\}$, we write $\partial_A T$ for the polynomial $\prod_{x \in A} \partial_x T$ when interpreting $T$ as a formal sum in the variables $a_{x_i, x_j}$ (which we recall are $\{0, 1\}$ valued. We then define for $1 \leq j \leq k$ the quantities
    \begin{equation}
        \mathbb{E}_j[ Y] = \max_{|A| \geq j} \mathbb{E}[ \partial_A Y],
        M_j(Y) = \max_{t, |A| \geq j} \partial_A Y(t).
    \end{equation}
    Let $v(A)$ denote the number of vertices specified within the set $A$, and
    let $v(j) - \min_{|A| \geq j} v(A)$. With this, we note that $\mathbb{E}[Y] = \Theta(n^m \rho_n^k)$ and $\mathbb{E}[ \partial_A Y) = \Theta(n^{m - v(A)} \rho_n^{k - |A|})$. Consequently, we have that 
    \begin{equation}
        \mathbb{E}_j[Y] = \max_{h \geq j} \Theta( n^{m - v(h) } \rho_n^{k - h} ), \mathbb{E}[Y] / \mathbb{E}_j[Y] = \Theta( \min_{h \geq j} n^{v(h)} \rho_n^h)
    \end{equation}
    where the implied constants depend only on $k$, $c_g$ and $c_p$. The same arguments as given in Claim~6.2 and Corollary~6.4 in \cite{vu_conc_2002} can then be applied verbatim to give the claimed result. 
\end{proof}

\begin{lemma}
    \label{app:thm:ustat_concentration}
    Let $T$ be a statistic of the form
    \begin{equation}
        T' = \sum_{x_1 \neq x_2 \neq \cdots \neq x_m} g(\lambda_{x_1}, \ldots, \lambda_{x_m})
    \end{equation}
    where $c_g \leq g(\cdot) \leq \|g\|_{\infty} < \infty$. Then we have that 
    \begin{equation}
        \mathbb{P}\big( | T' - \mathbb{E}[T'] | \geq \epsilon \mathbb{E}[T']
        \big) \leq 2 \exp\Big( \frac{-\epsilon^2 c_g^2 \lfloor n / m \rfloor }{ 2 \|g \|_{\infty}^2} \Big).
    \end{equation}
    Consequently, if we define 
    \begin{equation}
        T_{l, k} = \sum_{x_1, x_2, \ldots, x_m} g(\lambda_{x_1}, \ldots, \lambda_{x_m}, \lambda_l, \lambda_k), \quad
        T'_{l, k} = \sum_{x_1 \neq x_2 \neq \cdots \neq x_m} g(\lambda_{x_1}, \ldots, \lambda_{x_m}, \lambda_l, \lambda_k)
    \end{equation}
    where $c_g \leq g(\cdot) \leq \|g \|_{\infty} < \infty$ as above, then we have that 
    \begin{equation}
        \max_{l, k} \Big| \frac{T_{l,k}}{\mathbb{E}[T'_{l,k} \,|\, \lambda_l, \lambda_k]}
        - 1 \Big| = O_p\Big( \Big(  \frac{ \log n }{ n } \Big)^{1/2} \Big)
    \end{equation}
    where the implied constant depends only on $m$ and $c_g$. 
\end{lemma}

\begin{proof}
    The first part is an immediate consequence of Hoeffding's inequality for U-statistics \cite{pitcan2019note}, which states that for $U = ((n-m)! / n! )\cdot T$ that
    \begin{equation}
        \mathbb{P}\Big(  | U - \mathbb{E}[U] | \geq t \Big) 
        \leq 2 \exp\Big( \frac{ - t^2 \lfloor n/m \rfloor }{2 \| g \|_{\infty}^2} \Big),
    \end{equation}
    by substituting in $t \mapsto t \mathbb{E}[U]$ and making use of the bound $\mathbb{E}[U] \geq c_g$.
    
    For the second part, we work conditionally on $\lambda_l, \lambda_k$ and note we can decompose $T_{l, m}$ for each ${l, m}$ into a sum of statistics of the form $T'$, one of order $\Theta_p(n^m)$ and $m \choose k$ of order $\Theta_p(n^{m-k})$ (corresponding to when some of the indices $x_i$ are equal) for $1 \leq k \leq m$. By applying the first concentration inequality to these $m! \cdot n^2$ random variables, conditional on the $(\lambda_l, \lambda_k)$, we note the RHS is independent of these quantities, and so the probability bounds hold unconditionally. Consequently, we know that asymptotically $T_{l,k}$ is asymptotic to $T'_{l,k}$, from which we can then apply the resulting concentration bound for this term.
\end{proof}

\begin{theorem}
    \label{app:thm:main_conc_thm}
    Suppose we have a statistic of the form
    \begin{equation}
        T_{n, \beta, J}(\lambda_u, \lambda_v) = \rho_n^{-\beta - |J|} \sum_{\alpha \in \mcV^{\beta - 1}} 
        g( \lambda_{\tilde{\alpha}_{0}}, \ldots, 
        \lambda_{\tilde{\alpha}_{\beta-1}}, \lambda_u, \lambda_v) \prod_{i \leq \beta} a_{ \tilde{\alpha}_{i-1}, \tilde{\alpha}_i }
        \cdot \prod_{j \in J} a_{ \tilde{\alpha}_{j-1}, \tilde{\alpha}_{j+1} }
    \end{equation}
    where $\tilde{\alpha} = (\alpha, u, v)$ is a concatenation of $\alpha$, $u$ and
    $v$ in order, $g: \mathbb{R}^{\beta + 1} \to \mathbb{R}$ is a positive function
    which satisfies $c_g \leq g \leq \| g \|_{\infty} < \infty$ for some constant $c_g$, and $J$
    is a possibly empty set of indices. Define $\lambda' = (\widetilde{\lambda}_0, \ldots, 
    \widetilde{\lambda}_{\beta - 1}, \lambda_u, \lambda_v)$ where $\widetilde{\lambda}$ is an
    independent copy of $\lambda$. Further define the statistic
    \begin{equation}
        T'_{n, \beta, J}(\lambda_u, \lambda_v) :=
        \frac{ (n - \beta)! }{n!} \cdot \mathbb{E}\Big[
            g(\lambda') \prod_{i \leq \beta} W(\lambda'_{i-1}, \lambda'_{i}) \prod_{j \in J}
            W(\lambda'_{j-1}, \lambda'_{j+1} )
        \,|\, \lambda_u, \lambda_v\Big].
    \end{equation}
    Then for any $\rho_n = n^{-\alpha}$
    for $\alpha$ sufficiently small, we have that 
    \begin{equation}
        \max_{\beta, J, u, v} \Big|  \frac{ T_{n, \beta, J}(\lambda_u, \lambda_v) }{ 
        T'_{n, \beta, J}(\lambda_u, \lambda_v)}   - 1   \Big| = O_p\Big(   \Big( \frac{
            (\log n)^k } { n \cdot (n \rho_n) } \Big)^{1/2} \Big).
    \end{equation}
\end{theorem}

\begin{proof}
    For this, we apply the above results. We begin by working conditionally on
    all of the $\lambda$, whose collection we denote $\bold{\lambda}$, and note that by Theorem~\ref{app:thm:path_concentration} by taking $\lambda = (\log n)^k$ for some $k > 1$ and a union bound, we have that 
    \begin{equation}
        T_{n, \beta, J}(\lambda_u, \lambda_v) = \mathbb{E}[
        T_{n, \beta, J}(\lambda_u, \lambda_v) \,|\, \bold{\lambda} ] \cdot 
        (1 + E_n^{(1)}) \text{ where } E_n^{(1)} = O\Big( \Big(\frac{
            (\log n)^k } { n \cdot (n \rho_n) } \Big)^{1/2} \Big)
    \end{equation}
    uniformly over all $ O(m^2 m! \cdot n^2)$ random variables with probability 
    $1 - \exp( O( (\log n)^k ))$. As we have that
    \begin{equation}
        \mathbb{E}[
        T_{n, \beta, J}(\lambda_u, \lambda_v) \,|\, \bold{\lambda} ]
        = \sum_{\alpha \in \mcV^{\beta - 1}} 
         g( \lambda_{\tilde{\alpha}_{0}}, \ldots, 
        \lambda_{\tilde{\alpha}_{\beta-1}}, \lambda_u, \lambda_v) \prod_{i \leq \beta} W( \lambda_{\tilde{\alpha}_{i-1}}, 
        \lambda_{\tilde{\alpha}_i} )
        \cdot \prod_{j \in J} W( \lambda_{\tilde{\alpha}_{j-1}}, \lambda_{ \tilde{\alpha}_{j+1} } )
    \end{equation}
    where the function is bounded below by $c_g \cdot c_p^{\beta + |J|}$ and is bounded above by $\|g \|_{\infty}$, we can make use of Lemma~\ref{app:thm:ustat_concentration} to show that 
    \begin{equation}
        \max_{\beta, J, u, v} \Big| 
            \frac{ \mathbb{E}[
        T_{n, \beta, J}(\lambda_u, \lambda_v) \,|\, \bold{\lambda} ]   }{    
        T'_{n, \beta, J}(\lambda_u, \lambda_v)} - 1
        \Big| = O_p\Big( \Big(  \frac{ \log n }{ n } \Big)^{1/2} \Big)
    \end{equation}
    from which the claimed result follows.
\end{proof}

\begin{remark}
    One natural question to ask about the necessity of the range of values of $\rho_n$ specified above. Generally speaking, one can show for Erdos-Renyi graphs $G(n, p)$ that the number of subgraphs $Y_H$ of $H$ in $\mcG_n$ satisfy a zero-one law, where 
    \begin{equation}
        \mathbb{P}(Y_H = 0) = \begin{cases}
        1 - o(1) \text{ if } p \ll n^{-c(H)},\\
        o(1) \text{ if } p \gg n^{-c(H)}
        \end{cases}
    \end{equation}
    for some constant $c(H)$ which relates to the geometry of the graph $G$ \citep{bollobas1981threshold}. In the latter regime, one can then show that $Y_H \sim E[Y_H]$ asymptotically again, and in the former this shows that the term is asymptotically negligible. As the purpose of this result is to derive an asymptotic expansion for the sum of various statistics of the form of $T$ to the highest order, provided $\rho_n$ is of an order which avoids any of the "phase transition" stages of the form above we could eventually generalize our results further. As this involves even more additional book-keeping, we do not do so here.
\end{remark}



\begin{lemma}
    \label{app:thm:bern_1}
    Let $I$ be a finite index set of size $|I| = m$. Suppose that 
    there exist constants $\tau > 0$, a 
    bounded non-negative sequence $(p_i)_{i \in I}$
    such that $p_i \leq \tau^{-1}$ for all $i$, 
    and a real sequence $(t_i)_{i \in I}$. Define 
    the random variable
    \begin{equation}
        X = \frac{1}{m} \sum_{i \in I} \big( \tau^{-1} a_i - p_i \big) t_i 
        \qquad \text{ where } \qquad a_i \id \mathrm{Bernoulli}( \tau p_i ) \text{ for }
        i \in I.
    \end{equation}
    Then for all $u > 0$, we have that
    \begin{equation}
        \mathbb{P}\big( |X| \geq u \big) \leq 2 \exp \Big( 
            - \min\Big\{ \frac{ u^2}{ 4 \tau^{-1} m^{-2} \| t \|_2^2  } 
        , \frac{u}{ 2 \tau^{-1} m^{-1} \| t \|_{\infty} } \Big\}  \Big).
    \end{equation}
\end{lemma}

\begin{proof}
    This follows by an application of Bernstein's inequality, by noting that $X$
    is a sum of independent mean zero random variables 
    $X_i = m^{-1} (\tau^{-1} a_i - p_i) t_i$
    which satisfy 
    \begin{equation*}
        |X_i| \leq \tau^{-1} m^{-1} |t_i| 
        \leq  \tau^{-1} m^{-1} \| t \|_{\infty}
    \text{ for all } i, 
    \qquad \mathbb{E}[X_i^2] \leq m^{-2} \tau^{-1} t_i^2. \qedhere
    \end{equation*}
\end{proof}

\begin{lemma}
    \label{app:thm:bern_2}
    Define the random variable
    \begin{equation}
        Y = \frac{1}{n(n-1)} \sum_{i \neq j} 
        \big( \rho_n^{-1} a_{ij} - W(\lambda_i, \lambda_j)   \big) T_{ij}
    \end{equation}
    for some constants $(T_{ij})$. Write $ \| T \|_2^2 = \sum_{i \neq j} 
    T_{ij}^2$ and $\| T \|_{\infty} = \max_{i \neq j} | T_{ij} |$. Then we have that
    \begin{equation}
        \mathbb{P}\Big( |Y| \geq u \Big) \leq 
        2 \exp\Big( - \min\Big\{    
            \frac{u^2}{ 128 \rho_n^{-1} n^{-4} \| T \|_2^2}, \frac{u}{16 \rho_n^{-1} n^{-2} \| T \|_{\infty}}
        \Big\}  \Big)
    \end{equation}
    In particular, when $T_{ij} = 1$ for all $i \neq j$, we have that
    $Y = O_p( (n^2 \rho_n)^{-1/2} )$. 
\end{lemma}

\begin{proof}
    Note that under the assumptions on the model (where we have that $a_{ij} = a_{ji}$
    and $W(\lambda_i, \lambda_j) = W(\lambda_j, \lambda_i)$ for all $i \neq j$), we
    can write
    \begin{equation}
        Y = \frac{2}{n(n-1) / 2} \sum_{i < j} 
        \big( \rho_n^{-1} a_{ij} - W(\lambda_i, \lambda_j)   \big) (T_{ij} + T_{ji} ).
    \end{equation}
    Note that
    \begin{align}
        \sum_{i < j} (T_{ij} + T_{ji})^2 & \leq 2 \sum_{i < j} \big( T_{ij}^2 + T_{ji}^2 
        \big) \leq 2 \| T \|_2^2, \\
        \max_{i < j} | T_{ij} + T_{ji} | & \leq \max_{i < j} |T_{ij}| + \max_{i < j} |T_{ji} |
        \leq 2 \| T \|_{\infty},
    \end{align}
    where we have used the inequality $(a + b)^2 \leq 2(a^2 + b^2)$ which holds for
    all $a, b \in \mathbb{R}$. 
    Consequently, as a result of Lemma~\ref{app:thm:bern_1}, we have
    conditional on $\lambda$ that
    \begin{equation}
        \mathbb{P}\Big( |Y| \geq u \,|\, \lambda \Big) \leq 
        2 \exp\Big( - \min\Big\{    
            \frac{u^2}{ 128 \rho_n^{-1} n^{-4} \| T \|_2^2}, \frac{u}{16 \rho_n^{-1} n^{-2} \| T \|_{\infty}}
        \Big\}  \Big)
    \end{equation}
    As the right hand side has no dependence on $\lambda$, taking expectations gives
    the first part of the lemma statement. For the second part, note that if
    $T_{ij} = 1$ for all $i \neq j$, then we have that $\| T\|_2^2 \leq n^2$ and
    $\| T \|_{\infty} = 1$, and consequently 
    \begin{equation}
        \mathbb{P}\big( |Y| \geq u) \leq 2 \exp\Big( - \min\Big\{    
            \frac{u^2}{ 128 \rho_n^{-1} n^{-2}}, \frac{u}{16 \rho_n^{-1} n^{-2} }
        \Big\}  \Big)
    \end{equation}
    In particular, this implies that $Y = O_p( (n^2 \rho_n)^{-1/2})$.
\end{proof}

\subsection{Miscellaneous results}
\begin{lemma}
    \label{thm:other_results:nice_mat_lemma}
    Suppose that $A \in \mathbb{R}^{m \times m}$ is a matrix whose diagonal
    entries are $\alpha$, and off-diagonal entries are $\beta$, so
    $A_{ij} = \alpha \delta_{ij} + \beta (1 - \delta_{ij})$, 
    where $\delta_{ij}$
    is the Kronecker delta. Then $A$ has an 
    eigenvalue $\alpha + (m-1) \beta$ of multiplicity one 
    with eigenvector $1_m$, and an 
    eigenvalue $\alpha - \beta$ of multiplicity $m - 1$, whose 
    eigenvectors form an orthonormal
    basis of the subspace $\{ v \,:\, \langle v, 1_m \rangle = 0 \}$.
    For the subspace $\{ v \,:\, \langle v, 1_m \rangle = 0 \}$, we
    can take
    the eigenvectors to be
    \begin{equation*}
        v_i = \frac{1}{\sqrt{2}} (e_{m, 1} - e_{m, i+1}) \text{ for } i \in [m-1]
    \end{equation*}
    where $e_{m,i}$ are the unit column vectors in $\mathbb{R}^m$,
    The singular values of $A$ are $|\alpha - \beta|$
    and $|\alpha + (\kappa - 1)\beta|$. Consequently, we can write $A = UV^T$
    for matrices $U, V \in \mathbb{R}^{m \times m}$ with
    $UU^T = VV^T$, where the rows of $U$ satisfy
    \begin{align}
        U_{1\cdot} & = \frac{|\alpha + \beta(m-1) |^{1/2}}{ \sqrt{m}} e_{m, 1}
        + \frac{|\alpha - \beta|^{1/2}}{ \sqrt{2}} e_{m, 2} \\
        U_{i\cdot} & = \frac{|\alpha + \beta(m-1) |^{1/2}}{ \sqrt{m}} e_{m, 1}
        - \frac{|\alpha - \beta|^{1/2}}{ \sqrt{2}} e_{m, i} \text{ for } i \in 
        \{2, \ldots, m \}.
    \end{align}
    Consequently, we then have that $\| U_{i\cdot} \|_2 \leq
    \big( 2|\alpha + \beta(m-1)|/m + |\alpha - \beta|/2\big)^{1/2}$ 
    for all $i$, 
    and $\min_{i \neq j} \| U_{i\cdot} - U_{j\cdot} \|_2 = (|\alpha - \beta|)^{1/2}$.
    
    Further suppose that $\beta = -\alpha/(m-1)$. Then provided $\alpha > 0$, 
    $A$ is positive semi-definite, is of rank $m-1$, with a singular 
    non-zero eigenvalue $\alpha m/(m-1)$ of multiplicity $m-1$. Consequently one can write 
    $A = UU^T$ where $U \in \mathbb{R}^{m \times (m-1)}$
    and whose columns equal the $\sqrt{\alpha m/(m-1)} v_i$. In particular, the rows
    of $U$ equal
    \begin{equation*}
        U_{1\cdot} = \big( \frac{\alpha m}{2(m-1)} \big)^{1/2} e_{m-1, 1}^T, 
        \quad U_{i\cdot} = - \big( \frac{\alpha m}{2(m-1)} \big)^{1/2} e_{m-1,i-1}^T \text{ for } i \in [2, m].
    \end{equation*}
    Consequently, one has that $\| U_{i\cdot} \|_2 
    = \sqrt{\alpha m/(m-1)}$ for all $i$,
    and moreover we have the separability condition 
    $\min_{1 \leq i < j \leq m} \| U_{i\cdot} - U_{j\cdot} \|_2 = (\alpha m/(m-1) )^{1/2}$.
\end{lemma}

\begin{proof}
    It is straightforward to verify that $A$ has an eigenvalue
    of $\alpha + (n-1) \beta$ with the claimed eigenvector. For the second
    part, we note that the characteristic polynomial of $A$ is
    \begin{equation*}
        \mathrm{det}(A - tI) = (\alpha - \beta - t)^{n-1} \cdot (\alpha + (n-1)\beta - t) 
    \end{equation*}
    and so $A$ has $m-1$ eigenvalues equal to $\alpha - \beta$; as $A$
    is symmetric, we know that we can always take eigenvectors to be
    orthogonal to each other, and consequently the eigenspace associated
    with such an eigenvalue must be a subspace of 
    $\{ v \,:\, \langle v, 1_m \rangle = 0 \}$. As both of these subspaces 
    are of dimension $m-1$, it
    consequently follows that they are equal. We then
    highlight that if $A$ is a symmetric matrix with eigendecomposition
    $A = Q \Lambda Q^T$ for an orthogonal matrix $Q$, then
    the SVD is given by $Q | \Lambda| \mathrm{sgn}(\Lambda) Q^T$,
    and we can write $A = UV^T$ with $U = Q |\Lambda|^{1/2}$
    and $V = Q \mathrm{sgn}(\Lambda) |\Lambda|^{1/2}$ 
    such that $UU^T = VV^T$. 
    This allows us to derive the remaining statements about the matrix $A$ 
    which hold in generality. The remaining parts discussing what occurs
    when $\beta = -\alpha/(m-1)$ follow by routine calculation. 
\end{proof}

\begin{lemma}
    \label{thm:other_results:poly_solve}
    Let $\sigma(x) = (1+\exp(-x))^{-1}$ be the sigmoid function. Then there exists a unique $y \in \mathbb{R}$ which solves the equation
    \begin{equation}
        \alpha \sigma(y) = \beta + \gamma \sigma(-y/s)
    \end{equation}
    for $\alpha, \gamma, s > 0$ and $\beta \in \mathbb{R}$ if and only if $\beta < \alpha$ and $\beta + \gamma > 0$. Moreover, $y > 0$ 
    if and only if $\beta + \gamma/2 > \alpha/2$.
\end{lemma}

\begin{proof}
    Note that $\alpha \sigma(x)$ is a function whose range is $(0, \alpha)$ 
    on $x \in (-\infty, \infty)$, and is strictly monotone increasing on 
    the domain. Similarly, $\beta + \gamma \sigma(-y/s)$ is strictly 
    monotone decreasing with range $(\beta, \beta + \gamma)$, and so 
    simple geometric considerations of the graphs of the two functions
    gives the existence result. For the second part, note that the 
    ranges of the functions on the LHS and the RHS on the 
    range $y > 0$ are $[\alpha/2, \alpha)$ and
    $(\beta, \beta + \gamma/2]$ respectively, and so the same considerations
    as above give the second claim.
\end{proof}

\begin{lemma}
    \label{thm:eg:other:curvature}
    Let $\sigma(x) = (e^x)/(1+e^x)$ be the sigmoid function. Then for 
    any $x, y \in \mathbb{R}$, we have that
    \begin{equation}
        -\log(1-\sigma(x)) \geq - \log(1- \sigma(y)) + \sigma(y)(x - y) 
        + E(x - y)
    \end{equation}
    where
    \begin{equation}
        E(z) = \begin{cases}
            \frac{1}{2} e^{-A} z^2 & \text{ if } |x|, |y| \leq A, \\
            \frac{1}{4} e^{-A } \min\{ z^2, 2|z| \} & \text{ if either }
            |x| \leq A \text{ or } |y| \leq A.
        \end{cases}            
    \end{equation}
\end{lemma}

\begin{proof}
    Note that by the integral version of Taylor's theorem,
    for a twice differentiable function $f$ one has for all $x, y \in \mathbb{R}$ that 
    \begin{align}
        f(x) = f(y) + f'(y)(x-y) + \int_0^1 (1-t) f''(tx + (1-t)y) (x - y)^2 \, dt.
    \end{align}
    Applying this to $f(x) = - \log \sigma(x)$ gives 
    \begin{equation}
        -\log\sigma(x) = -\log\sigma(y) + (-1 + \sigma(y))(x-y) + \int_0^1
        (1 - t) (x- y)^2 \sigma'(tx + (1-t)y) \, dt
    \end{equation}
    where $\sigma'(x) = e^{x}/(1+e^x)^2$. Applying 
    this to $f(x) = \log(1 - \sigma(x))$ gives
    \begin{equation}
        -\log(1-\sigma(x)) = - \log(1- \sigma(y)) + \sigma(y)(x - y) 
        + \int_0^1
        (1 - t) (x- y)^2 \sigma'(tx + (1-t)y) \, dt
    \end{equation}
    As the integral terms are the same, we concentrate on lower bounding
    this quantity. To do so, we make use of the lower bound 
    $\sigma'(x) \geq e^{-|x|}/4$ (Lemma 68 of \citet{davison_asymptotics_2023})
    which holds for all $x \in \mathbb{R}$. We then note that if
    $|x|, |y| \leq A$, then we have that
    \begin{align}
        -\log(1-\sigma(x)) & = - \log(1- \sigma(y)) + \sigma(y)(x - y) 
        + \int_0^1
        (1 - t) (x- y)^2 \sigma'(tx + (1-t)y) \, dt \\
        & \geq - \log(1- \sigma(y)) + \sigma(y)(x - y) + \frac{e^{-|A|}}{2} (x -y)^2.
    \end{align}
    Alternatively, if we only make use of the fact that $|x| \leq A$ (without loss of
    generality - the argument is essentially equivalent if we only
    assume that $|y| \leq A$), then we
    have that
    \begin{align}
        \int_0^1 (1 - t) \sigma'(tx + (1-t) y) (x - y)^2 \, dt 
        & \geq \int_0^1 (1-t) e^{-|tx + (1- t)y | } (x - y)^2 \, dt \\
        & \geq \int_0^1  (1- t) e^{-|x|} e^{-(1-t)|x-y| } ( x - y)^2 \, dt \\
        & = e^{-|x|} \big\{  |x - y| + e^{-|x-y| } - 1 \big\} \\
        & \geq \frac{1}{4} e^{-A } \min\{ (x - y)^2, 2|x- y| \},
    \end{align}
    and consequently we get that
    \begin{equation}
        -\log(1-\sigma(x)) \geq - \log(1- \sigma(y)) + \sigma(y)(x - y) 
        +  \frac{1}{4} e^{-A } \min\{ |x - y|^2, 2|x- y| \}
    \end{equation}
    as claimed.
\end{proof}

\begin{lemma}
    \label{thm:eg:other:curvature_lower_bound}
    Suppose that we have a function
    \begin{equation}
        f(X) = \frac{1}{m^2} \sum_{i, j = 1}^m 
        \min\{ X_{ij}^2, 2|X_{ij}| \}.
    \end{equation}
    Then if $f(X) \leq r$, we have that $m^{-2} \sum_{i, j = 1}^m |X_{ij}|
    \leq r + r^{1/2}$.
\end{lemma}

\begin{proof}
    To proceed, note that if we have that
    \begin{equation}
        \mathbb{E}[ \min\{X^2, 2X \} ] \leq r
    \end{equation}
    for a non-negative random variable $X$, then by Jensen's inequality
    we get that 
    \begin{equation}
        \big( \mathbb{E}[X 1[X < 2 ]] \big)^2 + \mathbb{E}[X 1[X \geq 2]]
        \leq \mathbb{E}[ \min\{X^2, 2X \} ] \leq r
    \end{equation}
    and consequently $\mathbb{E}[X] \leq r + r^{1/2}$ by decomposing $\mathbb{E}[X]$
    into the parts where $X \geq 2$ and $X < 2$. Applying this result
    to the empirical measure on the $|X_{ij}|$ across indices $i, j \in [m]$
    gives the desired result.
\end{proof}

\section{Minimizers for degree corrected SBMs when $\alpha \neq 1$}
\label{app:sec:ugly_stuff}

In this section, we give an informal discussion of how to study
the minimizers of $\mcR_n(M)$ for degree corrected SBMs
when the unigram parameter $\alpha \neq 1$. We begin by highlighting
that $\mcR_n(M)$ does not concentrate around its expectation 
when averaging over only the degree heterogenity parameters $\theta_i$, 
which rules out using a similar proof approach as to what was carried out
earlier in Appendix~1. 

Recall that we were able to derive that the global minima of
$\mcR_n(M)$ was the matrix 
\begin{align}
    M_{ij}^* = \log\Big( \frac{ 2 \mcE_W(\alpha) }{ 
        (1 + k^{-1}) \mathbb{E}[\theta] \mathbb{E}[\theta]^{\alpha} } \cdot 
        \frac{  P_{c(i), c(j)} }{ 
        \widetilde{P}_{c(i)} \widetilde{P}_{c(j)}   
        \cdot \big(   
            \theta_i^{\alpha - 1} \widetilde{P}_{c(i)}^{\alpha - 1}
        + \theta_j^{\alpha - 1} \widetilde{P}_{c(j)}^{\alpha - 1}
        \big)}   \Big).
\end{align}
When $\alpha = 1$ or the $\theta_i$ are constant, this allows us 
to write $M^* = \Pi M \Pi^T$ where $\Pi$ is the matrix 
of community assignments for the network and $M$ is some matrix, which
allows us to simplify the problem. If we supposed that the $\theta$
actually had some dependence on the $c(i)$ and were discrete - in that
$\theta_i | c(i) = l \sim Q_l$ for some discrete distributions $Q_l$
for $l \in [\kappa]$, then we could in fact employ the same type of argument
as done throughout the paper. The major change is that then the
embedding vectors would each concentrate around a vector decided
by both a) their community assignment, and b) the particular degree
correction parameter they were assigned. This would then potentially effect
our ability to perform community detection depending on the underlying
geometry of these vectors. One possible idea would be to explore
$\mcR_n(M)$ partially averaged over the $\theta_i$ - we divide the $\theta_i$
into $B$ bins where $B = n^{\beta}$ for some $\beta \in (0, 1)$, and average over only over the refinement of the $\theta_i$
as belonging to the different bins. This would be similar to the
argument employed in \citet{davison_asymptotics_2023}.

An alternative perspective to give some type of guarantee on the 
concentration of the embedding vectors is to study the rank
of the matrix $M^*$. If we are able to prove that is of finite
rank $r$ even as $n$ grows large, then we are able
to give a convergence result for the embeddings as soon as the
embedding dimension $d$ is greater than or equal to $r$. To study this, it suffices
to look at the matrix
\begin{equation}
    (M_E^*)_{ij} = \log  \big(   
            \theta_i^{\alpha - 1} \widetilde{P}_{c(i)}^{\alpha - 1}
        + \theta_j^{\alpha - 1} \widetilde{P}_{c(j)}^{\alpha - 1}
        \big)
\end{equation}
and argue that this is low rank (due to the logarithm, we can write
$M^*$ as the difference between this matrix and a matrix of rank $\kappa$, 
which is therefore also low rank). The entry-wise logarithm is a complicating
factor here, as otherwise it is straightforward to argue that the entry-wise
exponential of this matrix is of rank 2. One can reduce studying the rank
of the matrix $M_E^*$ to studying the rank of the kernel
\begin{equation}
    K_M\big( (x, c_x), (y, c_y) \big) = \log\big( x^{\alpha - 1} \tilde{P}^{\alpha - 1}_{c_x} + y^{\alpha - 1} \tilde{P}^{\alpha - 1}_{c_y} \big) 
\end{equation}
of an operator $L^2(P) \to L^2(P)$, where $P$ is the product measure 
induced
by $\theta$ and the community assignment mechanism $c$. As $K_M$ is of
finite rank $r$ if and only if it can be written as
\begin{equation}
    K_M\big( (x, c_x), (y, c_y) \big) = 
    \sum_{i=1}^r \phi_i(x, c_x) \psi_i(y, c_y)
\end{equation}
for some functions $\phi_i, \psi_i$, it follows that the matrix
$(M_E^*)_{ij}$ will be of finite rank $r$ also. Indeed, this representation
forces that $M_E^* = \Phi \Psi^T$ for some matrices $\Phi, \Psi \in \mathbb{R}^{n \times r}$, meaning that $M_E^*$ is of rank $\leq r$; Corollary 5.5 of \citet{koltchinskii_random_2000} then guarantees convergence
of the eigenvalues of the matrix $M_E^*$ to the operator $K_M$ so that
$M_E^*$ is actually of full rank.

\clearpage
\subsubsection*{Supplement References}

\printbibliography[heading=none]

@article{gao-dcsbm,
author = {Chao Gao and Zongming Ma and Anderson Y. Zhang and Harrison H. Zhou},
title = {{Community detection in degree-corrected block models}},
volume = {46},
journal = {The Annals of Statistics},
number = {5},
publisher = {Institute of Mathematical Statistics},
pages = {2153 -- 2185},
year = {2018},
doi = {10.1214/17-AOS1615},
URL = {https://doi.org/10.1214/17-AOS1615}
}

@article{jian2014giantcomponent,
    author = {Ding, Jian and Lubetzky, Eyal and Peres, Yuval},
    title = {Anatomy of the giant component: The strictly supercritical regime},
    year = {2014},
    issue_date = {January, 2014},
    publisher = {Academic Press Ltd.},
    address = {GBR},
    volume = {35},
    issn = {0195-6698},
    url = {https://doi.org/10.1016/j.ejc.2013.06.004},
    doi = {10.1016/j.ejc.2013.06.004},
    journal = {Eur. J. Comb.},
    month = jan,
    pages = {155–168},
    numpages = {14}
}

@article{zhang2023fundamental,
  title={Fundamental limits of spectral clustering in stochastic block models},
  author={Zhang, Anderson Ye},
  journal={arXiv preprint arXiv:2301.09289},
  year={2023}
}

@article{danon2005comparing,
  title={Comparing community structure identification},
  author={Danon, Leon and Diaz-Guilera, Albert and Duch, Jordi and Arenas, Alex},
  journal={Journal of statistical mechanics: Theory and experiment},
  volume={2005},
  number={09},
  pages={P09008},
  year={2005},
  publisher={IOP Publishing}
}

@misc{snapnets,
  author       = {Jure Leskovec and Andrej Krevl},
  title        = {{SNAP Datasets}: {Stanford} Large Network Dataset Collection},
  howpublished = {\url{http://snap.stanford.edu/data}},
  month        = jun,
  year         = 2014
}

@inproceedings{polblogs,
author = {Adamic, Lada A. and Glance, Natalie},
title = {The political blogosphere and the 2004 U.S. election: divided they blog},
year = {2005},
isbn = {1595932151},
publisher = {Association for Computing Machinery},
address = {New York, NY, USA},
url = {https://doi.org/10.1145/1134271.1134277},
doi = {10.1145/1134271.1134277},
booktitle = {Proceedings of the 3rd International Workshop on Link Discovery},
pages = {36–43},
numpages = {8},
location = {Chicago, Illinois},
series = {LinkKDD '05}
}

@article{robbins_stochastic_1951,
	title = {A {Stochastic} {Approximation} {Method}},
	volume = {22},
	issn = {0003-4851, 2168-8990},
	url = {https://projecteuclid.org/journals/annals-of-mathematical-statistics/volume-22/issue-3/A-Stochastic-Approximation-Method/10.1214/aoms/1177729586.full},
	doi = {10.1214/aoms/1177729586},
	number = {3},
	urldate = {2022-03-15},
	journal = {The Annals of Mathematical Statistics},
	author = {Robbins, Herbert and Monro, Sutton},
	month = sep,
	year = {1951},
	note = {Number: 3
Publisher: Institute of Mathematical Statistics},
	pages = {400--407},
}

@article{holland_stochastic_1983,
  title={Stochastic blockmodels: First steps},
  author={Holland, Paul W and Laskey, Kathryn Blackmond and Leinhardt, Samuel},
  journal={Social networks},
  volume={5},
  number={2},
  pages={109--137},
  year={1983},
  publisher={Elsevier}
}

@article{abbe_community_2017,
  title={Community detection and stochastic block models: recent developments},
  author={Abbe, Emmanuel},
  journal={The Journal of Machine Learning Research},
  volume={18},
  number={1},
  pages={6446--6531},
  year={2017},
  publisher={JMLR. org}
}

@article{deng_strong_2021,
  title={Strong consistency, graph laplacians, and the stochastic block model},
  author={Deng, Shaofeng and Ling, Shuyang and Strohmer, Thomas},
  journal={The Journal of Machine Learning Research},
  volume={22},
  number={1},
  pages={5210--5253},
  year={2021},
  publisher={JMLRORG}
}

@article{ma_determining_2021,
  title={Determining the number of communities in degree-corrected stochastic block models},
  author={Ma, Shujie and Su, Liangjun and Zhang, Yichong},
  journal={The Journal of Machine Learning Research},
  volume={22},
  number={1},
  pages={3217--3279},
  year={2021},
  publisher={JMLRORG}
}

@article{davison_asymptotics_2023,
  title={Asymptotics of Network Embeddings Learned via Subsampling},
  author={Davison, Andrew and Austern, Morgane},
  journal={Journal of Machine Learning Research},
  volume={24},
  number={138},
  pages={1--120},
  year={2023}
}

@article{davison_asymptotics_2022,
  title={Asymptotics of $\ell\_2 $ Regularized Network Embeddings},
  author={Davison, Andrew},
  journal={Advances in Neural Information Processing Systems},
  volume={35},
  pages={24960--24974},
  year={2022}
}

@article{cui_survey_2018,
  title={A survey on network embedding},
  author={Cui, Peng and Wang, Xiao and Pei, Jian and Zhu, Wenwu},
  journal={IEEE transactions on knowledge and data engineering},
  volume={31},
  number={5},
  pages={833--852},
  year={2018},
  publisher={IEEE}
}

@ARTICLE{zhang_consistency_2024,
  author={Zhang, Yichi and Tang, Minh},
  journal={IEEE Transactions on Pattern Analysis and Machine Intelligence}, 
  title={A Theoretical Analysis of DeepWalk and Node2vec for Exact Recovery of Community Structures in Stochastic Blockmodels}, 
  year={2024},
  volume={46},
  number={2},
  pages={1065-1078},
  doi={10.1109/TPAMI.2023.3327631}}

@article{vu_rank_random_graph_2008,
author = {Costello, Kevin P. and Vu, Van H.},
title = {The rank of random graphs},
journal = {Random Structures \& Algorithms},
volume = {33},
number = {3},
pages = {269-285},
doi = {https://doi.org/10.1002/rsa.20219},
url = {https://onlinelibrary.wiley.com/doi/abs/10.1002/rsa.20219},
eprint = {https://onlinelibrary.wiley.com/doi/pdf/10.1002/rsa.20219},
year = {2008}
}

@inproceedings{bollobas1981threshold,
  title={Threshold functions for small subgraphs},
  author={Bollob{\'a}s, B{\'e}la},
  booktitle={Mathematical Proceedings of the Cambridge Philosophical Society},
  volume={90},
  number={2},
  pages={197--206},
  year={1981},
  organization={Cambridge University Press}
}

@article{airoldi_mixed_2008,
  title={Mixed membership stochastic blockmodels},
  author={Airoldi, Edoardo M and Blei, David M. and Fienberg, Stephen E. and Xing, Eric P.},
  journal={Advances in neural information processing systems},
  volume={21},
  year={2008}
}

@misc{pitcan2019note,
      title={A Note on Concentration Inequalities for U-Statistics}, 
      author={Yannik Pitcan},
      year={2019},
      eprint={1712.06160},
      archivePrefix={arXiv},
      primaryClass={math.ST}
}

@article{vu_conc_2002,
author = {Vu, V. H.},
title = {Concentration of non-Lipschitz functions and applications},
journal = {Random Structures \& Algorithms},
volume = {20},
number = {3},
pages = {262-316},
doi = {https://doi.org/10.1002/rsa.10032},
url = {https://onlinelibrary.wiley.com/doi/abs/10.1002/rsa.10032},
eprint = {https://onlinelibrary.wiley.com/doi/pdf/10.1002/rsa.10032},
year = {2002}
}

@article{lei_consistency_2015,
author = {Jing Lei and Alessandro Rinaldo},
title = {{Consistency of spectral clustering in stochastic block models}},
volume = {43},
journal = {The Annals of Statistics},
number = {1},
publisher = {Institute of Mathematical Statistics},
pages = {215 -- 237},
year = {2015},
doi = {10.1214/14-AOS1274},
URL = {https://doi.org/10.1214/14-AOS1274}
}

@article{karrer_stochastic_2011,
  title={Stochastic blockmodels and community structure in networks},
  author={Karrer, Brian and Newman, Mark EJ},
  journal={Physical review E},
  volume={83},
  number={1},
  pages={016107},
  year={2011},
  publisher={APS}
}

@article{freeman_development_2004,
  title={The development of social network analysis},
  author={Freeman, Linton},
  journal={A Study in the Sociology of Science},
  volume={1},
  number={687},
  pages={159--167},
  year={2004}
}

@article{koltchinskii_random_2000,
	title = {Random {Matrix} {Approximation} of {Spectra} of {Integral} {Operators}},
	volume = {6},
	issn = {1350-7265},
	url = {http://www.jstor.org/stable/3318636},
	doi = {10.2307/3318636},
	number = {1},
	urldate = {2021-07-21},
	journal = {Bernoulli},
	author = {Koltchinskii, Vladimir and Giné, Evarist},
	year = {2000},
	note = {Number: 1
Publisher: International Statistical Institute (ISI) and Bernoulli Society for Mathematical Statistics and Probability},
	pages = {113--167},
}

@article{luo_modular_2007,
  title={Modular organization of protein interaction networks},
  author={Luo, Feng and Yang, Yunfeng and Chen, Chin-Fu and Chang, Roger and Zhou, Jizhong and Scheuermann, Richard H},
  journal={Bioinformatics},
  volume={23},
  number={2},
  pages={207--214},
  year={2007},
  publisher={Oxford University Press}
}

@article{hoff_latent_2002,
  title={Latent space approaches to social network analysis},
  author={Hoff, Peter D and Raftery, Adrian E and Handcock, Mark S},
  journal={Journal of the american Statistical association},
  volume={97},
  number={460},
  pages={1090--1098},
  year={2002},
  publisher={Taylor \& Francis}
}

@inproceedings{grover_node2vec_2016,
  title={node2vec: Scalable feature learning for networks},
  author={Grover, Aditya and Leskovec, Jure},
  booktitle={Proceedings of the 22nd ACM SIGKDD international conference on Knowledge discovery and data mining},
  pages={855--864},
  year={2016}
}

@article{zhang_graph_2019,
  title={Graph convolutional networks: a comprehensive review},
  author={Zhang, Si and Tong, Hanghang and Xu, Jiejun and Maciejewski, Ross},
  journal={Computational Social Networks},
  volume={6},
  number={1},
  pages={1--23},
  year={2019},
  publisher={SpringerOpen}
}

@article{velivckovic_graph_2017,
  title={Graph attention networks},
  author={Veli{\v{c}}kovi{\'c}, Petar and Cucurull, Guillem and Casanova, Arantxa and Romero, Adriana and Lio, Pietro and Bengio, Yoshua},
  journal={arXiv preprint arXiv:1710.10903},
  year={2017}
}

@article{ward_next_2021,
  title={Next waves in veridical network embedding},
  author={Ward, Owen G and Huang, Zhen and Davison, Andrew and Zheng, Tian},
  journal={Statistical Analysis and Data Mining: The ASA Data Science Journal},
  volume={14},
  number={1},
  pages={5--17},
  year={2021},
  publisher={Wiley Online Library}
}

@article{ng_spectral_2001,
  title={On spectral clustering: Analysis and an algorithm},
  author={Ng, Andrew and Jordan, Michael and Weiss, Yair},
  journal={Advances in neural information processing systems},
  volume={14},
  year={2001}
}

@inproceedings{perozzi_deepwalk_2014,
  title={Deepwalk: Online learning of social representations},
  author={Perozzi, Bryan and Al-Rfou, Rami and Skiena, Steven},
  booktitle={Proceedings of the 20th ACM SIGKDD international conference on Knowledge discovery and data mining},
  pages={701--710},
  year={2014}
}

@article{hamilton_inductive_2017,
  title={Inductive representation learning on large graphs},
  author={Hamilton, Will and Ying, Zhitao and Leskovec, Jure},
  journal={Advances in neural information processing systems},
  volume={30},
  year={2017}
}

@article{rubin-delanchy_statistical_2017,
  title={A statistical interpretation of spectral embedding: the generalised random dot product graph. arXiv e-prints},
  author={Rubin-Delanchy, P and Priebe, CE and Tang, M and Cape, J},
  journal={arXiv preprint arXiv:1709.05506},
  year={2017}
}

@article{lei_network_2021,
	author = {Jing Lei},
	journal = {The Annals of Statistics},
	number = {2},
	pages = {745 -- 768},
	title = {{Network representation using graph root distributions}},
	volume = {49},
	year = {2021}}

@article{levin_limit_2021,
  title={Limit theorems for out-of-sample extensions of the adjacency and Laplacian spectral embeddings},
  author={Levin, Keith D and Roosta, Fred and Tang, Minh and Mahoney, Michael W and Priebe, Carey E},
  journal={The Journal of Machine Learning Research},
  volume={22},
  number={1},
  pages={8707--8765},
  year={2021},
  publisher={JMLRORG}
}

@inproceedings{qiu_network_2018,
  title={Network embedding as matrix factorization: Unifying deepwalk, line, pte, and node2vec},
  author={Qiu, Jiezhong and Dong, Yuxiao and Ma, Hao and Li, Jian and Wang, Kuansan and Tang, Jie},
  booktitle={Proceedings of the eleventh ACM international conference on web search and data mining},
  pages={459--467},
  year={2018}
}

@article{latouche_blog_2011,
author = {Pierre Latouche and Etienne Birmel{\'e} and Christophe Ambroise},
title = {{Overlapping stochastic block models with application to the French political blogosphere}},
volume = {5},
journal = {The Annals of Applied Statistics},
number = {1},
publisher = {Institute of Mathematical Statistics},
pages = {309 -- 336},
year = {2011},
doi = {10.1214/10-AOAS382},
URL = {https://doi.org/10.1214/10-AOAS382}
}

@article{legramanti_extended_2022,
  title={Extended stochastic block models with application to criminal networks},
  author={Legramanti, Sirio and Rigon, Tommaso and Durante, Daniele and Dunson, David B},
  journal={The Annals of Applied Statistics},
  volume={16},
  number={4},
  pages={2369},
  year={2022}
}

@inproceedings{airoldi_mixed_2006,
  title={Mixed membership stochastic block models for relational data with application to protein-protein interactions},
  author={Airoldi, Edoardo M and Blei, David M. and Fienberg, Stephen E. and Xing, Eric P. and Jaakkola, Tommi},
  booktitle={Proceedings of the international biometrics society annual meeting},
  volume={15},
  pages={1},
  year={2006}
}

@article{hartigan_kmeans_1979,
  title={Algorithm AS 136: A k-means clustering algorithm},
  author={Hartigan, John A and Wong, Manchek A},
  journal={Journal of the royal statistical society. series c (applied statistics)},
  volume={28},
  number={1},
  pages={100--108},
  year={1979},
  publisher={JSTOR}
}

@article{mikolov_distributed_2013,
  title={Distributed representations of words and phrases and their compositionality},
  author={Mikolov, Tomas and Sutskever, Ilya and Chen, Kai and Corrado, Greg S and Dean, Jeff},
  journal={Advances in neural information processing systems},
  volume={26},
  year={2013}
}

@article{aloise_np-hardness_2009,
title = {{NP}-hardness of {Euclidean} sum-of-squares clustering},
volume = {75},
issn = {1573-0565},
url = {https://doi.org/10.1007/s10994-009-5103-0},
doi = {10.1007/s10994-009-5103-0},
language = {en},
number = {2},
urldate = {2023-10-16},
journal = {Machine Learning},
author = {Aloise, Daniel and Deshpande, Amit and Hansen, Pierre and Popat, Preyas},
month = may,
year = {2009},
pages = {245--248},
}

@inproceedings{kumar_linear_2005,
address = {Berlin, Heidelberg},
series = {Lecture {Notes} in {Computer} {Science}},
title = {Linear {Time} {Algorithms} for {Clustering} {Problems} in {Any} {Dimensions}},
isbn = {978-3-540-31691-6},
doi = {10.1007/11523468_111},
language = {en},
booktitle = {Automata, {Languages} and {Programming}},
publisher = {Springer},
author = {Kumar, Amit and Sabharwal, Yogish and Sen, Sandeep},
editor = {Caires, Luís and Italiano, Giuseppe F. and Monteiro, Luís and Palamidessi, Catuscia and Yung, Moti},
year = {2005},
pages = {1374--1385},
}

@inproceedings{veitch_empirical_2019,
  title={Empirical risk minimization and stochastic gradient descent for relational data},
  author={Veitch, Victor and Austern, Morgane and Zhou, Wenda and Blei, David M and Orbanz, Peter},
  booktitle={The 22nd International Conference on Artificial Intelligence and Statistics},
  pages={1733--1742},
  year={2019},
  organization={PMLR}
}

@article{zhu_global_2021,
  title={The global optimization geometry of low-rank matrix optimization},
  author={Zhu, Zhihui and Li, Qiuwei and Tang, Gongguo and Wakin, Michael B},
  journal={IEEE Transactions on Information Theory},
  volume={67},
  number={2},
  pages={1308--1331},
  year={2021},
  publisher={IEEE}
}

@article{ma_beyond_2021,
  title={Beyond Procrustes: Balancing-free gradient descent for asymmetric low-rank matrix sensing},
  author={Ma, Cong and Li, Yuanxin and Chi, Yuejie},
  journal={IEEE Transactions on Signal Processing},
  volume={69},
  pages={867--877},
  year={2021},
  publisher={IEEE}
}

@article{jin_optimal_2022,
  title={Optimal estimation of the number of network communities},
  author={Jin, Jiashun and Ke, Zheng Tracy and Luo, Shengming and Wang, Minzhe},
  journal={Journal of the American Statistical Association},
  pages={1--16},
  year={2022},
  publisher={Taylor \& Francis}
}

@article{le_estimating_2022,
  title={Estimating the number of communities by spectral methods},
  author={Le, Can M and Levina, Elizaveta},
  journal={Electronic Journal of Statistics},
  volume={16},
  number={1},
  pages={3315--3342},
  year={2022},
  publisher={The Institute of Mathematical Statistics and the Bernoulli Society}
}

@book{talagrand_upper_2014,
	address = {Berlin Heidelberg},
	series = {Ergebnisse der {Mathematik} und ihrer {Grenzgebiete}. 3. {Folge} / {A} {Series} of {Modern} {Surveys} in {Mathematics}},
	title = {Upper and {Lower} {Bounds} for {Stochastic} {Processes}: {Modern} {Methods} and {Classical} {Problems}},
	isbn = {978-3-642-54074-5},
	shorttitle = {Upper and {Lower} {Bounds} for {Stochastic} {Processes}},
	url = {https://www.springer.com/gp/book/9783642540745},
	language = {en},
	urldate = {2020-05-07},
	publisher = {Springer-Verlag},
	author = {Talagrand, Michel},
	year = {2014},
	doi = {10.1007/978-3-642-54075-2},
}

@inproceedings{tu_low-rank_2016,
	title = {Low-rank {Solutions} of {Linear} {Matrix} {Equations} via {Procrustes} {Flow}},
	url = {https://proceedings.mlr.press/v48/tu16.html},
	language = {en},
	urldate = {2022-05-03},
	booktitle = {Proceedings of {The} 33rd {International} {Conference} on {Machine} {Learning}},
	publisher = {PMLR},
	author = {Tu, Stephen and Boczar, Ross and Simchowitz, Max and Soltanolkotabi, Mahdi and Recht, Ben},
	month = jun,
	year = {2016},
	note = {ISSN: 1938-7228},
	pages = {964--973},
}

@book{horn_matrix_2012,
	address = {New York, NY, USA},
	edition = {2nd},
	title = {Matrix {Analysis}},
	isbn = {978-0-521-54823-6},
	publisher = {Cambridge University Press},
	author = {Horn, Roger A. and Johnson, Charles R.},
	year = {2012},
}

@article{velickovic_deep_2018,
	title = {Deep {Graph} {Infomax}},
	url = {http://arxiv.org/abs/1809.10341},
	language = {en},
	urldate = {2019-11-24},
	journal = {arXiv:1809.10341 [cs, math, stat]},
	author = {Veličković, Petar and Fedus, William and Hamilton, William L. and Liò, Pietro and Bengio, Yoshua and Hjelm, R. Devon},
	month = dec,
	year = {2018},
	note = {arXiv: 1809.10341}
}

\end{document}